\RequirePackage{fix-cm}
\documentclass[twocolumn]{svjour3}          
\smartqed  

\makeatletter
\renewcommand\subsection{\@startsection{subsection}{2}{\z@}%
    {-21dd plus-8pt minus-4pt}{10.5dd}
     {\normalsize\upshape\bfseries}}

\expandafter\let\csname opt@amsmath.sty\endcsname\relax

\makeatother

\usepackage{mathptmx}      
\usepackage{esvect}
\usepackage{multicol}
\usepackage{multirow}
\usepackage[defblank]{paralist}
\usepackage{enumerate}
\usepackage[hidelinks]{hyperref}
\usepackage[T1]{fontenc}
\usepackage[utf8]{inputenc}
\usepackage{amsmath}
\usepackage{amssymb}
\usepackage{xspace}
\usepackage{graphicx}
\usepackage[misc]{ifsym}
\usepackage{algpseudocode}
\usepackage{algorithm}
\usepackage{algorithmicx}
\usepackage[english]{babel}
\usepackage[numbers]{ natbib}
\usepackage{bigstrut}
\DeclareMathAlphabet{\mathcal}{OMS}{cmsy}{m}{n} 

\usepackage[english]{cleveref}
\crefformat{footnote}{#2\footnotemark[#1]#3}
\journalname{ }

\newif\ifShowOutlineShowcase 
\ShowOutlineShowcasetrue
\ShowOutlineShowcasefalse

\newif\ifdraft 
\drafttrue

\ifdraft 
\else
\ShowOutlineShowcasefalse
\fi

\hypersetup{%
  pdftitle={Specification-Driven Predictive Business Process
    Monitoring},%
  pdfauthor={Ario Santoso (ORCID:0000-0001-6212-4365), Michael
    Felderer (ORCID:0000-0003-3818-4442)},%
  pdfsubject={Specification-Driven Predictive Business Process
    Monitoring},%
  pdfkeywords={Predictive Business Process Monitoring, Prediction Task
    Specification, Automatic Prediction Model Creation, Machine
    Learning Based Prediction},%
  pdfcreator={Ario Santoso (ORCID:0000-0001-6212-4365), Michael
    Felderer (ORCID:0000-0003-3818-4442)},%
}

\newcommand{\set}[1]{\{#1\}}                      

\newcommand{\card}[1]{|{#1}|}                     

\newcommand{\tup}[1]{\langle#1\rangle}            

\newcommand{\A}{\mathcal{\uppercase{A}}}

\newcommand{\E}{\mathcal{\uppercase{E}}}

\renewcommand{\P}{\mathcal{\uppercase{P}}}

\newcommand{\R}{\mathcal{\uppercase{R}}}

\newcommand{\ra}{\rightarrow}

\newcommand{\bqm}{\textbf{ ? }}
\newcommand{\cm}{\textbf{ \checkmark }}

\renewcommand{\vec}[1]{\vv{#1}}

{\itshape}{\rmfamily}

\let\myendexample=\endexample
\def\endexample{\hfill\ensuremath{\blacksquare}\myendexample}

\newcommand{\udefined}{\bot}

\newcommand{\seq}{\sigma}

\newcommand{\prefix}[2]{#2^{#1}}
\newcommand{\pref}[2]{#2,#1}

\newcommand{\eventuniv}{\ensuremath{\E}\xspace}
\newcommand{\attnameuniv}{\ensuremath{\A}\xspace}

\newcommand{\attval}[1]{\#_{\text{{#1}}}}

\newcommand{\event}{\ensuremath{e}\xspace}
\newcommand{\trace}{\ensuremath{\tau}\xspace}
\newcommand{\eventlog}{\ensuremath{L}\xspace}

\newcommand{\str}{\mathsf{String}}
\newcommand{\num}{\mathsf{number}}

\newcommand{\langname}{\text{First-Order Event Expression}\xspace}
\newcommand{\langnameabr}{\text{FOE}\xspace}

\newcommand{\true}{\mathsf{true}}
\newcommand{\false}{\mathsf{false}}

\newcommand{\fforall}{\forall}
\newcommand{\fexists}{\exists}
\newcommand{\fand}{~\wedge~}
\newcommand{\for}{~\vee~}
\newcommand{\fimpl}{~\ra~}

\newcommand{\fandcompact}{\wedge}
\newcommand{\forcompact}{\vee}
\newcommand{\fimplcompact}{\ra}

\newcommand{\eventquery}[2]{\text{\texttt{\small e}}[#1]\textbf{. }{\text{{#2}}}}

\newcommand{\eventexpshort}{\ensuremath{\mathsf{eventExp}}\xspace}
\newcommand{\numexpb}{\ensuremath{\mathsf{numExp}}\xspace}
\newcommand{\nonnumexpb}{\ensuremath{\mathsf{nonNumExp}}\xspace}

\newcommand{\noteq}{\neq}

\newcommand{\logeq}{==}

\newcommand{\idx}{\mathsf{idx}\xspace}
\newcommand{\nat}{\text{\textit{pint}}\xspace}

\newcommand{\last}{\ensuremath{\mathsf{last}\xspace}}
\newcommand{\curr}{\ensuremath{\mathsf{curr}\xspace}}

\newcommand{\ar}{R\xspace}
\newcommand{\arset}{\R\xspace}
\newcommand{\targetarrow}{\Longrightarrow}

\newcommand{\cond}{\mathsf{Cond}}
\newcommand{\target}{\mathsf{Target}}
\newcommand{\otarget}{\mathsf{DefaultTarget}}

\newcommand{\ruletup}[1]{\tup{#1}}

\newcommand{\analrule}{analytic rule\xspace}
\newcommand{\analrules}{analytic rules\xspace}
\newcommand{\AnalRule}{Analytic Rule\xspace}

\newcommand{\rep}[1]{\ensuremath{\textsf{#1}}}

\newcommand{\eval}{eval}

\newcommand{\condtargetrules}{\text{conditional-target expressions}\xspace}

\newcommand{\inter}[4]{(#1)^{#2,#3}_{#4}}

\newcommand{\val}{\ensuremath{\nu}}

\newcommand{\satisfyb}[4]{(#4)^{#1,#2}_{#3}}

\newcommand{\aggsum}{\textsf{\textbf{sum}}\xspace}
\newcommand{\aggavg}{\textsf{\textbf{avg}}\xspace}
\newcommand{\aggmin}{\textsf{\textbf{min}}\xspace}
\newcommand{\aggmax}{\textsf{\textbf{max}}\xspace}
\newcommand{\aggcount}{\textsf{\textbf{count}}\xspace}
\newcommand{\aggcountval}{\textsf{\textbf{countVal}}\xspace}
\newcommand{\aggconcat}{\textsf{\textbf{concat}}\xspace}

\newcommand{\aggrange}{\textsf{\scalebox{.8}[.8]{\textbf{where}}}}
\newcommand{\aggwithin}{\textsf{\scalebox{.8}[.8]{\textbf{within}}}}
\newcommand{\aggcond}{\textsf{\scalebox{.8}[.8]{\textbf{and}}}}

\newcommand{\aggcondexp}{\ensuremath{\mathsf{aggCond}}\xspace}
\newcommand{\aggnumsrc}{\ensuremath{\mathsf{numSrc}}\xspace}
\newcommand{\aggnonnumsrc}{\ensuremath{\mathsf{nonNumSrc}}\xspace}

\newcommand{\idxset}{\mathsf{Idx}}

\newcommand{\st}{\mathsf{st}}
\newcommand{\ed}{\mathsf{ed}}

\newcommand{\concatoperand}{\odot}

\newcommand{\aggseparator}{;~}

\newcommand{\emptystr}{\text{"}\!\ \text{"}\xspace}

\newcommand{\agga}[5]{#1(#2\aggseparator\aggrange~x=#3:#4\aggseparator\aggcond~#5)}
\newcommand{\uagga}[4]{#1(#2\aggseparator\aggrange~x=#3:#4)}
\newcommand{\aggtwo}[3]{#1(#2,#3)}

\newcommand{\suma}[4]{\agga{\aggsum}{#1}{#2}{#3}{#4}}
\newcommand{\avga}[4]{\agga{\aggavg}{#1}{#2}{#3}{#4}}
\newcommand{\mina}[4]{\agga{\aggmin}{#1}{#2}{#3}{#4}}
\newcommand{\maxa}[4]{\agga{\aggmax}{#1}{#2}{#3}{#4}}
\newcommand{\concata}[4]{\agga{\aggconcat}{#1}{#2}{#3}{#4}}

\newcommand{\usuma}[3]{\uagga{\aggsum}{#1}{#2}{#3}}
\newcommand{\uavga}[3]{\uagga{\aggavg}{#1}{#2}{#3}}
\newcommand{\umina}[3]{\uagga{\aggmin}{#1}{#2}{#3}}
\newcommand{\umaxa}[3]{\uagga{\aggmax}{#1}{#2}{#3}}
\newcommand{\uconcata}[3]{\uagga{\aggconcat}{#1}{#2}{#3}}

\newcommand{\mintwo}[2]{\aggtwo{\aggmin}{#1}{#2}}
\newcommand{\maxtwo}[2]{\aggtwo{\aggmax}{#1}{#2}}

\newcommand{\counta}[3]{\aggcount(#1;~\aggrange~x=#2:#3)}
\newcommand{\countvala}[3]{\aggcountval(\text{#1};~\aggwithin~#2:#3)}

\newcommand{\aggb}[5]{
\ensuremath{
\begin{array}[t]{l@{ }l}
#1(&#2;~\aggrange~x=#3:#4;\\
&\aggcond~#5)
\end{array}
}
}

\newcommand{\sumb}[4]{\aggb{\aggsum}{#1}{#2}{#3}{#4}}

\newcommand{\countb}[3]{
\ensuremath{
\begin{array}[t]{l@{ }l}
\aggcount(&#1;\\
&\aggrange~x=#2:#3)
\end{array}
}
}

\newcommand{\aggc}[5]{
\ensuremath{
\begin{array}[t]{l@{ }l}
{#1}(&#2;\\
&\aggrange~x=#3:#4;\\
&\aggcond~#5)
\end{array}
}
}

\newcommand{\uaggc}[4]{
\ensuremath{
\begin{array}[t]{l@{ }l}
{#1}(&#2;\\
&\aggrange~x=#3:#4)
\end{array}
}
}

\newcommand{\sumc}[4]{\aggc{\aggsum}{#1}{#2}{#3}{#4}}

\newcommand{\usumc}[3]{\uaggc{\aggsum}{#1}{#2}{#3}}
\newcommand{\uavgc}[3]{\uaggc{\aggavg}{#1}{#2}{#3}}

\newcommand{\uconcatc}[3]{\uaggc{\aggconcat}{#1}{#2}{#3}}

\newcommand{\aggd}[5]{
\ensuremath{
\begin{array}[t]{l@{ }l}
{#1}(&#2;\\
&\aggrange~x=#3:#4;\\
&\aggcond~#5
\end{array}
}
}
\newcommand{\sumd}[4]{\aggd{\aggsum}{#1}{#2}{#3}{#4}}

\algnewcommand\algorithmicforeach{\textbf{for each}}
\algdef{S}[FOR]{ForEach}[1]{\algorithmicforeach\ #1\ \algorithmicdo}

\newcommand{\encfunc}{\mathsf{enc}}
\newcommand{\encset}{\mathsf{Enc}}
\newcommand{\predfunc}{\P}

\newcommand{\condPingPonga}{\ensuremath{\cond_{1}}}
\newcommand{\condPingPongb}{\ensuremath{\cond_{2}}}
\newcommand{\condAbnormalWaitDur}{\ensuremath{\cond_{3}}}
\newcommand{\condEveryOrderWillBeDelivered}{\ensuremath{\cond_{4}}}
\newcommand{\condSoD}{\ensuremath{\cond_{5}}}
\newcommand{\condActDurLessThanAThreshold}{\ensuremath{\cond_{6}}}
\newcommand{\condActDurLessThanAThresholdb}{\ensuremath{\cond_{7}}}
\newcommand{\condConstValDur}{\ensuremath{\cond_{8}}}
\newcommand{\condDelay}{\ensuremath{\cond_{9}}}

\newcommand{\condNumValidation}{\ensuremath{\cond_{10}}}
\newcommand{\condInvolveThreeResources}{\ensuremath{\cond_{11}}}

\newcommand{\condProcPerformance}{\ensuremath{\cond_{12}}}
\newcommand{\condProcPerformanceb}{\ensuremath{\cond_{13}}}
\newcommand{\condEfficiency}{\ensuremath{\cond_{14}}}
\newcommand{\condEfficiencyb}{\ensuremath{\cond_{15}}}

\newcommand{\arPingPongExample}{\ar_{1}}
\newcommand{\arRemainingTimeExample}{\ar_{2}}
\newcommand{\arPingPongShowcase}{\ar_{3}}
\newcommand{\arAbWaitDur}{\ar_{4}}
\newcommand{\arOrderWillBeDeliveredShowcase}{\ar_{5}}
\newcommand{\arSoD}{\ar_{6}}
\newcommand{\arActDurSLA}{\ar_{7}}
\newcommand{\arActDurSLAb}{\ar_{8}}
\newcommand{\arDelay}{\ar_{9}}
\newcommand{\arOvertime}{\ar_{10}}
\newcommand{\arRemWaitDur}{\ar_{11}}
\newcommand{\arAvgActDur}{\ar_{12}}

\newcommand{\arRemAct}{\ar_{13}}
\newcommand{\arRemValidationAct}{\ar_{14}}
\newcommand{\arComplexProcess}{\ar_{15}}
\newcommand{\arNumDifResources}{\ar_{16}}
\newcommand{\arNumHandovers}{\ar_{17}}
\newcommand{\arLaborIntensive}{\ar_{18}}
\newcommand{\arTotalCost}{\ar_{19}}
\newcommand{\arMaxCost}{\ar_{20}}
\newcommand{\arAvgCost}{\ar_{21}}
\newcommand{\arValidationCost}{\ar_{22}}
\newcommand{\arTotalCostb}{\ar_{23}}
\newcommand{\arExpensive}{\ar_{24}}

\newcommand{\arProcPerformance}{\ar_{25}}
\newcommand{\arProcPerformanceb}{\ar_{26}}
\newcommand{\arEfficiency}{\ar_{27}}
\newcommand{\arEfficiencyb}{\ar_{28}}
\newcommand{\arNextEvent}{\ar_{29}}
\newcommand{\arLifeCycle}{\ar_{30}}
\newcommand{\arNextThreeEvents}{\ar_{31}}

\newcommand{\targetremtime}{\target_{\text{remainingTime}}}

\newcommand{\condPingPongExpa}{\ensuremath{\cond_{\text{E1}}}}
\newcommand{\condPingPongExpb}{\ensuremath{\cond_{\text{E2}}}}
\newcommand{\condPingPongExpc}{\ensuremath{\cond_{\text{E3}}}}
\newcommand{\condPingPongExpd}{\ensuremath{\cond_{\text{E4}}}}
\newcommand{\condPingPongExpe}{\ensuremath{\cond_{\text{E5}}}}
\newcommand{\condPingPongExpf}{\ensuremath{\cond_{\text{E6}}}}

\newcommand{\condEventuallyDeclined}{\ensuremath{\cond_{\text{E8}}}}

\newcommand{\arPingPongExpa}{\ar_{\text{E1}}} 
\newcommand{\arPingPongExpb}{\ar_{\text{E2}}} 
\newcommand{\arPingPongExpc}{\ar_{\text{E3}}} 
\newcommand{\arRemWaiting}{\ar_{\text{E4}}} 
\newcommand{\arRemWaitingb}{\ar_{\text{E5}}} 

\newcommand{\arRemTimeFillingApplication}{\ar_{\text{E6}}}
\newcommand{\arEventuallyDeclined}{\ar_{\text{E7}}}

\newcommand{\arComplexAppB}{\ar_{\text{E8}}} 
\newcommand{\arRemainingAct}{\ar_{\text{E9}}} 

\newcommand{\totaltasksforexperiments}{9\xspace}

\newcommand{\auc}{AUC\xspace}
\newcommand{\acc}{Accuracy\xspace}
\newcommand{\wprecision}{W. Prec\xspace}
\newcommand{\wrecall}{W. Rec\xspace}
\newcommand{\fmeasure}{F-Measure\xspace}

\newcommand{\mae}{MAE\xspace}
\newcommand{\rmse}{RMSE\xspace}

\newcommand{\zeroR}{ZeroR\xspace}
\newcommand{\logreg}{Logistic Reg.\xspace}
\newcommand{\nbayes}{Naive Bayes\xspace}
\newcommand{\dectree}{Decision Tree\xspace}

\newcommand{\randfor}{Random Forest\xspace}
\newcommand{\adaboost}{Ada Boost\xspace}
\newcommand{\extratree}{Extra Trees\xspace}
\newcommand{\voting}{Voting\xspace}
\newcommand{\nn}{Deep Neural Net.\xspace}

\newcommand{\zeroRreg}{ZeroR\xspace}
\newcommand{\linreg}{Linear Reg.\xspace}
\newcommand{\dectreereg}{Decision Tree\xspace}

\newcommand{\randforreg}{Random Forest\xspace}
\newcommand{\adaboostreg}{Ada Boost\xspace}
\newcommand{\extratreereg}{Extra Trees\xspace}
\newcommand{\nnreg}{Deep Neural Net.\xspace}

\usepackage[normalem]{ulem} 
\ifdraft
  \usepackage{mychanges}
  \usepackage[colorinlistoftodos]{todonotes}
\else
  \usepackage[final]{mychanges}
  \usepackage[disable]{todonotes}
\fi
\setauthormarkup[right]{{\scriptsize[#1]}}
\definechangesauthor{AS}{blue}
\definechangesauthor{TW}{orange}

\begin{document}

\title{Specification-Driven Predictive Business Process Monitoring
  \thanks{This research has been supported by the Euregio
    Interregional Project Network IPN12 ``\textit{KAOS:
      Knowledge-Aware Operational Support} (KAOS)'', which is funded
    by the ``European Region Tyrol-South Tyrol-Trentino'' (EGTC) under
    the first call for basic research projects.
}
}


\author{Ario Santoso$^{1,2}$ 
\and Michael Felderer$^{1,3}$
}
\authorrunning{Ario Santoso \and Michael Felderer
} 

\institute{
Ario Santoso$^{1,2}$\\
{ario.santoso@uibk.ac.at}\\[0.5ex]
Michael Felderer$^{1,3}$\\ 
{michael.felderer@uibk.ac.at}\\
\at 
$
\begin{array}{@{}l@{\ }l}
^1&\text{Department of Computer Science, University of
Innsbruck, Austria} \\[0.3ex]
^2&\text{Faculty of Computer Science, }\\ &\text{Free University of
Bozen-Bolzano, Italy} \\[0.3ex]
^3&\text{Department of Software Engineering, } \\
&\text{Blekinge Institute of Technology, Sweden} 
\end{array}
$
%
}

\date{ }

\maketitle

\begin{abstract}

  Predictive analysis in business process monitoring aims at
  forecasting the future information of a running business
  process. The prediction is typically made based on the model
  extracted from historical process execution logs (event logs). In
  practice, different business domains might require different kinds
  of predictions. Hence, it is important to have a means for properly
  specifying the desired prediction tasks, and a mechanism to deal
  with these various 
  prediction tasks. 
  Although there have been many studies in this area, they mostly
  focus on a specific prediction task.
%
  This work introduces 
  a language for specifying the desired prediction tasks, and this
  language allows us to express various kinds of prediction
  tasks. This work also presents a mechanism for automatically
  creating the corresponding prediction model based on the given
  specification.
%
  Differently from previous studies, instead of focusing on a
  particular prediction task, we present an approach to deal with
  various
  prediction tasks based on the given specification of the desired
  prediction tasks.
%
%
  We also provide an implementation of the approach which is used to
  conduct experiments using real-life event logs.



\end{abstract}

\section{Introduction}
\label{sec:introduction}

%

Process mining~\cite{ProcessMiningManifesto,Aalst:2016} provides a
collection of techniques for extracting process-related information
from the logs of business process executions (event logs). One
important area in this field is predictive business process
monitoring, which aims at forecasting the future information of a
running process based on the models extracted from event logs.
Through predictive analysis, potential future problems can be detected
and preventive actions can be taken in order to avoid unexpected
situation, e.g., processing delay and Service-Level Agreement (SLA)
violations.
%
%
Many studies have been conducted in order to deal with various
prediction tasks
such as 
predicting the remaining processing
time~\cite{ASS11,TVLD17,RW13,PSBD14, PSBD18}, 
predicting the outcomes of a
process~\cite{MFDG14,DDFT16,VDLMD15,PVWFT16},
predicting future events~\cite{DGMPY17,TVLD17,ERF17b}, etc
(cf.~\cite{MLISFCDP15,MFE12,SWGM14,PVFTW12,BMDB16,CDLVT15}).
An 
overview of various works in the area of predictive business process
monitoring can be found in~\cite{MRR18,DGMM18}.


In practice, different business areas might need different kinds of
prediction tasks. For instance, 
an online retail company might be interested in predicting the
processing time until an order can be delivered to the customer,
%
%
while for an insurance company, predicting
the outcome of an insurance claim process would be interesting. On
the other hand, both of them might be interested in predicting whether
their processes comply with some business constraints (e.g., the
processing time must be less than  a certain amount of time).

When it comes to predicting the outcome of a process, business
constraint satisfaction and the existence of an unexpected behaviour,
it is important to specify the desired outcomes, the business
constraint and the unexpected behaviour precisely.
For instance, in the area of customer problem management, to increase
the customer satisfaction as well as to promote efficiency, we might
be interested in predicting the possibility of \emph{ping-pong
  behaviour} among the Customer Service (CS) officers while handling
the customer problems.
%
%
However, the definition of a ping-pong behaviour could be varied.  For
instance, when a CS officer transfers a customer problem to another
CS officer who belongs into the same group, it can already be considered
as a ping-pong behaviour since both of them should be able to handle
the same problem.  Another possible definition would be to consider a
ping-pong behaviour as a situation when a CS officer transfers a
problem to another CS officer who has the same expertise, and the
problem is transfered back to the original CS officer.


To have a suitable prediction service for our domain, we need to be
able to 
specify the desired prediction tasks properly.  Thus, we need a means
to express the specification.
Once we have characterized the prediction objectives and are able to
express them properly, we need a mechanism to create the corresponding
prediction model. To automate the prediction model creation, the
specification should be unambiguous and machine processable.
%
As illustrated above, such specification mechanism should also allow
us to specify constraints over the data, and compare data values at
different time points. For example, to characterize the ping-pong
behaviour, one possibility is to specify the behaviour as follows:
%
``\textit{there is an event at a certain time point in which the CS
  officer (who handles the problem) is different from the CS officer
  in the event at the next time point, but both of them belong to the
  same group}''. Note that here we need to compare the information
about the CS officer names and groups at different time points.
In other cases, we might even need to involve arithmetic
expressions. For instance, consider a business constraint that
requires that the length of customer order processing time to be less
than 3 hours, where the length of the processing time is the time
difference between the timestamp of the first activity and the last
activity within the process. To express this constraint, we need to be
able to specify that ``\textit{the time difference between the
  timestamp of the first activity and the last activity within the
  process is less than 3 hours}''.

The language should also enable us to specify how to compute/obtain
the target information to be predicted. For instance, in the
prediction of remaining processing time, we need to be able to define
that the remaining processing time is \emph{the time difference
  between timestamp of the last activity and the current
  activity}. 
We might also need to aggregate some data values, for instance in the
prediction of the total processing cost where the total cost is
\emph{the sum over the cost of all activities/events}.
%
%
%
In other cases, we might even need 
to specify an expression that counts the number of a certain activity.
For example in the prediction of the amount of work to be done
(workload), we might be interested in predicting the \emph{number of
  the remaining validation activities} that are necessary to be done
for processing a client application.

In this work, we tackle those problems by proposing an approach for
obtaining the desired prediction services based on the specification
of the desired prediction tasks. Specifically, we provide
the following contributions:
\begin{enumerate}
\item We introduce a 
  rich language for expressing the desired prediction tasks.
  This language allows us to specify various 
  desired prediction tasks.
%
%
  In some sense, this language 
  allows us to specify how to create the desired prediction models
  based on the event logs.
%
  We also provide a formal semantics for the language in order to
  ensure a uniform understanding 
   and 
   avoid ambiguity.
%
\item We devise a mechanism for building the corresponding prediction
  model based on the given 
  specification. This 
  includes the mechanism 
  for automatically processing the specification.
  Once created, the prediction model can be used to provide predictive
  analysis services in business process monitoring.
\item To provide a general idea on the capability of our
  language, 
  we exhibit how our proposal can be used for specifying
  various 
  prediction tasks (cf.~\Cref{sec:showcase}).
\item 
%
%
  We provide an implementation of our approach which enables the
  automatic creation of prediction models based on the specified
  prediction objective.
\item To demonstrate the applicability of our approach, we carry out
  experiments using real-life event logs that were provided for the
  Business Process Intelligence Challenge (BPIC) 2012, 2013, and 2015.
%
%
\end{enumerate}

Our approach for obtaining prediction services essentially consists of
the following main steps:
\begin{inparaenum}[\itshape (i)]
\item First, we specify the desired prediction tasks,
\item Second, we automatically create the prediction models based on
  the given specification,
\item Once created, we can use the constructed prediction models for
  predicting the future information of a running process.
\end{inparaenum}
%

Roughly speaking,
%
we specify the desired prediction task by specifying
how we want to map each (partial) business processes execution
information into the expected predicted information.
Based on this specification, we 
train either a classification or regression model that will serve as
the prediction model.
%
%
By specifying a set of desired prediction
tasks, 
we could obtain \emph{multi-perspective prediction services} that
enable us to focus on different aspects and predict various
information of interest.
Our approach is independent with respect to the
classification/regression model that is used. 
In our implementation, to get the expected quality of predictions, the
users are allowed to choose the desired classification/regression
model as well as the feature encoding mechanisms (in order to allow
some sort of feature engineering).
%


This article extends~\cite{AS-BPMDS-18} in several ways. First, we
extend the specification language so as to incorporate various
aggregate functions such as Max, Min, Average, Sum, Count, and
Concat. Importantly, our aggregate functions allow us not
only 
to perform aggregation over some values but also to choose the values
to be aggregated.
Obviously this extension increases the expressivity of the language
and 
allows us to specify many more interesting prediction tasks. Next, we
add various new showcases that exhibit the capabilities of our
language in specifying 
prediction tasks. We also extend the implementation of our prototype
in order to incorporate those extensions.
%
%
%
To demonstrate the applicability of our approach, more experiments on
different prediction tasks are also conducted and presented.
Apart from using the real-life event log that was provided for BPIC
2013~\cite{BPI-13-data}, 
we also use another real-life event logs, namely the event logs that
were provided for BPIC 2012~\cite{BPI-12-data} and BPIC
2015~\cite{BPI-15-data}.
Notably, our experiments also exhibit the usage of a Deep Learning
model~\cite{GBC16} in predictive process monitoring. In particular, we
use Deep Feed-Forward Neural Network. Though there have been some
works that exhibit the usage of deep learning models in predictive
process monitoring (cf.~\cite{TVLD17,ERF17a,ERF17b,DGMPY17,MEF17}),
here we consider the prediction tasks that are different from the
tasks that have been studied in those works.
%
%
We also add more thorough explanation on several concepts and ideas of
our approach so as to provide a better 
understanding. The discussion on the related work is also
extended. Last but not least, several examples are added in order to
support the explanation of various technical concepts as well as to
ease the understanding of the ideas.




The remainder of this article is structured as follows. In
\Cref{sec:preliminaries}, we provide the required background on the
concepts that are needed for the rest of the paper. Having laid the
foundation, in \Cref{sec:spec-language}, we present the language that
we introduce for specifying the desired prediction tasks. In
\Cref{sec:pred-model}, we present a mechanism for building the
corresponding prediction model based on the given specification. In
\Cref{sec:showcase}, we continue the explanation by providing numerous
showcases that exhibit the capability of our language in specifying
various prediction tasks. In \Cref{sec:implementation-experiment}, we
present the implementation of our approach as well as the experiments
that we have conducted. Related work is presented in
\Cref{sec:related-work}. 
Finally, in \Cref{sec:discussion} we present a discussion on some
potential limitations which pave the way towards our future direction,
and \Cref{sec:conclusion} concludes this work.


\section{Preliminaries}
\label{sec:preliminaries}

We will see later that we build the prediction models by using machine
learning classification/regression techniques and based on the data in
event logs.
To provide some background concepts, this section briefly explains the
typical structure of event logs as well as the notion of
classification and regression in machine learning.

\subsection{Trace, Event and Event Log}\label{sec:event-log-structure}
We follow the usual notion of event logs as in process
mining~\cite{Aalst:2016}. Essentially, an event log captures
historical information of business process executions.
%
Within an event log, an execution of a business process instance (a
case) is represented as a trace. In the following, we may use the
terms \emph{trace} and \emph{case} interchangeably.
Each trace has several events, and each event in a trace captures the
information about a particular event/activity that happens during the
process execution.
Events are characterized by various attributes, e.g., \emph{timestamp}
(the time when the event occurred). 
%


We now proceed to formally define the notion of event logs as well as
their components.  Let $\eventuniv$ be the \emph{event universe}
(i.e., the set of all event identifiers), and $\attnameuniv$ be the
set of \emph{attribute names}. For any event $\event \in \eventuniv$,
and attribute name $n \in \attnameuniv$, $\attval{n}(\event)$ denotes
the \emph{value of attribute} $n$ of $\event$. E.g.,
$\attval{timestamp}(\event)$ denotes the timestamp of the event
$\event$. If an event $\event$ does not have an attribute named $n$,
then $\attval{n}(\event) = \udefined$ (where $\udefined$ is undefined
value).
A \emph{finite sequence over $\eventuniv$ of length
  $n$} 
is a mapping $\seq: \set{1, \ldots, n} \ra
\eventuniv$, and we represent such a sequence as a tuple of elements
of $\eventuniv$, i.e., $\seq = \tup{\event_1, \event_2, \ldots,
  \event_n}$
where $\event_i = \seq(i)$ for $i \in \set{1, \ldots, n}$.
The set of all \emph{finite sequences} over $\eventuniv$ is denoted by
$\eventuniv^*$.
The \emph{length} of a sequence $\seq$ is denoted by $\card{\seq}$.

A \emph{trace} $\trace$ is a finite sequence over $\eventuniv$ such
that each event $\event \in \eventuniv$ occurs at most once in
$\trace$, i.e., $\trace \in \eventuniv^{*}$ and for
$1~\leq~i~<~j~\leq~\card{\trace}$, we have $\trace(i) \neq \trace(j)$,
where
$\trace(i)$ refers to the \emph{event of the trace $\trace$ at the
  index $i$}.
Let $\trace = \tup{e_1, e_2, \ldots, e_n}$ be a trace,
$\prefix{k}{\trace} = \tup{e_1, e_2, \ldots, e_{k}}$ denotes the
\emph{$k$-length trace prefix}~of~$\trace$ (for $1~\leq~k~<~n$).  

\begin{example}
  For example, let
  $\set{e_1, e_2, e_3, e_4, e_5, e_6, e_7}~\subset~\eventuniv$ be some
  event identifiers, then the sequence
  $\trace~=~\tup{e_3, e_7, e_6, e_4, e_5}~\in~\eventuniv^{*}$ is an
  example of a trace. In this case, we have that $\card{\trace}~=~5$,
  and $\trace(3)$ refers to the event of the trace $\trace$ at the
  index 3, i.e., $\trace(3) = e_6$. Moreover, $\prefix{2}{\trace}$ is
  the prefix of length 2 of the trace $\trace$, i.e.,
  $\prefix{2}{\trace} = \tup{e_3, e_7}$.
\end{example}

Finally, an \emph{event log} $\eventlog$ is a set of traces such that
each event occurs at most once in the entire log, i.e., for each
$\trace_1, \trace_2 \in \eventlog$ such that $\trace_1 \neq \trace_2$,
we have that $\trace_1 \cap \trace_2 = \emptyset$, where
$\trace_1 \cap \trace_2 = \set{\event \in \eventuniv~\mid~\exists i, j \in
  \mathbb{Z}^+ \text{ . } \trace_1(i) = \trace_2(j) = e }$.

An IEEE standard for representing event logs, called XES (eXtensible
Event Stream), has been introduced in~\cite{IEEE-XES:2016}. The
standard defines the XML format for organizing the structure of
traces, events and attributes in event logs. It also introduces some
extensions that define some attributes with pre-defined meaning such
as:
\begin{compactenum}
\item \textit{concept:name}, which stores the name of event/trace;
\item \textit{org:resource}, which stores the name/identifier of the
  resource that triggered the event (e.g., a person name);
\item \textit{org:group}, which stores the group name of the resource
  that triggered the event.
\end{compactenum}

\subsection{Classification and Regression}\label{sec:classification-regression}
In machine learning, 
a classification and regression model can be seen as a function
$f: \vv{X} \ra Y$ that takes some \emph{input
  features}/\emph{variables} $\vv{x} \in \vv{X}$ and predicts the
corresponding \emph{target value/output} $y \in Y$.
The key difference is that the output range of the classification task
is a finite number of discrete categories (qualitative outputs) while
the output range of the regression task is continous values
(quantitative outputs)~\cite{FHT01,HPK11}.  Both of them are
supervised machine learning techniques where the models are trained
with labelled data. I.e., the inputs for the training are pairs of
input variables~$\vec{x}$ and (expected) target value~$y$. This way,
the models learn how to map certain inputs~$\vec{x}$ into the expected
target value~$y$.



\section{Specifying the Desired Prediction Tasks }\label{sec:spec-language}

This section elaborates our mechanism for specifying the desired
prediction tasks.  Here we introduce a language that is able to
capture the desired prediction task in terms of the specification on
how to map each (partial) trace in the event log into the desired
prediction results.
Such specification can be used to train a classification/regression
model that will be used as the prediction model.


To express the specification of a prediction task, we introduce the
notion of \emph{\analrule}. An \emph{\analrule} $\ar$ is
an expression of the form:
%
\[
\begin{array}{r@{ \ }l}
  \ar = \ruletup{& 
    \cond_1 \targetarrow \target_1, ~\\
    &\cond_2 \targetarrow \target_2, ~\\
    &\qquad \qquad\vdots ~ \\
    &\cond_n \targetarrow \target_n, ~\\
    &\otarget \ },
\end{array}
\]
%
where
\begin{inparaenum}[\itshape (i)]
\item $\cond_i$ (for $i \in \set{1,\ldots,n}$) is called
  \emph{condition expression};
\item $\target_i$ (for $i \in \set{1,\ldots,n}$) is called
  \emph{target expression}.
\item $\otarget$ is a special target expression called \emph{default
    target expression}.
\item The expression $\cond_i~\targetarrow~\target_i$ is called
  \emph{conditional-target expression}.
\end{inparaenum}


\Cref{sec:overview-lang} provides an informal intuition of our
language for specifying prediction tasks. Throughout
\Cref{sec:formalization-cond-target-exp,sec:FOE}, we introduce the
language for specifying the condition and target expressions in
\analrules.
Specifically, \Cref{sec:FOE} introduces a language called \langname
(\langnameabr), while \Cref{sec:formalization-cond-target-exp}
elaborates several components that are needed to define such
language. We will see later that \langnameabr can be used to formally
specify condition expressions and a fragment of \langnameabr can be used
to specify target expressions. Finally, the formalization of
\analrules is provided in \Cref{sec:formalize-analytic-rule}.

\subsection{Overview: Prediction Task Specification Language
}\label{sec:overview-lang}

An \analrule $\ar$ is interpreted as a mapping that maps each
(partial) trace into a value that is obtained by evaluating the target
expression in which the corresponding condition is satisfied by the
corresponding trace.
%
Let $\trace$ be a (partial) trace, such mapping $\ar$ can be
illustrated as follows

\[
\ar(\trace)
  = \left\{ \begin{array}{l@{ \ }l}
\eval(\target_1) &\mbox{ if  } \trace \mbox{  satisfies } \cond_1 \mbox{, }\\
\eval(\target_2) &\mbox{ if  } \trace \mbox{  satisfies } \cond_2 \mbox{, }\\
\ \ \ \ \ \ \ \ \ \vdots& \ \ \ \ \ \ \ \  \ \ \ \ \ \ \ \ \vdots\\
\eval(\target_n) &\mbox{ if  } \trace \mbox{  satisfies } \cond_n \mbox{, }\\
\eval(\otarget) & \mbox{ otherwise}
          \end{array}
\right.
\]

\noindent
where $\eval(\otarget)$ and $\eval(\target_i)$ consecutively denote
the results of evaluating the target expression $\otarget$ and
$\target_i$, for $i \in \set{1,\ldots, n}$ 
 (The formal definition of this evaluation operation is given later).

We will see later that a target expression
specifies either the desired prediction result or expresses the way to
compute the desired prediction result. Thus, an \analrule $\ar$ can
also be seen as a means to map (partial) traces into either the
desired prediction results, or to compute the expected prediction
results of (partial) traces.

To specify condition expressions in \analrules, we introduce a
language called \langname (\langnameabr). Roughly speaking, an
\langnameabr formula is a First-Order Logic (FOL)
formula~\cite{Smul68} where the atoms are expressions over some event
attribute values and some comparison operators, e.g., $\logeq$,
$\noteq$, $>$, $\leq$.
%
The quantification in \langnameabr is restricted to the
indices of events (so as to quantify the time points).
The idea of condition expressions is to capture a certain property of
(partial) traces. To give some intuition, before we formally define
the language in \Cref{sec:FOE}, consider the ping-pong behaviour that
can be specified as follows:

\begin{center}
$
\begin{array}{r@{ \ }l}
  \condPingPonga =   &\fexists i . (i > \curr \fand i+1 \leq \last \fand\\
                                   & \eventquery{i}{org:resource} \noteq
                                     \eventquery{i+1}{org:resource}  
                                     \fand \\
                                   &\eventquery{i}{org:group} \logeq
                                     \eventquery{i+1}{org:group}
                                     )
\end{array}
$
\end{center}

\noindent
where 
\begin{inparaenum}[\itshape (i)]
\item $\eventquery{i+1}{org:group}$ is an expression for getting the
  org:group attribute value of the event at index $i+1$ (similarly for
  $\eventquery{i}{org:resource}$, $\eventquery{i+1}{org:resource}$,
  and $\eventquery{i}{org:group}$),
\item $\curr$ refers to the current time point, and 
\item $\last$ refers to
  the last time point.
\end{inparaenum}
 
The formula $\condPingPonga$ basically says that \textit{there exists
  a time point i that is greater 
  than the current time point (i.e., in the future), in which the
  resource (the person in charge) is different from the resource at
  the time point $\mathit{i+1}$ (i.e., the next time point), their
  groups are the same, and the next time point is still not later than
  the last time point}.
As for the target expression, some simple examples would be some
strings such as ``Ping-Pong'' and ``Not Ping-Pong''. Based on these,
we can create an example of an \analrule $\arPingPongExample$ as
follows:
\begin{center}
$
  \arPingPongExample = \ruletup{ \condPingPonga \targetarrow \mbox{``Ping-Pong''}, ~ \mbox{``Not Ping-Pong''} },
$
\end{center}
where $\condPingPonga$ is as above.  In this case,
$\arPingPongExample$ specifies a task for predicting the ping-pong
behaviour. In the prediction model creation phase, we will create a
classifier that classifies (partial) traces based on whether they
satisfy $\condPingPonga$ or not (i.e., a trace will be classified into
``Ping-Pong'' if it satisfies $\condPingPonga$, otherwise it will be
classified into ``Not Ping-Pong'').  During the prediction phase, such
classifier can be used to predict whether a given (partial) trace will
lead to ping-pong behaviour or not.
 
The target expression can be more complex than merely a string. For
instance, it can be an expression that involves arithmetic operations
over numeric values such as
%
%

\medskip
\noindent
$\targetremtime =$ \\
\hspace*{10mm}$\eventquery{\last}{time:timestamp}~-~\eventquery{\curr}{time:timestamp}$,
\footnote{Note that, as usual, a timestamp can be represented as
  milliseconds since Unix epoch (i.e., the number of milliseconds that
  have elapsed since Jan 1, 1970 00:00:00 UTC). 
}

\medskip
\noindent
where $\eventquery{\last}{time:timestamp}$ refers to the timestamp
of the last event and $\eventquery{\curr}{time:timestamp}$ refers
to the timestamp of the current event. 
Essentially, the expression $\targetremtime$ computes
\textit{the time difference between the timestamp of the last event
  and the current event (i.e., remaining processing time)}. Then we
can create an \analrule
\[
\arRemainingTimeExample = \ruletup{ \curr < \last \targetarrow \targetremtime, ~0}, 
\]
which specifies a task for predicting the remaining processing time,
because $\arRemainingTimeExample$ maps each (partial) trace into its
remaining processing time. In this case, during the prediction model
creation phase, we will create a regression model 
for predicting the remaining processing time of a given (partial)
trace.
%
\Cref{sec:showcase} provides more examples of prediction tasks
specification using our language.


\subsection{Towards Formalizing the Condition and Target
  Expressions}\label{sec:formalization-cond-target-exp}
%
%
This section is devoted to introduce several components that are
needed to define the language for specifying condition and target
expressions in \Cref{sec:FOE}.

As we have seen in
\Cref{sec:overview-lang}, 
we often need to refer to a particular index of an event within a
trace.
Recall the expression $\eventquery{i\!~\!+\!~\!1}{org:group}$ that refers to
the $\text{org:group}$ attribute value of the event at the index $i\!~\!+\!~\!1$, and
also the expression $\eventquery{\last}{time:timestamp}$ that
refers to the timestamp of the last event. The former requires us to
refer to the event at the index $i\!~\!+\!~\!1$, while the latter requires us to
refer to the last event in the trace.  
To capture this, we introduce the notion of \emph{index
  expression} $\idx$ defined as follows:
\[
\begin{array}{l@{ }c@{ }l}
  \idx &~::=~& i ~\mid~ \nat ~\mid~ \last ~\mid~ \curr ~\mid~ 
\idx_1\!~\!+\!~\!\idx_2 ~\mid~ \idx_1\!~\!-\!~\!\idx_2\\
\end{array}
\]
where
\begin{inparaenum}[\itshape (i)]
\item $i$ is an \emph{index variable}.
\item $\nat$ is a positive integer (i.e., $\nat \in \mathbb{Z}^+$).
\item $\last$ and $\curr$ are special indices in which the former
  refers to the index of the last event in 
  a trace, and the latter refers to the index of the current event
  (i.e., last event of the trace prefix 
  under consideration). For instance, given a $k$-length trace prefix
  $\prefix{k}{\trace}$ of the trace $\trace$, $\curr$ is equal to $k$
  (or $\card{\prefix{k}{\trace}}$), and $\last$ is equal to
  $\card{\trace}$.
\item $\idx + \idx$ and $\idx - \idx$ are the usual arithmetic addition and
  subtraction operations over indices.
\end{inparaenum}

The semantics of index expression is defined over traces and
considered trace prefix length. 
%
Since an index expression can be a variable, 
given a trace $\trace$ and a considered trace prefix length $k$, 
we first introduce a \emph{variable valuation} $\val$, i.e., a
mapping from index variables into $\mathbb{Z}^+$.
%
We assign meaning to index expression by associating to
$\trace$, $k$, and $\val$ an \emph{interpretation function}
$\inter{\cdot}{\trace}{k}{\val}$ which maps an index expression into
$\mathbb{Z}^+$. Formally, $\inter{\cdot}{\trace}{k}{\val}$ is
inductively defined as follows:
\[
\begin{array}{r@{\quad}c@{\quad}l }
\inter{i}{\trace}{k}{\val} &=& \val(i) \\
\inter{\nat}{\trace}{k}{\val} &=& \nat \in \mathbb{Z}^+ \\
\inter{\curr}{\trace}{k}{\val} &=& k \\
\inter{\last}{\trace}{k}{\val} &=& \card{\trace} \\
\inter{\idx_1~+~\idx_2}{\trace}{k}{\val} &=&\inter{\idx_1}{\trace}{k}{\val}~+~\inter{\idx_2}{\trace}{k}{\val} \\
\inter{\idx_1~-~\idx_2}{\trace}{k}{\val} &=& \inter{\idx_1}{\trace}{k}{\val}~-~\inter{\idx_2}{\trace}{k}{\val}
\end{array}
\]
The definition above says that the interpretation
function $\inter{\cdot}{\trace}{k}{\val}$ interprets index expressions
as follows:
\begin{inparaenum}[\itshape (i)]
\item each variable is interpreted based on how the variable valuation
  $\val$ maps the corresponding variable into a positive integer in
  $\mathbb{Z}^+$;
\item each positive integer is interpreted as itself, e.g.,
  $\inter{2603}{\trace}{k}{\val} = 2603$;
\item $\curr$ is interpreted into $k$;
\item $\last$ is interpreted into $\card{\trace}$; and
\item the arithmetic addition/subtraction operators are interpreted as
  usual.  
\end{inparaenum}



To access the value of an event attribute, we introduce
so-called \emph{event attribute accessor}, which is an expression of
the form
\begin{center}
$
  \eventquery{\idx}{attName}
$
\end{center}
where 
$\text{\textit{attName}}$ is an attribute name 
and $\idx$ is an index expression.
To define the semantics of event attribute accessor, 
we extend the definition of our interpretation function
$\inter{\cdot}{\trace}{k}{\val}$ such that it interprets an event
attribute accessor expression into the attribute value of the
corresponding event at the given index. 
Formally, $\inter{\cdot}{\trace}{k}{\val}$ is defined as follows:
\[
  \inter{\eventquery{\idx}{attName}}{\trace}{k}{\val}
  = \left\{ \begin{array}{l@{ \ \ }l}
\attval{attName}(e) &\mbox{{\normalsize if }}
\begin{array}[t]{l}
                                        \inter{\idx}{\trace}{k}{\val}
                                        = i, \\
 1 \leq i \leq \card{\trace},\\\mbox{{\normalsize and }} e = \trace(i) 
\end{array}
              \\
\ & \ \\
\udefined & \mbox{{\normalsize otherwise}}
          \end{array}
\right.
\]
Note that the above definition also says that if the event attribute
accessor refers to an index that is beyond the valid event indices in
the corresponding trace, then we will get undefined value (i.e.,
$\udefined$).


As an example of event attribute accessor, the expression
$\eventquery{i}{org:resource}$ refers to the value of the
attribute $\text{org:resource}$ of the event at the position $i$.

\begin{example}
  Consider the trace $\trace = \tup{e_1, e_2, e_3, e_4, e_5}$, let
  ``Bob'' be the value of the attribute org:resource of the event
  $e_3$ in $\trace$, i.e.,
  $\attval{org:resource}(e_3) = \mbox{``Bob''}$, and $e_3$ does not
  have any attributes named org:group, i.e.,
  $\attval{org:group}(e_3) = \udefined$. In this example, we have that
  $\inter{\eventquery{3}{org:resource}}{\trace}{k}{\val}~=~\mbox{``Bob''}$,
  and
  $\inter{\eventquery{3}{org:group}}{\trace}{k}{\val}~=~\udefined$.
\end{example}

The value of an event attribute within a trace can be either numeric (e.g., 26, 3.86)
or non-numeric (e.g., ``sendOrder''), and we might want to specify
properties that involve arithmetic operations over numeric values. 
%
Thus, we introduce the notion of \emph{numeric expression} and
\emph{non-numeric expression} 
as follows:
\[
\begin{array}{l@{ \ }c@{ \ }l} 
  \nonnumexpb &~::=~& \true ~\mid~ \false
                      ~\mid~ \str ~\mid~ \\
              &&
                 \eventquery{\idx}{NonNumericAttribute}
  \\ \\
  \numexpb &~::=~&\num ~\mid~ \idx ~\mid~ \\ &&
                                                \eventquery{\idx}{NumericAttribute}~\mid~ \\
              &&\numexpb_1 +
                 \numexpb_2 ~\mid~
  \\ &&
        \numexpb_1 -
        \numexpb_2
\end{array}
\]
where 
\begin{inparaenum}[\itshape (i)]
\item $\true$ and $\false$ are the usual boolean values,
\item $\str$ is the usual string (i.e., a sequence of characters),
\item $\num$ is a real number,
\item $\eventquery{\idx}{NonNumericAttribute}$ is
  an event attribute accessor for accessing an attribute with non-numeric
  values, and $\eventquery{\idx}{NumericAttribute}$ is
  an event attribute accessor for accessing an attribute with numeric
  values,
\item $\numexpb_1 + \numexpb_2$ and $\numexpb_1 - \numexpb_2$ are the
  usual arithmetic operations over numeric expressions.
\end{inparaenum}

To give the semantics for \emph{numeric expression} and
\emph{non-numeric expression}, we extend the definition of our
interpretation function $\inter{\cdot}{\trace}{k}{\val}$ by
interpreting $\true$, $\false$, $\str$, and $\num$ as themselves,
e.g., 
\begin{center}
$\inter{3}{\trace}{k}{\val}~=~3$,\\
$\inter{\text{``sendOrder"}}{\trace}{k}{\val}~=~\text{``sendOrder"}$,
\end{center}
and by interpreting the arithmetic operations as usual, e.g.,
\begin{center}
$\inter{26~+~3}{\trace}{k}{\val}~=~\inter{26}{\trace}{k}{\val}~+~\inter{3}{\trace}{k}{\val}~=~26~+~3~=~29$,\\
$\inter{86~-~3}{\trace}{k}{\val}~=~\inter{86}{\trace}{k}{\val}~-~\inter{3}{\trace}{k}{\val}~=~86~-~3~=~83$.
\end{center}
Formally, we extend our interpretation function as follows:
\[
\begin{array}{rcl} 
     \inter{\true}{\trace}{k}{\val}
     &=& \true \\
     \inter{\false}{\trace}{k}{\val}
     &=& \false \\
     \inter{\str}{\trace}{k}{\val}
     &=& \str \\
     \inter{\num}{\trace}{k}{\val}
     &=& \num \\
     \inter{\numexpb_1 + \numexpb_2}{\trace}{k}{\val}
     &=& \inter{\numexpb_1}{\trace}{k}{\val} + \inter{\numexpb_2}{\trace}{k}{\val} \\
  \inter{\numexpb_1 - \numexpb_2}{\trace}{k}{\val}
  &=& \inter{\numexpb_1}{\trace}{k}{\val} - \inter{\numexpb_2}{\trace}{k}{\val}
\end{array}
\]
Note that the value of an event attribute might be undefined, i.e., it
is equal to $\udefined$. In this case, we define that the arithmetic
operations involving $\udefined$ give $\udefined$,
e.g.,~$26~+~\udefined~=~\udefined$.

We now define the notion of \emph{event expression} as a comparison
between either numeric expressions or non-numeric expressions.
Formally, it is defined as follows:
\[
\begin{array}{l@{ \ }r@{ \ }l@{ \ }l}
     \eventexpshort  &~::=~& \true ~\mid~ \false ~\mid~&\\
                     &&\numexpb_1 ~\logeq~ \numexpb_2 &~\mid~ \\
                     &&\numexpb_1 ~\noteq~ \numexpb_2 &~\mid~ \\
                     &&\numexpb_1 ~<~ \numexpb_2 &~\mid~ \\
                     &&\numexpb_1 ~>~ \numexpb_2 &~\mid~ \\
                     &&\numexpb_1 ~\leq~ \numexpb_2 &~\mid~ \\
                     &&\numexpb_1 ~\geq~ \numexpb_2 &~\mid~ \\
                     && \nonnumexpb_1 ~\logeq~ \nonnumexpb_2 & ~\mid~ \\
                     && \nonnumexpb_1 ~\noteq~ \nonnumexpb_2 
             \end{array}
\]
%
%
where 
\begin{inparaenum}[\itshape (i)]
\item $\numexpb$ is a numeric expression;
\item $\nonnumexpb$ is a non-numeric expression;
\item the operators $\logeq$ and $\noteq$ are the usual logical
  comparison operators, namely \emph{equality} and \emph{inequality};
\item the operators $<$, $>$, $\leq$, and $\geq$ are the usual
  arithmetic comparison operators, namely \emph{less than},
  \emph{greater than}, \emph{less than or equal}, and \emph{greater
    than or equal}.
\end{inparaenum}

\begin{example} The expression
\[
  \eventquery{i}{org:resource} \noteq \eventquery{i+1}{org:resource}
\]
is an example of an event expression which says that the resource at
the time point $i$ is different from the resource at the time point
$i+1$. As another example, the expression
\[
\eventquery{i}{concept:name} \logeq \text{``OrderCreated"}
\]
is an event expression saying that the value of the attribute
concept:name of the event at the index $i$ is equal to
$\text{``OrderCreated"}$.
\end{example}

We interpret each logical/arithmetic comparison operator (i.e.,
$\logeq$, $\neq$, $<$, $>$, etc) in the event expressions as
usual. For instance, the expression $26 \geq 3$ is interpreted as
$\true$, while the expression ``receivedOrder'' $\logeq$ ``sendOrder''
is interpreted as $\false$. Additionally, any comparison involving
undefined value ($\udefined$) is interpreted as false. 
It is easy to see how to extend the formal definition of our
interpretation function $\inter{\cdot}{\trace}{k}{\val}$ towards
interpreting event expressions, therefore 
we omit the details.

\subsubsection{Adding Aggregate Functions}\label{sec:add-aggregate-func}

We now extend the notion of \emph{numeric expression} and
\emph{non-numeric expression} by adding several numeric and
non-numeric aggregate functions. A numeric (resp.\ non-numeric)
aggregate function is a function that performs an aggregation
operation over some values and return a numeric (resp.\ non-numeric)
value. Before providing the formal syntax and semantics of our
aggregate functions, in the following we illustrate the needs of
having aggregate functions and we provide some intuition on the shape
of our aggregate functions.

Suppose that each event in each trace has an attribute named
cost. Consider the situation where we want to specify a task for
predicting the total cost of all activities (from the first until the
last event) within a trace. In this case, we need to sum up all values
of the cost attribute in all events. To express this need, we
introduce the aggregate function $\aggsum$ and we can specify the
notion of total cost as follows:
\[
\usuma{\eventquery{x}{cost}}{1}{\last}.
\]
The expression above computes the sum of the values of
$\eventquery{x}{cost}$ for all ${x \in \set{1,\ldots,\last}}$.
In this case $x$ is called \emph{aggregation variable}, the expression
$\eventquery{x}{cost}$ specifies the \emph{aggregation source}, i.e.,
the source of the values to be aggregated, and the expression
${x = 1 : \last}$ specifies the \emph{aggregation range} by
defining the range of the aggregation variable~$x$.

In some situation, we might only be interested to compute the total
cost of a certain activity.
E.g., the total cost of all validation activities within a trace. To
do this, we introduce the notion of \textit{aggregation condition},
which allows us to select only some values that we want to
aggregate. For example, the expression
\[
\sumb{\eventquery{x}{cost}}{1}{\last}{\eventquery{x}{concept:name} \logeq \text{"Validation"}}.
\]
computes the sum of the values of the attribute $\eventquery{x}{cost}$
for all ${x \in \set{1,\ldots,\last}}$ in which 
the expression
\[
\eventquery{x}{concept:name} \logeq \text{"Validation"}
\]
is evaluated to true. Therefore, the summation only considers the values
of $x$ in which the activity name is $\text{"Validation"}$, 
and we only compute the total cost of all validation activities.
As before, $\eventquery{x}{cost}$ specifies the source of the values
to be aggregated, the expression ${x = 1 : \last}$ specifies the
\emph{aggregation range} by defining the range of the aggregation
variable $x$, and the expression
$
\eventquery{x}{concept:name}~\logeq~\text{"Validation"}
$
provides the \emph{aggregation condition}.

The expression for specifying the source of the values to be
aggregated can be more complex, for example when we want to compute
the average activity running time within a trace. In this case, the
running time of an activity is specified as the time difference
between the timestamp of that activity and the next activity, i.e.,
\[
  \eventquery{x+1}{time:timestamp} - \eventquery{x}{time:timestamp}.
\] 
Then, the average activity running time can be specified as follows:
\[
\uavgc{\eventquery{x+1}{time:timestamp} - \eventquery{x}{time:timestamp}}{1}{\last}.
\]
Essentially, the expression above computes the average of the time
difference between the activity at the timepoint ${x+1}$ and $x$,
where ${x \in \set{1,\ldots, \last}}$.

In other cases, we might not be interested in aggregating the
data values but we are interested in counting the number of a certain
activity/event. 
To do this, we introduce the aggregate function $\aggcount$. As an
example, we can specify an expression to count the number of
validation activities within a trace as follows:
\[
\countb{\eventquery{x}{concept:name} \logeq \text{"validation"}}{1}{\last}
\]
where $\eventquery{x}{concept:name}~\logeq~\text{"validation"}$ is an
\emph{aggregation condition}. The expression above counts
how many times the specified aggregation condition is true within the
specified range. Thus, in this case,
it counts the number of the events between the first and the last
event, in which the activity name is $\text{"validation''}$.

We might also be interested in counting the number of different values
of a certain attribute within a trace. For example, we might be
interested in counting the number of different resources that are
involved within a trace. To capture this, we introduce the aggregate
function $\aggcountval$. We can then specify the expression to count
the number of different resources between the first and the last event
as follows:
\[
\countvala{org:resource}{1}{\last}
\]
where 
\begin{inparaenum}[\itshape (i)]
\item org:resource is the name of the attribute in which we want to
  count its number of different values; and
\item the expression ``$\aggwithin~1~:~\last$'' is the \emph{aggregation
    range}.
\end{inparaenum}








We will see later in \Cref{sec:showcase} that the presence of
aggregate functions allows us to express numerous interesting
prediction tasks.
Towards formalizing the aggregate functions, we first formalize the
notion of \emph{aggregation conditions}. An aggregation condition is
an unquantified First Order Logic (FOL)~\cite{Smul68} formula where
the atoms are event expressions and may use only a single unquantified
variable, namely the aggregation variable. The values of the
unquantified/free variable in aggregation conditions is ranging over
the specified aggregation range in the corresponding aggregate
function. Formally \emph{aggregation conditions} are defined as
follows:
\[
\begin{array}{l}
  \aggcondexp ~::=~ \eventexpshort ~\mid~ \neg \psi ~\mid~ \psi_1
  \wedge \psi_2 ~\mid~ \psi_1 \vee \psi_2 
\end{array}
\]
where $\eventexpshort$ is an event expression, and the semantics of
$\aggcondexp$ is based on the usual FOL semantics. 
Formally, we extend the definition of our interpretation function
$\inter{\cdot}{\trace}{k}{\val}$ as follows:
\[
\begin{array}{l@{ \ }l@{ \ \ }c@{\ }l}
  \satisfyb{\trace}{k}{\val}{\neg\psi}&= \true, &\mbox{{\normalsize if }}&
                                                                \satisfyb{\trace}{k}{\val}{\psi}
  = \false\\
  \satisfyb{\trace}{k}{\val}{\psi_1 \fandcompact \psi_2}&= \true, &\mbox{{\normalsize
                                                                          if }}&
                                                                                 \satisfyb{\trace}{k}{\val}{\psi_1}
                                                                                 =
                                                                                 \true
                                                                                 \mbox{{\normalsize,
                                                                                 and
                                                                                 }}\satisfyb{\trace}{k}{\val}{\psi_2}
                                                   =
                                                   \true\\
  \satisfyb{\trace}{k}{\val}{\psi_1 \forcompact \psi_2}&= \true, &\mbox{{\normalsize
                                                                         if }}&
                                                                                \satisfyb{\trace}{k}{\val}{\psi_1}
                                                                                =
                                                                                \true
                                                                                \mbox{{\normalsize,
                                                                                or
                                                                                \
                                                                                }}
                                                                                \satisfyb{\trace}{k}{\val}{\psi_2}
                                                                                =
                                                                                \true\\
\end{array}
\]

With this machinery in hand, we are ready to define the syntax and the
semantics of numeric and non-numeric aggregate functions.  We first
extend the syntax of the numeric and non-numeric expressions by adding
the \emph{numeric and non-numeric aggregate functions} as follows:



\medskip
\noindent
$  \nonnumexpb ~::=~ $\\
\begin{tabular}{@{\quad}l}
  $\true ~\mid~ \false ~\mid~ \str ~\mid~ \eventquery{\idx}{NonNumericAttribute} ~\mid~$\\
  $\concata{\aggnonnumsrc}{\st}{\ed}{\aggcondexp}$
\end{tabular}

\medskip
\noindent
$  \numexpb ~::=~$$\num$ $~\mid~$ \hspace*{0.7mm}$\idx$\hspace*{0.7mm} $~\mid~$
$\eventquery{\idx}{NumericAttribute}$\hspace*{1.7mm} $\mid~$\\
\begin{tabular}{@{\quad}l@{\ }l@{}l}
  $\numexpb_1~+~\numexpb_2$&$~\mid~\numexpb_1~-~\numexpb_2$&$\mid~$\\
  \multicolumn{2}{@{\quad}l}{$\suma{\aggnumsrc}{\st}{\ed}{\aggcondexp}$}&$\mid~$\\
  \multicolumn{2}{@{\quad}l}{$\avga{\aggnumsrc}{\st}{\ed}{\aggcondexp}$}&$\mid~$\\
  \multicolumn{2}{@{\quad}l}{$\mina{\aggnumsrc}{\st}{\ed}{\aggcondexp}$}&$\mid~$\\
  \multicolumn{2}{@{\quad}l}{$\maxa{\aggnumsrc}{\st}{\ed}{\aggcondexp}$}&$\mid~$\\
  $\mintwo{\numexpb_1}{\numexpb_2}$&$~\mid~\maxtwo{\numexpb_1}{\numexpb_2}$&$\mid~$\\
  \multicolumn{2}{@{\quad}l}{$\counta{\aggcondexp}{\st}{\ed}$}&$\mid~$\\
  \multicolumn{2}{@{\quad}l}{$\countvala{attName}{\st}{\ed}$}
\end{tabular}



\medskip
\noindent
where 
\begin{inparaenum}[\itshape (i)]
%
\item $\num$, $\idx$, $\eventquery{\idx}{NumericAttribute}$,
  $\eventquery{\idx}{NonNumericAttribute}$, ${\numexpb_1+\numexpb_2}$,
  and ${\numexpb_1 - \numexpb_2}$ are as before;
\item $\st$ and $\ed$ are either positive integers (i.e.,
  $\st \in \mathbb{Z}^+$ and $\ed \in \mathbb{Z}^+$) or special
  indices (i.e., $\last$ or $\curr$), and $\st \leq \ed$;
\item $x$ is a variable called \emph{aggregation variable}, and the
  range of its value is between $\st$ and $\ed$ (i.e.,
  ${\st \leq x \leq \ed}$). The expression ${\aggrange~x=\st : \ed}$ as
  well as ${\aggwithin~x = \st : \ed}$ are called \emph{aggregation
    variable range};
%
%
\item $\aggnumsrc$ and $\aggnonnumsrc$ specify the source of the
  values to be aggregated. The $\aggnumsrc$ is specified as numeric
  expression while $\aggnonnumsrc$ is specified as non-numeric
  expression. Both of them may and can only use the corresponding
  aggregation variable $x$, and they cannot contain any aggregate
  functions;
%
%
\item \aggcondexp is an \emph{aggregation condition} over the corresponding
  aggregation variable $x$ and no other variables are allowed to
  occur in \aggcondexp;
%
\item $\text{attName}$ is an attribute name;
\item For the aggregate functions, as the names describe, \aggsum
  stands for summation, \aggavg stands for average, \aggmin stands for
  minimum, \aggmax stands for maximum, \aggcount stands for counting,
  \aggcountval stands for counting values, and \aggconcat stands for
  concatenation. The behaviour of these aggregate functions is quite
  intuitive. Some intuition has been given previously and we explain
  their details behaviour while providing their formal semantics
  below. The aggregate functions \aggsum, \aggavg, \aggmin, \aggmax,
  \aggconcat that have aggregation conditions \aggcondexp are also
  called \emph{conditional aggregate functions}.
%
\end{inparaenum}


Notice that a numeric aggregate function is also a numeric expression
and a numeric expression is also a component of a numeric aggregate
function (either in the source value or in the aggregation
condition). Hence, it may create some sort of nested aggregate
function.  However, to simplify the presentation, in this work we do
not allow nested aggregation functions of this form, but technically
it is possible to do that under a certain care on the usage of the
variables (Similarly for the non-numeric aggregate function).


To formalize the semantics of aggregate functions, we first introduce
some notations. Given a variable valuation $\val$, we write
$\val[x \mapsto d]$ to denote a new variable valuation obtained from
the variable valuation $\val$ as follows:
\[
\val[x \mapsto d](y)
  = \left\{ \begin{array}{l@{ \qquad}l}
              d &\mbox{{\normalsize if  }} y = x\\
              \val(y) &\mbox{{\normalsize if }} y \neq x
          \end{array}
\right.
\]
Intuitively, $\val[x \mapsto d]$ substitutes each variable $x$ with
$d$, while the other variables (apart from $x$) are substituted the
same way as $\val$ is defined. Given a conditional summation aggregate
function 
\[
\suma{\aggnumsrc}{\st}{\ed}{\aggcondexp},
\]
a trace $\trace$, a considered trace prefix length $k$, and a variable
valuation $\val$, we define its corresponding \emph{set $\idxset$ of
  valid aggregation indices} as follows:
\[
\begin{array}{ll}
  \idxset = \set{ d \in \mathbb{Z}^+~\mid~ &\st \leq d \leq \ed, \inter{\aggcondexp
                                              }{\trace}{k}{\val[x~\mapsto~d]}
                                              = \true, \\
                                            &\text{and } \inter{\aggnumsrc
                                              }{\trace}{k}{\val[x~\mapsto~d]} \neq \udefined}.
\end{array}
\]
basically, $\idxset$ collects the values within the given aggregation
range (i.e., between $\st$ and $\ed$), in which, by substituting the
aggregation variable $x$ with those values, the aggregation condition
$\aggcondexp$ is evaluated to $\true$ and $\aggnumsrc$ is not
evaluated to undefined value $\udefined$.
For the other conditional aggregate functions $\aggavg$, $\aggmax$,
$\aggmin$, and $\aggconcat$,
the corresponding set of valid aggregation indices can be defined
similarly. 
%

\begin{example}\label{ex:idx-set}
  Consider the trace $\trace~=~\tup{e_1, e_2, e_3, e_4}$, let
  $\text{"validation"}$ be the value of the attribute
  $\text{concept:name}$ of the event $e_2$ and $e_4$ in $\trace$,
  i.e.,
  \begin{center}
    $\attval{concept:name}(e_2) = \attval{concept:name}(e_4) =
    \text{"validation"}$.
  \end{center}
  Moreover, let $\attval{concept:name}(e_1)~=~\text{"initialization"}$
  and $\attval{concept:name}(e_3)~=~\text{"assembling"}$.  Suppose
  that the cost of each activity is the same, let say it is equal to
  3, i.e.,
  \begin{center}
    $\attval{cost}(e_1) = \attval{cost}(e_2) = \attval{cost}(e_3) =
    \attval{cost}(e_4) = 3$,
  \end{center}
  and we have the following aggregate function specification:
  \[
    \suma{\eventquery{x}{cost}}{1}{\last}{\true},
  \]
  and
  \[
    \sumb{\eventquery{x}{cost}}{1}{\last}{\eventquery{x}{cost} \logeq
      \text{"validation"}}
  \]
  The former computes the total cost of all activities while the
  latter computes the total cost of validation activities. In this
  case, the corresponding set of the valid aggregation indices (with
  respect to the given trace $\trace$) for the first aggregate
  function is $\idxset_1 = \set{1, 2, 3, 4}$, while 
  for the second aggregate function we have $\idxset_2 = \set{2, 4}$
  because the second aggregate function requires that the activity
  name (i.e., the value of the attribute $\text{concept:name}$) to be
  equal to $\text{"validation"}$ and it is only true when $x$ is equal
  to either 2 or 4.
\end{example}

Having this machinery in hand, we are now ready to formally define the
semantics of aggregate functions.  The formal semantics of the
conditional aggregate functions $\aggsum$, $\aggavg$, $\aggmax$,
$\aggmin$
is provided in
\Cref{fig:semantics-aggregate-function-conditional}. Intuitively, the
aggregate function $\aggsum$ computes
%
%
the sum of the values that are obtained from the evaluation of the
specified numeric expression $\aggnumsrc$ over the specified
aggregation range (i.e., between $\st$ and $\ed$). Additionally, the
computation of the summation ignores undefined values and it only
considers those indices within the specified aggregation range in
which the aggregation condition is evaluated to true. The intuition
for the aggregate functions $\aggavg$, $\aggmax$, $\aggmin$ is
similar, except that $\aggavg$ computes the average, $\aggmax$
computes the maximum values, and $\aggmin$ computes the minimum
values.

\begin{example}\label{ex:conditional-aggregation-semantics}
  Continuing \Cref{ex:idx-set}, the first aggregate function is
  evaluated to 12 because we have that $\idxset_1 = \set{1,2,3,4}$,
  and
\[
\begin{array}{lll}
  \sum\limits_{d \in \idxset_1}
  \inter{\eventquery{x}{cost}}{\trace}{k}{\val[x~\mapsto~d]} &=& 
                                                                 \inter{\eventquery{x}{cost}}{\trace}{k}{\val[x~\mapsto~1]}~+\\
                                                             &&\inter{\eventquery{x}{cost}}{\trace}{k}{\val[x~\mapsto~2]}~+\\
                                                             &&\inter{\eventquery{x}{cost}}{\trace}{k}{\val[x~\mapsto~3]}~+\\
                                                             &&\inter{\eventquery{x}{cost}}{\trace}{k}{\val[x~\mapsto~4]} \\
                                                             &=& 12.
\end{array}
\]
On the other hand, the second aggregate
function is evaluated to 6 because we have that $\idxset_2 =
\set{2,4}$, and 
\[
\begin{array}{lll}
  \sum\limits_{d \in \idxset_2}
  \inter{\eventquery{x}{cost}}{\trace}{k}{\val[x~\mapsto~d]} &=& 
                                                             \inter{\eventquery{x}{cost}}{\trace}{k}{\val[x~\mapsto~2]}~+\\
                                                             &&\inter{\eventquery{x}{cost}}{\trace}{k}{\val[x~\mapsto~4]} \\
                                                             &=& 6
\end{array}
\]
\end{example}

\begin{figure*}[btp]
\centering
{
\begin{tabular}{|r@{ \ }c@{ \ }l|}
\hline
&& \\[-1.5ex]
$\inter{\suma{\aggnumsrc}{\st}{\ed}{\aggcondexp}}{\trace}{k}{\val}$ &=& 
$\left\{ \begin{array}{l@{\qquad \qquad \ \qquad \qquad}l}
\sum\limits_{d \in \idxset} \inter{\aggnumsrc
  }{\trace}{k}{\val[x~\mapsto~d]}&\text{if  }
                                                               \idxset
           \neq \emptyset\\[2ex]
              \udefined &\mbox{otherwise}
          \end{array}
\right.$\\[3.9ex] 
$\inter{\avga{\aggnumsrc}{\st}{\ed}{\aggcondexp}}{\trace}{k}{\val}$ &=& 
$\left\{ \begin{array}{l@{\qquad\ \ \ \quad}l}
\left(\sum\limits_{d \in \idxset} \inter{\aggnumsrc
  }{\trace}{k}{\val[x~\mapsto~d]}\right)~/~\card{\idxset}&\text{if }
                                                               \idxset\neq \emptyset\\[2ex]
              \udefined &\mbox{otherwise}
          \end{array}
\right.$\\[3.9ex] 
  $ \inter{\maxa{\aggnumsrc}{\st}{\ed}{\aggcondexp}}{\trace}{k}{\val}$
                                                                    &=& 
$\left\{ \begin{array}{l@{\qquad}l}
           \inter{\aggnumsrc
           }{\trace}{k}{\val[x~\mapsto~i]}
           &\mbox{if there exists } i \in \idxset
             \mbox{ such that}\\
           &\text{for each } j \in \idxset \text{ we have that  }  \\
           &{\scriptsize \inter{\aggnumsrc
           }{\trace}{k}{\val[x~\mapsto~i]}
           \geq \inter{\aggnumsrc
           }{\trace}{k}{\val[x~\mapsto~j]} } \\[1.5ex]
              \udefined &\mbox{otherwise}
          \end{array}
\right.$\\[6.5ex] 
  $ \inter{\mina{\aggnumsrc}{\st}{\ed}{\aggcondexp}}{\trace}{k}{\val}$
                                                                    &=& 
$\left\{ \begin{array}{l@{\qquad}l}
           \inter{\aggnumsrc
           }{\trace}{k}{\val[x~\mapsto~i]}
           &\mbox{if there exists } i \in \idxset
             \mbox{ such that}\\
           &\text{for each } j \in \idxset \text{ we have that  }  \\
           &{\scriptsize \inter{\aggnumsrc
           }{\trace}{k}{\val[x~\mapsto~i]}
           \leq \inter{\aggnumsrc
           }{\trace}{k}{\val[x~\mapsto~j]} } \\[1.5ex]
              \udefined &\mbox{otherwise}
          \end{array}
\right.$\\[6.5ex] 
  \multicolumn{3}{|l|}{where \ 
  $\idxset = \set{ d \in \mathbb{Z}^+ \mid \st \leq d \leq
  \ed,~\inter{\aggcondexp }{\trace}{k}{\val[x~\mapsto~d]} = \true,~
  \text{ and }~\inter{\aggnumsrc }{\trace}{k}{\val[x~\mapsto~d]} \neq
  \udefined}$.
  }\\[1.5ex]
\hline
\end{tabular}
}
\caption{Formal Semantics of Aggregate Functions $\aggsum$,
  $\aggavg$, $\aggmax$, $\aggmin$ in the presence of aggregation conditions}
\label{fig:semantics-aggregate-function-conditional}
\end{figure*}

\begin{figure*}[btp]
\centering
{
\begin{tabular}{@{}|l|@{}}
\hline
\ \\[-1.5ex]
$
\begin{array}{l}
\inter{\concata{\aggnonnumsrc}{\st}{\ed}{\aggcondexp}}{\trace}{k}{\val}
  =\\
\qquad\left\{ \begin{array}{l@{\quad}l}
                \emptystr&\mbox{{\normalsize
                                                                                     if }}\ \st > \ed
                \\[1ex]
                \inter{\aggnonnumsrc}{\trace}{k}{\val[x~\mapsto~\st]}\
                \concatoperand 
                \inter{\concata{\aggnonnumsrc}{\st +
                1}{\ed}{\aggcondexp}}{\trace}{k}{\val}
                         &\mbox{{\normalsize
                           if }}\ \st \leq \ed,\\
                         &\quad\inter{\aggnonnumsrc}{\trace}{k}{\val[x~\mapsto~\st]}
                           \neq \udefined\\
                         &\quad\mbox{and } \inter{\aggcondexp}{\trace}{k}{\val[x~\mapsto~\st]}
                           =
                           \true\\[1ex]
                \emptystr\concatoperand \inter{\concata{\aggnonnumsrc}{\st
                + 1}{\ed}{\aggcondexp}}{\trace}{k}{\val}&\mbox{{\normalsize
                                                          if }}\ \st \leq \ed
                                                          \mbox{ and either}
                \\
                         &\quad \inter{\aggnonnumsrc}{\trace}{k}{\val[x~\mapsto~\st]}
                           = \udefined\\
                                                          &\quad
                                                            \mbox{or }\inter{\aggcondexp}{\trace}{k}{\val[x~\mapsto~\st]}
                                                          = \false                \\[1ex]
%
%
              \udefined &\mbox{otherwise}
          \end{array}
\right.
\end{array}
$                       
\\[3ex]
  where $\concatoperand$ is a concatenation operator that simply concatenates
  two non-numeric values.\\[1ex]
  \hline
\end{tabular}
}
\caption{Formal Semantics of the Aggregate Function $\aggconcat$}
\label{fig:semantics-aggregate-function-concat}
\end{figure*}

The aggregate function $\maxtwo{\numexpb_1}{\numexpb_2}$ computes the
maximum value between the two values that are obtained by evaluating
the specified two numeric expressions $\numexpb_1$ and
$\numexpb_2$. It gives undefined value $\udefined$ if one of them is
evaluated to undefined value $\udefined$
%
%
(Similarly for the aggregate function
$\mintwo{\numexpb_1}{\numexpb_2}$ except that it computes the minimum
value).
Formally, the semantics of these functions is defined as follows:
\[
\begin{array}{l}
  \inter{\mintwo{\numexpb_1}{\numexpb_2}}{\trace}{k}{\val} =\\
  \quad \quad \left\{ \begin{array}{l@{ \ \ }l}
                         \inter{\numexpb_1}{\trace}{k}{\val} &\mbox{{\normalsize if  }}
                                                               \inter{\numexpb_1}{\trace}{k}{\val}
                                                               \leq \inter{\numexpb_2}{\trace}{k}{\val}\\
                         \inter{\numexpb_2}{\trace}{k}{\val} &\mbox{{\normalsize if  }}                                                     
                                                               \inter{\numexpb_1}{\trace}{k}{\val}
                                                               > \inter{\numexpb_2}{\trace}{k}{\val}\\
              \udefined &\mbox{otherwise}
          \end{array}
\right.
\end{array}
\]

\[
\begin{array}{l}
  \inter{\maxtwo{\numexpb_1}{\numexpb_2}}{\trace}{k}{\val} =\\
  \quad\quad \left\{ \begin{array}{l@{ \ \ }l}
                         \inter{\numexpb_1}{\trace}{k}{\val} &\mbox{{\normalsize if  }}
                                                               \inter{\numexpb_1}{\trace}{k}{\val}
                                                               \geq \inter{\numexpb_2}{\trace}{k}{\val}\\
                         \inter{\numexpb_2}{\trace}{k}{\val} &\mbox{{\normalsize if  }}                                                     
                                                               \inter{\numexpb_1}{\trace}{k}{\val}
                                                               < \inter{\numexpb_2}{\trace}{k}{\val}\\
              \udefined &\mbox{otherwise}
          \end{array}
\right.
\end{array}
\]

The formal semantics of the aggregate function $\aggcount$ is provided below
\[
\begin{array}{l}
 \inter{\counta{\aggcondexp}{\st}{\ed}}{\trace}{k}{\val} =\\
\quad \ \card{\set{ d \in \mathbb{Z}^+ \mid \st \leq d \leq \ed \text{, and }
    \inter{\aggcondexp }{\trace}{k}{\val[x~\mapsto~d]} = \true}}
\end{array}
\]
Intuitively, 
it counts how many times the $\aggcondexp$ is evaluated to true within
the given range, i.e., between $\st$ and $\ed$. This aggregate
function is useful to count the number of events/activities within a
certain range that satisfy a certain condition. For example, to count
the number of the activity named ``\textit{modifying delivery appointment}''
within a certain range in a trace.

The semantics of the aggregate function $\aggcountval$ is formally
defined as follows:
\[
\begin{array}{l}
  \inter{ \countvala{attName}{\st}{\ed} }{\trace}{k}{\val} =\\
  \qquad \ \card{\set{ v \mid 
  \inter{\eventquery{x}{attName}}{\trace}{k}{\val[x~\mapsto~d]} = v
  \text{, and }\st \leq d \leq \ed }},
\end{array}
\]
intuitively, it counts the number of all possible values of the
attribute $\text{attName}$ within all events between the given start
and end timepoints (i.e., between $\st$ and $\ed$).

The aggregate function $\aggconcat$ concatenates the values that are
obtained from the evaluation of the given non-numeric expression under
the valid aggregation range (i.e., we only consider the value within
the given aggregation range in which the aggregation condition is
satisfied). Moreover, the concatenation ignores undefined values and
treats them as empty string. The formal semantics of the aggregate
function $\aggconcat$ is provided in
\Cref{fig:semantics-aggregate-function-concat}.

Notice that, for convenience, we could easily extend our language with
unconditional aggregate functions by adding the following:\\[1.5ex]
\hspace*{7mm}$\usuma{\aggnumsrc}{\st}{\ed}$\\
\hspace*{7mm}$\uavga{\aggnumsrc}{\st}{\ed}$\\
\hspace*{7mm}$\umina{\aggnumsrc}{\st}{\ed}$\\
\hspace*{7mm}$\umaxa{\aggnumsrc}{\st}{\ed}$\\
\hspace*{7mm}$\uconcata{\aggnonnumsrc}{\st}{\ed}$\\[1.5ex]
%
In this case, they simply perform an aggregation computation over the
values that are obtained by evaluating the specified
numeric/non-numeric expression over the specified aggregation range.
However, they do not give additional expressive power since they are
only syntactic variant of the current conditional aggregate
functions. This is the case because we can simply put ``$\true$'' as
the aggregation condition,
e.g., $\suma{\aggnumsrc}{\st}{\ed}{\true}$. Based on their
semantics, 
we get the aggregate functions that behave as unconditional aggregate
functions. I.e., they ignore the aggregation condition since it will
always be true for every values within the specified aggregation
range. In the following, for the brevity of presentation, when
aggregation condition is not important we often simply use the
unconditional version of aggregate functions.

\subsection{\langname (\langnameabr)}\label{sec:FOE}

Finally, we are ready to define the language for specifying condition
expression, 
namely \langname (\langnameabr). A part of this language is also used
to specify target expression. 

An \langnameabr formula is a First Order Logic (FOL)~\cite{Smul68}
formula where the atoms are event expressions and the quantification
is ranging over event indices. Syntactically \langnameabr is defined
as follows:
\[
\begin{array}{l@{ }c@{ }c@{}c@{}c@{}c@{}c@{}c@{}c}
  \varphi &~::=~ &  \eventexpshort &~\mid~& \neg \varphi &~\mid~&
                                                                  \fforall
                                                                  i. \varphi
  &~\mid~& \ \ \fexists
           i. \varphi
           \
           \ 
           ~\mid~\\
          && \varphi_1 \wedge \varphi_2 &~\mid~& \varphi_1 \vee \varphi_2 &~\mid~& \varphi_1 \ra
                                                                                   \varphi_2 &&
\end{array}
\]
\noindent
Where
\begin{inparaenum}[\itshape (i)]
\item \eventexpshort is an event expression;
\item $\neg \varphi$ is negated \langnameabr formula;
\item $\fforall i. \varphi$ is an \langnameabr formula where the variable $i$ is
  universally quantified;
\item $\fexists i. \varphi$ is an \langnameabr formula where the
  variable $i$ is existentially quantified;
\item $\varphi_1 \wedge \varphi_2$ is a conjunction of \langnameabr formulas;
\item $\varphi_1 \vee \varphi_2$ is a disjunction of \langnameabr formulas;
\item $\varphi_1 \ra \varphi_2$ is an \langnameabr implication formula saying that
  $\varphi_1$ implies $\varphi_2$;
\item The notion of free and bound variables is as usual in FOL,
  except that the variables inside aggregate functions, i.e.,
  aggregation variables, are not considered as free variables;
\item The aggregation variables cannot be existentially/universally
  quantified.
\end{inparaenum}

The semantics of \langnameabr constructs is based on the usual FOL
semantics.
Formally, 
%
we extend the definition of our interpretation function
$\inter{\cdot}{\trace}{k}{\val}$
as follows\footnote{We assume that variables are standardized apart,
  i.e., no two quantifiers bind the same variable (e.g.,
  $\fforall i . \fexists i . (i > 3)$), and no variable occurs both
  free and bound (e.g., $(i > 5) \fand \fexists i . (i > 3)$). As
  usual in FOL, every \langnameabr formula can be transformed into a
  semantically equivalent formula where the variables are standardized
  apart by applying some variable renaming~\cite{Smul68}.}:

\[
\begin{array}{l@{ \ }l@{ \ \ }c@{\ }l}
%
%
  \satisfyb{\trace}{k}{\val}{\neg\varphi}&= \true, &\mbox{{\normalsize if }}&
                                                                \satisfyb{\trace}{k}{\val}{\varphi}
  = \false\\
  \satisfyb{\trace}{k}{\val}{\varphi_1 \fandcompact \varphi_2}&= \true, &\mbox{{\normalsize
                                                                          if }}&
                                                                                 \satisfyb{\trace}{k}{\val}{\varphi_1}
                                                                                 =
                                                                                 \true
                                                                                 \mbox{{\normalsize,
                                                                                 and
                                                                                 }}\satisfyb{\trace}{k}{\val}{\varphi_2}
                                                   =
                                                   \true\\
  \satisfyb{\trace}{k}{\val}{\varphi_1 \forcompact \varphi_2}&= \true, &\mbox{{\normalsize
                                                                         if }}&
                                                                                \satisfyb{\trace}{k}{\val}{\varphi_1}
                                                                                =
                                                                                \true
                                                                                \mbox{{\normalsize,
                                                                                or
                                                                                \
                                                                                }}
                                                                                \satisfyb{\trace}{k}{\val}{\varphi_2}
                                                                                =
                                                                                \true\\

  \satisfyb{\trace}{k}{\val}{\varphi_1 \fimplcompact \varphi_2}&= \true, &\mbox{{\normalsize
                                                                           if }}&
                                                                                  \satisfyb{\trace}{k}{\val}{\varphi_1}
                                                                                  =
                                                                                  \true
                                                                                  \mbox{{\normalsize,
                                                                                  implies
                                                                                  \
                                                                                  }}\\
                                               &&&   \satisfyb{\trace}{k}{\val}{\varphi_2}
                                                   =
                                                   \true\\
   \satisfyb{\trace}{k}{\val}{\exists i. \varphi}&= \true,
                                                        &\mbox{{\normalsize
                                                          if }}& \mbox{{\normalsize for some }} c
                                                                 \in \set{1,
                                                                 \ldots,
                                                                 \card{\trace}}\mbox{{\normalsize, we
                                                                 have }}  \\
                                               &&&
                                                   \mbox{{\normalsize
                                                   that }}
                                                   \satisfyb{\trace}{k}{\val[i~\mapsto~c]}{\varphi} = \true\\
%
%
   \satisfyb{\trace}{k}{\val}{\forall i. \varphi} &= \true,&\mbox{{\normalsize if }}& \mbox{{\normalsize for every }} c
                                                                                     \in \set{1,
                                                                                     \ldots,
                                                                                     \card{\trace}}
                                                                                     \mbox{{\normalsize, we have
                                                                                     }}   \\
                                        &&&\mbox{{\normalsize
                                            that }}
                                                   \satisfyb{\trace}{k}{\val[i~\mapsto~c]}{\varphi}
                                                   = \true
\end{array}
\]




\noindent
As before, $\val[i \mapsto c]$ substitutes each variable $i$ with
$c$, while the other variables are substituted the same way as $\val$
is defined.
%
%
When $\varphi$ is a closed formula, its truth value does not depend on
the valuation of the 
variables, and we denote the interpretation of $\varphi$ simply by
$\inter{\varphi}{\trace}{k}{}$.
%
%
We also say that the trace $\trace$ and the prefix length $k$ satisfy
$\varphi$, written $\pref{k}{\trace} \models \varphi$, if
$\inter{\varphi}{\trace}{k}{} = \true$. With a little abuse of
notation, sometimes we also say that the \emph{$k$-length trace prefix
  $\prefix{k}{\trace}$ of the trace $\trace$ satisfies $\varphi$},
written $\prefix{k}{\trace} \models \varphi$, if
$\pref{k}{\trace} \models \varphi$.

\begin{example} \label{ex:foe-expressive-example}
An example of a closed \langnameabr formula is as follows:
\[
\begin{array}{l@{}l}
  \fforall i.&(\eventquery{i}{concept:name} \logeq \text{``OrderCreated"}
                \fimpl \\ 
              &\quad\begin{array}{l@{}l@{}l}
                      \fexists j.&(&j > i \fand 
                                      \eventquery{i}{orderID} \logeq
                                      \eventquery{j}{orderID} \fand  \\
                                    &&\eventquery{j}{concept:name} \logeq
                                      \text{``OrderDelivered"} \fand \\
                                    &&(\eventquery{j}{time:timestamp} -
                                      \eventquery{i}{time:timestamp}) \leq \\
                                    &&10.800.000\\ 
                                    &)&\\
\end{array}\\
 &)
\end{array}
\]

\noindent
which essentially says that \textit{whenever there is an event where
  an order is created, eventually there will be an event where the
  corresponding order is delivered and the time difference between the
  two events (the processing time) is less than or equal to 10.800.000
  milliseconds (3 hours)}.
\end{example}


In general, \langnameabr has the following main features:
\begin{inparaenum}[\itshape (i)]
%
\item it allows us to specify constraints over the data (attribute values);
\item it allows us to (universally/existentially) quantify different
  event time points and to compare different event attribute values at
  different event time points;
\item it allows us to specify arithmetic expressions/operations
  involving the data as well as aggregate functions;
%
%
\item it allows us to do selective aggregation operations (i.e.,
  selecting the values to be aggregated).
\item the fragments of \langnameabr, namely the numeric and
  non-numeric expressions, allow us to specify the way 
  to compute a certain value (We will see later that it is needed to
  specify how to compute the target value).
%
\end{inparaenum}


\subsubsection{Checking Whether a Closed \langnameabr Formula is Satisfied}\label{sec:checking-foe-satisfaction}
%

We now proceed to introduce several properties of \langnameabr
formulas that are useful for checking whether a trace $\trace$ and a
prefix length $k$ satisfy a closed \langnameabr formula
$\varphi$, 
i.e., to check whether $\pref{k}{\trace} \models \varphi$. This check
is needed when we create the prediction model based on the
specification of prediction task provided by an \analrule.

Let $\varphi$ be an \langnameabr formula, we write
$\varphi[i \mapsto c]$ to denote a new formula obtained by
substituting each variable $i$ in $\varphi$ by $c$. In the following,
\Cref{thm:exist-elimination,thm:forall-elimination} show that, while
checking whether a trace $\trace$ and a prefix length $k$ satisfy a
closed \langnameabr formula $\varphi$, we can eliminate the presence
of existential and universal quantifiers. 

\begin{theorem}\label{thm:exist-elimination}
  Given a closed \langnameabr formula $\fexists i.\varphi$, a trace $\tau$ and a
  prefix length $k$,
\begin{center}
$\pref{k}{\trace} \models \fexists i.\varphi 
\mbox{ \ iff \ }
%
  \pref{k}{\trace} \models \bigvee_{c \in \set{1, \ldots \card{\trace}}} \varphi[i \mapsto c]
$\end{center}
\end{theorem}
\begin{proof}
  By the definition of the semantics of \langnameabr, we have that
  $\trace$ and $k$ satisfy $\fexists i.\varphi$ (i.e.,
  $\pref{k}{\trace}~\models~\fexists i.\varphi $) iff there exists an
  index $c \in \set{1, \ldots, \card{\trace}}$, such that~$\trace$
  and~$k$ satisfy the formula~$\psi$ that is obtained from $\varphi$
  by substituting each variable~$i$ in $\varphi$ with~$c$ (i.e.,
  $\pref{k}{\trace}~\models~\psi$ where $\psi$ is
  $\varphi[i \mapsto c]$)
  Thus, it is the same as satisfying the disjunctions of formulas that
  is obtained by considering all possible substitutions of the
  variable $i$ in $\varphi$ by all possible values of $c$
  (i.e.,~$\bigvee_{c \in \set{1, \ldots \card{\trace}}} \varphi[i
  \mapsto c]$). This is the case because such disjunctions of formulas
  can be satisfied by $\trace$ and $k$ if and only if there exists at
  least one formula in that disjunctions of formulas that is satisfied
  by $\trace$ and $k$. \qed
\end{proof}

\begin{theorem}\label{thm:forall-elimination}
  Given a closed \langnameabr formula $\fforall i.\varphi$, a trace
  $\tau$ and a prefix length $k$,
\begin{center}
$\pref{k}{\trace} \models \fforall i.\varphi 
\mbox{ \ iff \ }
%
  \pref{k}{\trace} \models \bigwedge_{c \in \set{1, \ldots \card{\trace}}} \varphi[i \mapsto c]
$\end{center}
\end{theorem}
\begin{proof}
  The proof is quite similar to \Cref{thm:exist-elimination}, except
  that we use the conjunctions of formulas.  Basically, we have that
  $\trace$ and $k$ satisfy $\fforall i.\varphi$ (i.e.,
  $\pref{k}{\trace}~\models~\fforall i.\varphi $) iff for every
  $c \in \set{1, \ldots, \card{\trace}}$, we have that
  $\pref{k}{\trace}~\models~\psi$, where~$\psi$ is obtained
  from~$\varphi$ by substituting each variable~$i$ in~$\varphi$ with~$c$.
  In other words, $\trace$ and $k$ satisfy each formula that is
  obtained from~$\varphi$ by considering all possible substitutions of
  variable $i$ with all possible values of $c$. Hence it is the same
  as satisfying the conjunctions of those formulas
  (i.e.,~$\bigwedge_{c \in \set{1, \ldots \card{\trace}}} \varphi[i
  \mapsto c]$).  This is the case because such conjunctions of formulas
  can be satisfied by $\trace$ and $k$ if and only if each formula in
  that conjunctions of formulas is satisfied by $\trace$ and $k$. \qed
\end{proof}


To check whether a trace~$\trace$ and a prefix length~$k$
satisfy a closed \langnameabr formula~$\varphi$, i.e.,
$\pref{k}{\trace} \models \varphi$, we could perform the following
steps:
\begin{compactenum}
\item First, we eliminate all quantifiers. This can be done easily by
  applying \Cref{thm:exist-elimination,thm:forall-elimination}. As a
  result, each quantified variable will be instantiated with a
  concrete value;
%
%
%
\item Evaluate all aggregate functions as well as all event attribute
  accessor expressions based on the event attributes in $\trace$ so as
  to get the actual values of the corresponding event
  attributes. After this step, we have a formula that is constituted
  by only concrete values composed by either arithmetic operators
  (i.e.,~$+$~or~$-$), logical comparison operators
  (i.e.,~$\logeq$~or~$\neq$), or arithmetic comparison operators
  (i.e., $<$, $>$, $\leq$, $\geq$, $\logeq$ or $\neq$);
\item Last, we evaluate all arithmetic expressions as well as all
  expressions involving logical and arithmetic comparison operators.
  If the whole evaluation gives us $\true$
  (i.e.,~$\inter{\varphi}{\trace}{k}{} = \true$), then we have that
  $\pref{k}{\trace} \models \varphi$, otherwise
  $\pref{k}{\trace} \not\models \varphi$ (i.e.,~$\trace$ and $k$ do
  not satisfy $\varphi$).
\end{compactenum}


\noindent
The existence of this procedure gives us the following theorem:

\begin{theorem}
  Given a closed \langnameabr formula $\varphi$, a trace $\tau$ and a prefix
  length $k$, checking whether $\pref{k}{\trace} \models \varphi$ is
  decidable.
\end{theorem}

\noindent
This procedure has been implemented in our prototype as a part of the
mechanism for processing the specification of prediction task while
constructing the prediction model.




\subsection{Formalizing the \AnalRule}\label{sec:formalize-analytic-rule}
%
With this machinery in hand, we can formally say how to specify
condition and target expressions in \analrules, namely that condition
expressions 
are specified as closed \langnameabr formulas, while target
expressions 
are specified as either numeric expression or non-numeric expression,
except that 
target expressions are not allowed to have index variables (Thus, they
do not need variable valuation). We require 
an \analrule to be \emph{coherent}, i.e., all target expressions of an
\analrule should be either only numeric or non-numeric expressions. An
\analrule in which all of its target expressions are numeric
expressions is called \emph{numeric \analrule}, while an \analrule in
which all of its target expressions are non-numeric expressions is
called \emph{non-numeric \analrule}.

%
We can now formalize the semantics of \analrules as illustrated in
\Cref{sec:overview-lang}. Formally, given a trace $\trace$, a
considered prefix length $k$, and an \analrule $\ar$ of the form
\[
\begin{array}{r@{ \ }l}
\ar = \ruletup{
    &\cond_1 \targetarrow \target_1, ~\\
    &\cond_2 \targetarrow \target_2, ~ \\
    &\qquad \qquad\vdots ~ \\
    &\cond_n \targetarrow \target_n, ~\\
    &\otarget \ },
\end{array}
\]
%
$\ar$ maps $\trace$ and $k$ 
into a value obtained from evaluating the corresponding target
expression as follows:
\[
\ar(\pref{k}{\trace})
  = \left\{ \begin{array}{l@{ \qquad}l}
\inter{\target_1}{\trace}{k}{} &\mbox{{\normalsize if } } \pref{k}{\trace} \models \cond_1 \mbox{, }\\
\inter{\target_2}{\trace}{k}{} &\mbox{{\normalsize if } } \pref{k}{\trace} \models \cond_2 \mbox{, }\\
\ \ \ \ \ \ \ \ \ \vdots& \ \ \ \ \ \ \ \  \ \ \ \ \ \ \ \ \vdots\\
\inter{\target_n}{\trace}{k}{} &\mbox{{\normalsize if } } \pref{k}{\trace} \models \cond_n \mbox{, }\\
\inter{\otarget}{\trace}{k}{} & \mbox{{\normalsize otherwise}}
          \end{array}
\right.
\]
where 
%
%
$\inter{\target_i}{\trace}{k}{}$ is the application of our
interpretation function $\inter{\cdot}{\trace}{k}{}$ to the target
expression $\target_i$ in order to evaluate the expression and get the
value. Checking whether the given trace $\trace$ and the given prefix
length $k$ satisfy $\cond_i$, i.e.,
$\pref{k}{\trace} \models \cond_i$, can be done as explained
in~\Cref{sec:checking-foe-satisfaction}.
%
%
We also require an \analrule to be well-defined, i.e.,  
given a trace $\trace$, a prefix length $k$,
and an \analrule $\ar$, we say that \emph{$\ar$ is well-defined
  for~$\trace$~and~$k$}
if $\ar$ maps~$\trace$~and~$k$ 
into exactly one target value, i.e., for every condition expressions
$\cond_i$ and $\cond_j$ in which $\pref{k}{\trace}~\models~\cond_i$
and $\pref{k}{\trace}~\models~\cond_j$,
we have that
$\inter{\target_i}{\trace}{k}{}~=~\inter{\target_j}{\trace}{k}{}$.
This notion of well-definedness can be easily generalized to event logs as
follows: Given an event log $\eventlog$ and an \analrule $\ar$, we say
that \emph{$\ar$ is well-defined for $\eventlog$} if for every
possible
trace $\trace$ in $\eventlog$ and every possible prefix length $k$, we
have that $\ar$ is well-defined for~$\trace$~and~$k$. 
Note that such condition can be easily checked for the given event log
$\eventlog$ and an \analrule $\ar$ since the event log is finite. This
notion of well-defined is required in order to guarantee that the
given \analrule $\ar$ behaves as a function with respect to the given
event log $\eventlog$, i.e., $\ar$ maps every pair of trace $\trace$
and prefix length $k$ into a unique value.


Compared to enforcing that each condition in \analrules must not be
overlapped, our notion of well-defined gives us more flexibility in
making a specification using our language while also guaranteeing
reasonable behaviour. For instance, one can specify several
characteristics of ping-pong behaviour in a more convenient way by
specifying several \condtargetrules, i.e.,
\[
\begin{array}{c}
  \cond_1 \targetarrow \mbox{``Ping-Pong''},\\
  \cond_2 \targetarrow \mbox{``Ping-Pong''},\\ 
  \vdots\\
  \cond_m \targetarrow \mbox{``Ping-Pong''}\\ 
\end{array}
\]
(where each condition expression $\cond_i$ captures a particular
characteristic of a ping-pong behaviour), instead of using disjunctions
of these several condition expressions, i.e.,
\[
\begin{array}{c}
  \cond_1 \vee \cond_2 \vee \ldots \vee \cond_m  \targetarrow \mbox{``Ping-Pong''}
\end{array}
\]
which could end up into a very long specification of a condition
expression. 
%


\section{Building the Prediction Model }\label{sec:pred-model}

Given an \analrule $\ar$ and an event log $\eventlog$, if $\ar$ is a
numeric \analrule, we build a regression model. Otherwise, if $\ar$ is
a non-numeric \analrule, we build a classification model.


Given an \analrule $\ar$ and an event log $\eventlog$, our aim is to
create a prediction function that takes (partial) trace as the input
and predict the most probable output value for the given input.  To
this aim, 
we train a classification/regression model in which the input is the
features that are obtained from the encoding of all possible trace
prefixes in the event log $\eventlog$ (the training data).
There are several ways 
to encode (partial) traces into input features for training a machine
learning model.  For instance,~\cite{LCDDM15,SDGJM17} study various
encoding techniques such as index-based encoding, boolean encoding,
etc. In~\cite{TVLD17}, the authors use the so-called \emph{one-hot
  encoding} of event names, and also add some time-related features
(e.g., the time increase with respect to the previous event).
Some works consider the feature encodings that incorporate the
information of the last $n$-events. There are also several choices on
the information to be incorporated. One can incorporate only the name
of the events/activities, or one can also incorporate other
information (provided by the available event attributes) such as the
(human) resource who is in charged in the activity.

In general, an encoding technique can be seen as a function $\encfunc$
that takes a trace $\trace$ as the input and produces a set
$\set{x_1,\ldots, x_m}$ of features,
i.e.,~$\encfunc(\trace) = \set{x_1,\ldots, x_m}$. Furthermore, since a
trace $\trace$ might have arbitrary length (i.e., arbitrary number of
events), the encoding function must be able to transform these
arbitrary number of trace information into a fix number of
features. This can be done, for example, by considering the last
$n$-events of the given trace $\trace$ or by aggregating the
information within the trace itself. In the encoding that incorporates
the last $n$-events, if the number of the events within the trace
$\trace$ is less than $n$, then typically we can add~0 for all missing
information in order to get a fix number of features.

In our approach, users are allowed to choose the desired encoding
mechanism
by specifying a set $\encset$ of preferred encoding functions
(i.e., $\encset = \set{\encfunc_1, \ldots, \encfunc_n}$).
This allows us to do some sort of feature engineering (note that the
desired feature engineering approach, that might help increasing the
prediction performance, can also be added as one of these encoding
functions).
The set of features of a trace is then obtained by combining all
features produced by applying each of the selected encoding functions
into the corresponding trace.
In the implementation (cf.\ \Cref{sec:implementation-experiment}), we
provide some encoding functions that can be selected in order to
encode a trace. 
%

%
\begin{algorithm} \caption{- Procedure for building the prediction model} \label{algo:build-pred-model}
%
%
  \hspace*{0mm} \textbf{Input:}
  \hspace*{3.1mm}  an \analrule $\ar$, \\
  \hspace*{12.9mm}  an event log $\eventlog$, and \\
  \hspace*{12.9mm} a set
  $\encset = \set{\encfunc_1, \ldots, \encfunc_n}$ of encoding
  functions \\
  \hspace*{0mm} \textbf{Output:} \hspace*{1mm} a prediction function
  $\predfunc$ \smallskip
\begin{algorithmic}[1]

\ForEach {trace $\trace \in \eventlog$} 
\ForEach {$k$ where $1 < k< \card{\trace}$} 
\State $\prefix{k}{\trace}_{\text{encoded}}$ = $\encfunc_1(\prefix{k}{\trace}) \cup \ldots \cup \encfunc_n(\prefix{k}{\trace})$
\State targetValue = $\ar(\prefix{k}{\trace})$
\State add a new training instance for $\predfunc$, where
\Statex \hspace*{8mm} $\predfunc(\prefix{k}{\trace}_{\text{encoded}})$ = targetValue
\EndFor
\EndFor
\State Train the prediction function $\predfunc$ (either classification or regression model)
\end{algorithmic}
\end{algorithm}
%

\Cref{algo:build-pred-model} illustrates our procedure for building
the prediction model based on the given inputs, namely:
\begin{inparaenum}[\itshape (i)]
\item an \analrule $\ar$, 
\item an event log $\eventlog$, and 
\item a set $\encset = \set{\encfunc_1, \ldots, \encfunc_n}$ of encoding functions.
\end{inparaenum}
The algorithm works as follows: for each $k$-length trace
prefix $\prefix{k}{\trace}$ of each trace $\trace$ in the event log
$\eventlog$ (where $1 < k  < \card{\trace}$
), 
we do the following:
\begin{inparaenum}
\item[In line~3,] we apply each encoding function $\encfunc_i \in \encset$
  into $\prefix{k}{\trace}$,
  and combine all 
  obtained features. 
  This step gives us the encoded trace prefix.
\item[In line~4,] we compute the expected prediction result (target
  value) by applying the analytical rule $\ar$ to
  $\prefix{k}{\trace}$.
\item[In line~5,] we add a new training instance by specifying that
  the prediction function $\predfunc$ maps the encoded trace prefix
  $\prefix{k}{\trace}_{\text{encoded}}$ into the target value computed
  in the previous step.
\item[Finally,] we train the prediction function $\predfunc$ and get
  the desired prediction function.
\end{inparaenum}

Observe that the procedure above is independent with respect to the
classification/regression model and trace encoding technique that are
used. One can plug in different machine learning
classification/regression model as well as use different trace
encoding technique in order to get the desired quality of prediction.

\section{Showcases and Multi-Perspective Prediction
  Service}\label{sec:showcase}

An \analrule $\ar$ specifies a particular prediction task of
interest. To specify several desired prediction tasks, we only have to
specify several \analrules, i.e., $\ar_1, \ar_2, \ldots, \ar_n$. Given a set
$\arset = \set{\ar_1, \ar_2, \ldots, \ar_n}$ of \analrules, our approach
allows us to construct a prediction model for each \analrule
$\ar_i \in \arset$. 
By having all of the constructed prediction models where each of them
focuses on a particular prediction objective, we can obtain a
\emph{multi-perspective prediction analysis service}.
%

In~\Cref{sec:spec-language}, we have seen some examples of prediction
task specification for predicting the ping-pong behaviour and the
remaining processing time. In this section, we present numerous other
showcases of prediction task specification using our language.



\subsection{Predicting Unexpected Behaviour/Situation}

\ifShowOutlineShowcase
\todo[inline]{(\cm) 1. Predicting ping-pong behaviour\\
  (\bqm) 2. predicting whether the process involve many support
  groups or not if it involves more than 3 groups, then it is
  inefficient (kind of
  performance)\\
  (\bqm) 3. add other pingpong behaviour characterization?
}
\fi

\noindent
We can specify the task for predicting unexpected behaviour by first
expressing the characteristics of the unexpected behaviour.

\smallskip
\noindent
\textbf{Ping-pong Behaviour. \xspace} The condition expression
$\condPingPonga$ (in \Cref{sec:overview-lang}) expresses a possible
characteristic of ping-pong behaviour. Another possible
characterization of ping-pong behaviour is shown below:
%
\[
\begin{array}{r@{\ }l}
  \condPingPongb =   \fexists i . &(i > \curr 
                                     ~\fand~ i+2 \leq \last~\fand \\
                                   &\eventquery{i}{org:resource} \noteq
                                     \eventquery{i+1}{org:resource}
                                     \fand\\
                                   &\eventquery{i}{org:resource} \logeq
                                     \eventquery{i+2}{org:resource}\fand\\
                                   &\eventquery{i}{org:group} \logeq
                                     \eventquery{i+1}{org:group} 
                                     \fand\\
                                   & \eventquery{i}{org:group} \logeq
                                     \eventquery{i+2}{org:group}
                                     )
\end{array}
\]
In other word, $\condPingPongb$ characterizes the condition where
``\textit{an officer transfers a task into another officer of the same
  group, and then the task is transfered back to the original
  officer}''.
In the event log, this situation is captured by the changes of the
org:resource value in the next event, but then it changes back into
the original value in the next two events, while the values of
org:group remain the same.

We can then create an \analrule to specify the task for predicting
ping-pong behaviour as follows:
\[
\begin{array}{r@{ \ }l@{ \ }l@{ \ }l}
  \arPingPongShowcase = \langle &\condPingPonga &\targetarrow &\mbox{``Ping-Pong''}, \\
                    &\condPingPongb &\targetarrow &\mbox{``Ping-Pong''}, \\
                    &&&\mbox{``Not Ping-Pong''}\rangle,
\end{array}
\]
where $\condPingPonga$ is the same as specified in
\Cref{sec:overview-lang}.
During the construction of the prediction model, in the training
phase, $\arPingPongShowcase$ maps each trace prefix
$\prefix{k}{\trace}$ that satisfies either $\condPingPonga$ or
$\condPingPongb$ into the target value ``Ping-Pong'', and those
prefixes that neither satisfy $\condPingPonga$ nor $\condPingPongb$
into ``Not Ping-Pong''.
After training the model based on this rule, we get a classifier that
is trained for distinguishing between (partial) traces that most
likely and unlikely lead to ping-pong behaviour.
This example also exhibits the ability of our language to specify a
behaviour that has multiple characteristics.

\smallskip
\noindent
\textbf{Abnormal Activity Duration.\xspace}
The following expression specifies the \emph{existence of abnormal
  waiting duration} by stating that \emph{there exists a waiting
  activity in which the duration is more than 2 hours (7.200.000
  milliseconds)}:
\[
\begin{array}{l@{ }l}
  \condAbnormalWaitDur = \fexists i . &(i < \last) \fand \\
                                      &\eventquery{i}{concept:name}
                                        \logeq \mbox{``Waiting"} \fand \\
                                      &(\eventquery{i + 1}{time:timestamp}~-~\\
                                      &\qquad \quad \eventquery{i}{time:timestamp}) >
                                        7.200.000
\end{array}
\]
As before, we can then specify an \analrule for predicting whether a
(partial) trace is likely to have an abnormal waiting duration or not
as follows:
\[
\arAbWaitDur = \ruletup{\condAbnormalWaitDur \targetarrow
    \mbox{``Abnormal"},~\mbox{``Normal"}}.
\]
Applying the approach for constructing the prediction model in
\Cref{sec:pred-model}, we obtain a classifier that is trained to
predict whether a (partial) trace is most likely or unlikely to have
an abnormal waiting duration.


\subsection{Predicting SLA/Business Constraints Compliance} 
  
\ifShowOutlineShowcase
\todo[inline]{(\cm) 1. Similar to unexpected behaviour - activity duration is
    more than a certain threshold\\
    \quad$\bullet$ suppose that the SLA says that each activity\\
    \qquad must be done within 3 hours.  }

  \todo[inline]{(\cm) 2. Predicting the existence of abnormal activity duration\\
    \quad $\bullet$ Predicting the existence an activity in which the\\
    \qquad  running time is greater than a certain threshold\\
    \qquad   (e.g.,greater than 4 hours)\\

    \quad $\bullet$ Predicting the existence of abnormal waiting time \\
    \qquad  (using aggregation condition to filter the ‘waiting \\
    \qquad activity')\\
        
    \quad $\bullet$ Predicting the existence of abnormal time\\
    \qquad   difference between two activites (the time \\
    \qquad difference between the time when the job is \\
    \qquad submitted and being processed is more than \\
    \qquad 4 hours) }

  \todo[inline]{(\cm) 3. Predicting whether a certain activity will
    be followed by another activity (whether it complies with this constraint).\\

    \quad $\bullet$ ApplicationSubmission must be eventually \\
    \qquad followed by ApplicationValidation 

    \quad $\bullet$ CreatedOrder must be eventually \\
    \qquad followed by DeliveredOrder 
}

  \todo[inline]{(\cm) 4. Segregation of duties\\
  }

  \todo[inline]{(\bqm) 5. handover more than a certain number violate
    the SLA, suppose that there is an SLA that a process must be done
    with at most certain number of handover (\textbf{\emph{moved to performance}})\\
  }

  \todo[inline]{(\bqm) 6. Max Activity Duration Less than something,
    considered violation, suppose that we require that each activity
    must be done before x time (Note: \textit{\textbf{it is similar to
        number 1. Just different way to express
        it}}).\\
  }

  \todo[inline]{(\bqm) 7. Add  ApplicationSubmission must be eventually \\
    \qquad followed by ApplicationValidation?  (Note:
    \textit{\textbf{it is similar to number
        3. Just different constraint}}).\\
  }
\fi

  \noindent
  Using FOE, we can easily specify numerous expressive SLA conditions
  as well as business constraints. Furthermore, using the approach
  presented in \Cref{sec:pred-model}, we can create the corresponding
  prediction model, which predicts the compliance 
  of the corresponding SLA/business constraints.

\smallskip
\noindent
\textbf{Time-related SLA. \xspace}
Let $\condEveryOrderWillBeDelivered$ be the \langnameabr formula in
\Cref{ex:foe-expressive-example}. Roughly speaking,
$\condEveryOrderWillBeDelivered$ expresses an SLA stating that
\emph{each order that is created will be eventually delivered within 3
  hours}. We can then specify an \analrule for predicting the
compliance of this SLA as follows:
\[
\arOrderWillBeDeliveredShowcase = \ruletup{\condEveryOrderWillBeDelivered \targetarrow
    \mbox{``Comply"},~\mbox{``Not Comply"}}.
\]
Using $\arOrderWillBeDeliveredShowcase$, our procedure for
constructing the prediction model in \Cref{sec:pred-model} generates a
classifier that is trained to predict whether a (partial) trace is
likely or unlikely to comply with the given SLA.

\smallskip
\noindent
\textbf{Separation of Duties (SoD).\xspace} We could also specify a
constraint concerning \emph{Separation of Duties (SoD)}. For instance,
we require that the person who assembles the product
is different from the person who checks the product
(i.e., quality assurance). This can be expressed as follows:
 \[
\begin{array}{l}
  \condSoD = \fforall i. \fforall j. ((i < \last)
  \fand (j < \last) \fand \\
  \hspace*{20mm}\eventquery{i}{concept:name}
  \logeq \mbox{``assembling"} \fand\\
  \hspace*{20mm}\eventquery{j}{concept:name}
  \logeq \mbox{``checking"}) \fimpl\\
  \hspace*{25mm}(\eventquery{i}{org:resource}~\noteq~\eventquery{j}{org:resource}).
\end{array}
\]
Intuitively, $\condSoD$ states that for every two activities, if they
are assembling and checking activities, then the resources who are in
charge of those activities must be different.  Similar to previous
examples, we can specify an \analrule for predicting the compliance of
this constraint as follows:
\[
\arSoD = \ruletup{\condSoD \targetarrow
    \mbox{``Comply"},~\mbox{``Not Comply"}}.
\]
Applying our procedure for building the prediction model, we obtain a
classifier that is trained to predict whether or not a trace is likely
to
fulfil this constraint.

\smallskip
\noindent
\textbf{Constraint on Activity Duration.\xspace} 
Another example would be a constraint on the activity duration, e.g.,
a requirement which states that \emph{each activity must be finished
  within 2 hours}. This can be expressed as follows:
\[
\begin{array}{l}
  \condActDurLessThanAThreshold = \fforall i. (i < \last) \fimpl\\
  \hspace*{25mm}(\eventquery{i+1}{time:timestamp}~-\\
  \hspace*{33mm}\eventquery{i}{time:timestamp}) 
  < 7.200.000.
\end{array}
\]
$\condActDurLessThanAThreshold$ basically says that \emph{the time
  difference between two activities is always less than 2 hours
  (7.200.000 milliseconds)}.  
An analytic rule to predict the compliance of this SLA can be
specified as follows:
%
\begin{center}
$\arActDurSLA = \ruletup{\condActDurLessThanAThreshold \targetarrow
  \mbox{``Comply"},~\mbox{``Not Comply"}}$. 
\end{center}
Notice that we can express the same specification in a different way,
for instance
\[
\arActDurSLAb = \ruletup{\condActDurLessThanAThresholdb \targetarrow
  \mbox{``Not Comply"},~\mbox{``Comply"}}
\]
where 
\[
  \begin{array}{l}
    \condActDurLessThanAThresholdb = 
    \fexists i. (i < \last) \fand\\
    \hspace*{25mm} (\eventquery{i+1}{time:timestamp}~- \\
    \hspace*{33mm} \eventquery{i}{time:timestamp}) > 7.200.000.
\end{array}
\]
Essentially, $\condActDurLessThanAThresholdb$ expresses a
specification on the \emph{existence of abnormal activity
  duration}. 
It states that
%
%
\emph{there exists an activity in which the time difference between
  that activity and the next activity is greater than 7.200.000
  milliseconds (2 hours)}.
%
%
Using either $\arActDurSLA$ or $\arActDurSLAb$, our procedure for
building the prediction model (cf.~\Cref{algo:build-pred-model}) gives
us a classifier that is trained to distinguish between the partial
traces that most likely will and will not satisfy this activity
duration constraint.

We could even specify a more fine-grained constraint by focusing into
a particular activity. For instance, the following expression
specifies that \emph{each validation activity must be done within 2
  hours (7.200.000 milliseconds)}:
\[
\begin{array}{l}
  \condConstValDur = \fforall i. ((i < \last) \fand \\
  \hspace*{20mm}\eventquery{i}{concept:name}
  \logeq \mbox{``Validation"}) \fimpl\\
  \hspace*{25mm}(\eventquery{i+1}{time:timestamp}~-\\
  \hspace*{33mm}\eventquery{i}{time:timestamp}) 
  < 7.200.000.
\end{array}
\]
$\condConstValDur$ basically says that \emph{for each validation
  activity, the time difference between that activity and its next
  activity is always less than 2 hours (7.200.000 milliseconds)}.
Similar to the previous examples, it is easy to see that we could
specify an \analrule for predicting the compliance of this SLA and
create a prediction model that is trained to predict whether a
(partial) trace is likely or unlikely fulfilling this SLA.



\subsection{Predicting Time Related Information}

\ifShowOutlineShowcase
\todo[inline]{(\cm) 1. Predicting the Remaining Processing Time\\[1ex]
  (\bqm) 2. Predicting the Total Duration of a Certain Activity\\
  \quad $\bullet$ Predicting the Total Duration\\
  \qquad of the Remaining Validation Activity\\
  \quad $\bullet$ Predicting total duration of a certain activity\\
  \qquad within the  whole process (waiting time)\\
  \quad $\bullet$ Predicting the Total Duration of the \\
  \qquad Remaining Waiting  Activity\\[1ex]
} \todo[inline]{(\cm) 3. Predicting total REMAINING duration of a certain
  activity (e.g., the duration of the remaining waiting activities)\\[1ex]
  (\cm) 4. Predicting Over-time Fault (Overhead of Running Time)\\
  \quad $\bullet$ the overhead can also be considered as a \\
  \qquad way to quantifying the risk \\
  \qquad (the quantification of process risk)\\[1ex]
}
\todo[inline]{(\cm) 5. Predicting average activity duration\\[1ex]
  (\bqm) 6. Predicting maximal activity duration\\[1ex]
  (\bqm) 7. predicting average running time of a certain activity,
  and
  total time of a certain activity\\[1ex]
}
\todo[inline]{(\cm) 8. Predicting the timestamp of the next event\\[1ex]
  (\cm) 9. Predicting delay\\
  \quad $\bullet$ if the total duration is above the expected\\
  \qquad duration (suppose that the expected duration \\
  \qquad is given in the beginning)\\[1ex]
}
\fi

\noindent
In \Cref{sec:overview-lang}, we have seen how we can specify the task
for predicting the \emph{remaining processing time} (by specifying a
target expression that computes the time difference between the
timestamp of the last and the current events). In the following, we
provide another examples on predicting time related information.


\smallskip
\noindent
\textbf{Predicting Delay.\xspace} Delay can be defined as a condition
when the actual processing time is longer than the expected processing
time. Suppose we have the information about the expected processing
time, e.g., provided by an attribute ``$\text{expectedDuration}$'' of
the first event, we can specify an \analrule for predicting the
occurrence of delay as follows:
\begin{center}
$\arDelay = \ruletup{\condDelay \targetarrow \mbox{``Delay"},~\mbox{``Normal"}}$. 
\end{center}
where $\condDelay$ is specified as follows:
\[
\begin{array}{l}
(\eventquery{\last}{time:timestamp}~- \\
\qquad \qquad\eventquery{1}{time:timestamp}) > \eventquery{1}{expectedDuration}.
\end{array}
\]
$\condDelay$ states that the difference between the last event
timestamp and the first event timestamp (i.e., the processing time) is
greater than the expected duration (provided by the value of the event
attribute ``expectedDuration''). While training the classification
model, $\arDelay$ maps each 
trace prefix $\prefix{k}{\trace}$ into either ``Delay'' or ``Normal''
depending on whether the processing time of the whole trace $\trace$
is greater than the expected processing time or not. 
%

\smallskip
\noindent
\textbf{Predicting the Overhead of Running Time.\xspace} The
\emph{overhead of running time} is the amount of time that exceeds the
expected running time. If the actual running time does not go beyond
the expected running time, then the overhead is 0. Suppose that the
expected running time is 3 hours (10.800.000 milliseconds), the task
for predicting the overhead of running time can then be specified as
follows:
\[
  \arOvertime = \ruletup{ \curr < \last \targetarrow
    \maxtwo{\rep{Overhead}}{0}, ~0}.
\]
where \ $\rep{Overhead} = \rep{TotalRunTime} - 10.800.000$, and
\[
\begin{array}{l}
\rep{TotalRunTime} = \\
\quad\qquad\eventquery{\last}{time:timestamp} -  \eventquery{1}{time:timestamp}.
\end{array}
\]
In this case, $\arOvertime$ computes the difference between the actual
total running time and the expected total running time. Moreover, it
outputs 0 if the actual total running time is less than the expected
total running time, since it takes the maximum value between the
computed time difference and 0. Applying our procedure for creating
the prediction model, we obtain a regression model that predicts the
overhead of running time.

\smallskip
\noindent
\textbf{Predicting the Remaining Duration of a Certain
  \mbox{Event}. \xspace} Let the duration of an event be the
time difference between the timestamp of that event and its
succeeding event. The task for predicting the \emph{total duration
  of all remaining ``waiting'' events} can be specified as
follows:
\[
  \arRemWaitDur = \ruletup{ \curr < \last \targetarrow \rep{RemWaitingDur}, ~0}.
\]
where $\rep{RemWaitingDur}$ is defined as the \emph{sum of the duration of
all remaining waiting events}, formally as follows:
\[
\sumc{ \eventquery{x+1}{time:timestamp} -
  \eventquery{x}{time:timestamp}}{\curr}{\last}{\eventquery{x}{concept:name} \logeq \text{"waiting"}}
\]
As before, based on this rule, we can create a regression model that
predicts the total duration of all remaining waiting events.

\smallskip
\noindent
\textbf{Predicting Average Activity Duration.\xspace} 
We can specify the way to compute the \emph{average of activity
  duration} as follows:
\[
\begin{array}{l}
\rep{AvgActDur} = \\
\qquad\uavgc{ \eventquery{x+1}{time:timestamp} - \eventquery{x}{time:timestamp} }{1}{\last}
\end{array}
\]
where the activity duration is defined as the time difference between
the timestamp of that activity and its next activity. We can then
specify an \analrule that expresses the task for predicting the
average activity duration as follows:
\[
  \arAvgActDur = \ruletup{ \curr < \last \targetarrow \rep{AvgActDur}, ~0}.
\]
Similar to previous examples, applying our procedure for creating the
prediction model, we get a regression model that computes the
approximation of the average activity duration of a process. 



\subsection{Predicting Workload-related Information} 

\ifShowOutlineShowcase
\todo[inline]{(\cm) 1. Predicting the number of all remaining activities\\[1ex]
  (\bqm) 2. Predicting the TOTAL number of a \\
  \qquad CERTAIN ACTIVITY within a  process \\
  \qquad (e.g., validation)\\[1ex]
  (\cm) 3. Predicting the REMAINING number of \\
  \qquad a CERTAIN ACTIVITY (e.g., validation)\\[1ex]
}
\todo[inline]{(\cm)  4. Predicting whether a process is complex \\
  \qquad (based on the number of a certain \\
  \qquad activity - if there are more than 3 validation \\
  \qquad activities then it is a complex process)\\[1ex]
} \fi

\noindent
Knowing the information about the amount of work to be done (i.e.,
workload) would be beneficial. 
%
Predicting the activity frequency is one of the ways to get 
an overview of workload. The following task specifies how to predict
\emph{the number of the remaining activities} that are
necessary 
to be performed:
\[
\begin{array}{l}
  \arRemAct = \langle \curr < \last \targetarrow\\
  \qquad \qquad \qquad \counta{\true}{\curr}{\last}, ~0 \rangle
\end{array}
\]
In this case, $\arRemAct$ counts the number of remaining
activities. We could also provide a more fine-grained specification by
focusing on a certain activity. For instance, in the following we
specify the task for predicting \emph{the number of the remaining
  validation activities} that need to be done:
\[
\arRemValidationAct = \ruletup{ \curr < \last \targetarrow \rep{NumOfRemValidation}, ~0}.
\]
 where $\rep{NumOfRemValidation}$ is specified as follows:
\[
\countb{\eventquery{x}{concept:name} \logeq
  \text{"validation"}}{\curr}{\last}
\]
$\rep{NumOfRemValidation}$
%
%
counts the occurrence of validation activities
between the current event and the last event (the occurence of
validation activity is reflected by the fact that the value of the
attribute $\mbox{concept:name}$ is equal to $\text{"validation"}$).
Applying our procedure for creating the prediction model over
$\arRemAct$ and $\arRemValidationAct$,
consecutively we get regression models that predict the number of
remaining activities as well as the number of the remaining validation
activities.

We could also classify a process into complex or normal based on the
frequency of a certain activity. For instance, we could consider a
process that requires more than 25 validation activities as complex
(otherwise it is normal). The following \analrule specifies this task:
\[
  \arComplexProcess = \ruletup{ \condNumValidation > 25 \targetarrow
    \mbox{``complex"},~ \mbox{``normal"}}
\]
where $\condNumValidation$ is specified as follows:
\[
\countb{\eventquery{x}{concept:name} \logeq
  \text{"validation"}}{1}{\last}
\]
Based on $\arComplexProcess$, 
we could train a model to classify
whether a (partial) trace is 
likely to be a complex or a normal process.





\subsection{Predicting Resource-related Information}

\ifShowOutlineShowcase
\todo[inline]{(\cm) 1. Predicting the number of different \\
  \qquad resources that are needed to be involved.\\[1ex]

  (\cm) 2. Predicting the number of handovers \\
  \qquad among resources\\[1ex]

  (\cm) 3. Predicting a process is labor intensive \\
  \qquad or normal\\
  \qquad $\bullet$ based on whether the number of different \\
  \qquad resources that are involved in the process is \\
  \qquad below or above the threshold (whether it \\
  \qquad requires multiple different resources \\
  \qquad involvement (e.g., at least there are \\
  \qquad there different resources involved)) \\
  \qquad * Note: there are two ways to specify this \\[1ex]

} \fi


\noindent
Human resources could be a crucial factor in the process
execution. Knowing the number of different resources that are needed
for handling a process could be beneficial. The following \analrule
specifies the task for predicting the number of different resources
that are required: 
\[
\begin{array}{l}
\arNumDifResources = \ruletup{ \curr < \last \targetarrow \\
\qquad \qquad \qquad\countvala{org:resource}{1}{\last}, ~0}.
\end{array}
\]
During the training phase, since 
\begin{center}
$\countvala{org:resource}{1}{\last}$
\end{center}
is evaluated to the number of different values of the attribute
$\text{org:resource}$ within the corresponding trace,
$\arNumDifResources$ maps each trace prefix
$\prefix{k}{\trace}$ 
into the number of different resources.

To predict the number of \emph{task handovers} among
resources, we can specify the following prediction task:
\[
\begin{array}{l}
\arNumHandovers = \ruletup{ \curr < \last \targetarrow \rep{NumHandovers}, ~0}.
\end{array}
\]
where \rep{NumHandovers} is defined as follows:
\[
\begin{array}{l}
  \countb{\eventquery{x}{org:resource} \noteq
  \eventquery{x+1}{org:resource}}{1}{\last}
\end{array}
\]
i.e., \rep{NumHandovers} counts the number of changes on the value of
the attribute $\text{org:resource}$ and the changes of resources
reflect the task handovers among resources. Thus, in this case,
$\arNumHandovers$ maps each trace prefix $\prefix{k}{\trace}$ 
into the number of task handovers.

A process can be considered as \emph{labor intensive} if it involves
at least a certain number of different resources, e.g., three
different number of resources. This kind of task can be specified as
follows:
\[
  \arLaborIntensive = \ruletup{\condInvolveThreeResources \targetarrow
    \mbox{``LaborIntensive"},~ \mbox{``normal"}}
\]
where $\condInvolveThreeResources$ is as follows:
\[
\begin{array}{r@{\ }l}
\fexists i . \fexists j . \fexists k . &(i \noteq j \fand i \noteq k
                                         \fand j \noteq k \fand  \\
                                   &\eventquery{i}{org:resource} \noteq
                                     \eventquery{j}{org:resource}
                                     \fand\\
                                   &\eventquery{i}{org:resource} \noteq
                                     \eventquery{k}{org:resource}
                                     \fand\\
                                   &\eventquery{j}{org:resource} \noteq
                                     \eventquery{k}{org:resource}
                                     )
\end{array}
\]
Essentially, $\condInvolveThreeResources$ states that there are at
least three different events in which the values of the attribute
$\text{org:resource}$ in those events are different.
%


\subsection{Cost-related prediction}

\ifShowOutlineShowcase
\todo[inline]{(\cm) 1. Predicting total cost of the whole process\\[1ex]
  (\cm) 2. Predicting total cost of a certain activity\\[1ex]
}
\todo[inline]{(\cm) 3. Predicting whether the process is normal or expensive (based on
  whether the total cost is below the threshold or above the threshold)\\[1ex]
  (\cm) 4. Predicting average/total activity cost (in the case where the
  cost is the sum of resource and material costs)\\[1ex]
  (\cm) 5. Predicting maximal cost 
}
\todo[inline]{(\bqm) 6. Predicting the total cost of the remaining
  activities/events\\[1ex]
  (\bqm) 7. Predicting the total cost of the next three events\\[1ex]
}
\fi

Suppose that each activity within a process has its own cost and this
information is stored in the attribute named \emph{cost}. The task for
predicting the total cost of a process can be specified as follows:
\[
\begin{array}{l}
  \arTotalCost = \ruletup{ \curr <\last \targetarrow \\ 
 \quad \qquad \qquad \qquad \usuma{\eventquery{x}{cost}}{1}{\last}, ~0},
\end{array}
\]
where $\arTotalCost$ maps each trace prefix
$\prefix{k}{\trace}$ 
into the corresponding total cost that is computed by summing up the
cost of all activities. We can also specify the task for predicting
the maximal cost within a process as follows:
\[
\begin{array}{l}
  \arMaxCost = \ruletup{ \curr < \last \targetarrow \\
\qquad \qquad \qquad \quad \umaxa{\eventquery{x}{cost}}{1}{\last}, ~0}.
\end{array}
\]
In this case, $\arMaxCost$ computes the maximal cost among the cost of
all activities within the corresponding process. Similarly,
we can specify the task for predicting the average activity cost as
follows:
\[
\begin{array}{l}
  \arAvgCost = \ruletup{ \curr < \last \targetarrow\\
\qquad \qquad \qquad \quad\uavga{\eventquery{x}{cost}}{1}{\last}, ~0}.
\end{array}
\]

We could also create a more detailed specification. For instance, we
want to predict the \emph{total cost of all validation
  activities}. This task can be specified as follows:
\[
\begin{array}{l}
  \arValidationCost = \ruletup{ \curr < \last \targetarrow \rep{TotalValidationCost}, ~0}.
\end{array}
\]
where \rep{TotalValidationCost} is as follows:
\[
\sumb{\eventquery{x}{cost}}{1}{\last}{\eventquery{x}{concept:name}
  \logeq \text{"Validation"}}
\]

In a certain situation, the cost of an activity can be broken down
into several components such as human cost and material cost. Thus,
the total cost of each activity is actually the sum of the human and
material costs. To take these components into account, the prediction
task can be specified as follows:
\[
\begin{array}{l}
  \arTotalCostb = \ruletup{ \curr < \last \targetarrow \rep{TotalCost}, ~0}.
\end{array}
\]
where $\rep{TotalCost}$ is as follows:
\[
\usumc{\eventquery{x}{humanCost} + \eventquery{x}{materialCost} }{1}{\last}
\]

One might consider a process as expensive if its total cost is greater
than a certain amount (e.g., 550 Eur), otherwise it is normal. Based
on this characteristic, we could specify a task for predicting whether
a process would be expensive or not as follows:
\[
  \arExpensive = \ruletup{ \rep{TotalCost} > 550 \targetarrow
    \mbox{``expensive"},~ \mbox{``normal"}}
\]
where $\rep{TotalCost} = \usuma{\eventquery{x}{cost}}{1}{\last}$.







\subsection{Predicting Process Performance}

\ifShowOutlineShowcase
\todo[inline]{(\cm) 1. Classify process into slow, normal, or fast based on
  whether
  the total running time is below or above a certain threshold\\[1ex]
}
\todo[inline]{(\cm) 2. Predicting a process is efficient or not\\[1ex]
  \quad $\bullet$ based on the number of handovers between \\
  \qquad resources 
}
\todo[inline]{(\cm) 3. Predicting a process is efficient or not\\[1ex]
  \quad $\bullet$ based on whether the number of a certain \\
  \qquad activity is less than or greater than a certain \\
  \qquad number (e.g., the number of \\
  \qquad “modify
  delivery appointment” is less than 5)\\
}
\todo[inline]{(\bqm) 4. Predicting the number of handovers \\
  \qquad between  resources 
}
\fi

\noindent
One could consider the process that runs longer than a certain amount
of time as \emph{slow}, otherwise it is \emph{normal}. Given a
(partial) process execution information, we might be interested to
predict whether it will end up as a slow 
or a normal process. This prediction task can be specified as follows:
\begin{center}
  $\arProcPerformance = \ruletup{\condProcPerformance \targetarrow
    \mbox{``Slow"},~ \mbox{``normal"}}$.
\end{center}
 where 
\[
\begin{array}{l}
\condProcPerformance = (\eventquery{\last}{time:timestamp}~-\\
\hspace*{30mm} \eventquery{1}{time:timestamp}) > 18.000.000.
\end{array}
\]
$\arProcPerformance$ states that if the total running time of a
process is greater than 18.000.000 milliseconds (5 hours), then it is
categorized as slow, otherwise it is normal. During the
training, 
$\arProcPerformance$ maps each 
trace prefix $\prefix{k}{\trace}$ into the corresponding performance
category (i.e., slow or
normal). 
In this manner, we get a prediction model that is trained to predict
whether a certain (partial) trace will most likely be
slow 
or normal.

Notice that we can specify a more fine-grained characteristic of
process performance. For instance, we can add one more characteristic
into $\arProcPerformance$ by saying that the processes that spend less
than 3 hours (10.800.000 milliseconds) are considered as
\emph{fast}. This is specified by $\arProcPerformanceb$ as follows:
\[
\arProcPerformanceb = \ruletup{
    \condProcPerformance \targetarrow \mbox{``Slow"},~
    \condProcPerformanceb \targetarrow \mbox{``Fast"},~
    \mbox{``normal"}}
\]
 where 
\[
\begin{array}{l}
\condProcPerformanceb = (\eventquery{\last}{time:timestamp}~-\\
\hspace*{30mm}\eventquery{1}{time:timestamp}) < 10.800.000
\end{array}
\]

One might consider that a process is performed efficiently if there
are only small amount of \emph{task handovers} between resources. On
the other hand, one might consider a process is efficient if it
involves only a certain number of different resources.  Suppose that
the processes that have more than 7 times of task handovers
among the (human) resources are considered to be inefficient. We can
then specify a task to predict whether a (partial) trace is most
likely to be inefficient or not as follows:
\[
  \arEfficiency = \ruletup{ \condEfficiency > 7 \targetarrow
    \mbox{``inefficient"},~ \mbox{``normal"}}
\]
where $\condEfficiency$ is specified as follows:
\[
\begin{array}{l}
  \countb{\eventquery{x}{org:resource} \noteq
  \eventquery{x+1}{org:resource}}{1}{\last}
\end{array}
\]
i.e., $\condEfficiency$ counts how many times the value of the
attribute org:resource is changing from a one time point to another
time point by checking whether the value of the attribute org:resource
at a particular time point is different from the value of the
attribute org:resource at the next time point.
Now, suppose that the processes that involve more than 5 resources are
considered to be inefficient. We can then specify a task to predict
whether a (partial) trace is most likely to be inefficient or not as
follows:
\[
  \arEfficiencyb = \ruletup{ \condEfficiencyb > 5 \targetarrow
    \mbox{``inefficient"},~ \mbox{``normal"}}
\]
where $\condEfficiencyb = \countvala{org:resource}{1}{\last}$, i.e., 
it counts the number of different values of the attribute
org:resource.
As before, using $\arEfficiency$ and $\arEfficiencyb$, we could then
train a classifier to predict whether a process will most likely 
perform inefficiently or normal.





\subsection{Predicting Future Activities/Events}

\ifShowOutlineShowcase
\todo[inline]{(\cm) 1. Predicting the next activity\\
(\cm) 2. Predicting the next lifecycle\\
(\cm) 3. Predicting the next 3 activities\\
(\bqm) 4. Predicting the suffix\\
}  
\fi

\noindent
The task for predicting the next activity/event can be specified as follows:
\[
  \arNextEvent = \ruletup{ \curr < \last \targetarrow
    \eventquery{\curr + 1}{concept:name}, ~\text{""}}.
\]
During the construction of the prediction model, $\arNextEvent$ maps
each 
trace prefix $\prefix{k}{\trace}$
into its next activity name, because
$\eventquery{\curr+1}{concept:name}$ is evaluated to the name of the
next activity.
%
%

Similarly, we can specify the task for predicting the next lifecycle
as follows:
\[
\begin{array}{l}
  \arLifeCycle = \ruletup{ \curr < \last \targetarrow \eventquery{\curr + 1}{lifecycle:transition}, ~\text{""}}
\end{array}
\]
In this case, since $\eventquery{\curr+1}{lifecycle:transition}$ is
evaluated to the lifecycle information of the next event,
$\arLifeCycle$ maps each trace prefix
$\prefix{k}{\trace}$ 
into its next lifecycle.

Instead of just predicting the information about the next activity, we
might be interested in predicting more information such as the
information about the next three activities. This task can be
specified as follows:
\[
  \arNextThreeEvents = \ruletup{ \curr + 3 \leq \last \targetarrow
    \rep{Next3Activities}, ~\rep{RemEvents}}.
\]
where 
\[
\begin{array}[t]{lcl}
  \textup{\rep{Next3Activities}} &\textup{=}&
                                          \uconcatc{\eventquery{x}{concept:name}}{\curr + 1}{\curr+3}\\
  \textup{\rep{RemEvents}} &\textup{=}& 
                                        \uconcatc{\eventquery{x}{concept:name}}{\curr + 1}{\last}
\end{array}
\]
During the construction of the prediction model, in the training
phase, $\arNextThreeEvents$ maps each trace prefix
$\prefix{k}{\trace}$
into the information about the next three activities.






%
%

\ifShowOutlineShowcase

\bigskip \bigskip \bigskip \bigskip \bigskip \bigskip
\todo[inline]{\ \\ \ \\ \  \\ START  - SKETCH \ \\ \  \\ \ \\ }

\noindent
\textbf{Predicting quantitive information}

\todo[inline]{predicting the number of resources, a certain events, costs}

$\avga{\aggnumsrc}{\st}{\ed}{\aggcondexp}$
$\mina{\aggnumsrc}{\st}{\ed}{\aggcondexp}$
$\maxa{\aggnumsrc}{\st}{\ed}{\aggcondexp}$
$\concata{\aggnonnumsrc}{\st}{\ed}{\aggcondexp}$

$\usuma{\eventquery{x}{cost}}{1}{\last}$

$\sumb{\eventquery{x}{cost}}{1}{\last}{\eventquery{x}{concept:name}
  \logeq \text{"Validation"}}$

$\uavgc{\eventquery{x+1}{time:timestamp} -
  \eventquery{x}{time:timestamp}}{1}{\last}$

$\countb{\eventquery{x}{concept:name} \logeq \text{"validation"}}{1}{\last}$

$\countvala{org:resource}{1}{\last}$

\[
\begin{array}{r@{ \ }l@{ \ }l}
  \ar = \langle &\condPingPonga &\targetarrow \mbox{ha}, \\
                    &\multicolumn{2}{l}{\mbox{``Not Ping-Pong''}\rangle} 
\end{array}
\]
\[
  \ar_4 = \ruletup{ \curr + 1 \leq \last \targetarrow
    \mathsf{NextThree}, ~\udefined}.
\]
\todo[inline]{\ \\ \ \\ \  \\ END OF  - SKETCH \ \\ \  \\ \ \\ }

\fi

%
%

\section{Implementation and Experiment}
\label{sec:implementation-experiment}

As a proof of concept, we develop a prototype that implements
our approach.
%
%
%
%
%
This prototype includes a parser for our language and a program for
automatically processing the given prediction task specification as
well as for building the corresponding prediction model based on our
approach explained in \Cref{sec:spec-language,sec:pred-model}.
We also build a ProM\footnote{ProM is a widely used extendable
  framework for process mining (\url{http://www.promtools.org}).}
{plug-in} that wraps these functionalities.
%
Several feature encoding functions to be selected are also provided,
e.g., one hot encoding of event attributes, time since the previous
event, plain attribute values encoding, 
etc.
We can also choose the desired machine learning model to be built. 
Our implementation uses Java and Python. For the interaction between
Java and Python, we use Jep (Java Embedded Python)\footnote{Jep -
  \url{https://pypi.org/project/jep/}}. In general, we use Java for
implementing the program for processing the specification and we use
Python for dealing with the machine learning models.

Our experiments aim at demonstrating the applicability of our approach
in automatically constructing reliable prediction models based on the
given specification.
%
%
%
The experiments were conducted by applying our approach into several
case studies/problems that are based on real life event
logs. Particularly, we use the publicly available event logs that were
provided for Business Process Intelligence Challenge (BPIC)
2012, 
BPIC 2013, 
and BPIC 2015. 
For each event log, several relevant prediction tasks are
formulated 
based on the corresponding domain, 
and also by considering the available information.  For instance,
predicting the occurence of ping-pong behaviour among support groups
might be suitable for the BPIC~13 event log, but not for BPIC~12 event
log since there is no information about groups in BPIC~12 event log
(in fact, they are event logs from two different domains). For each
prediction task, we provide the corresponding formal
specification 
that can be fed into our tool in order to create the
corresponding prediction model.


For the experiment, we follow the standard \emph{holdout
  method}~\cite{HPK11}. Specifically, we partition the data into two
sets as follows: we use the first 2/3 of the log for the training data
and the last 1/3 of the log for the testing data. For each prediction
task
specification, 
we apply our approach in order to generate the corresponding
prediction model, and then we evaluate the prediction quality of the
generated prediction model by considering
each $k$-length trace prefix $\prefix{k}{\trace}$ of each trace
$\trace$ in the testing set (for $1 < k < \card{\trace}$).
In order to provide a baseline, 
we use a statistical-based prediction technique, which is often called
{Zero Rule (ZeroR)}. Specifically, for the classification task, the
prediction by $\text{ZeroR}$ is performed based on the most common
target value in the training set, while for the regression task, the
prediction is based on the mean value of the target values in the
training data.
%
%
%
%

Within these experiments, we consider several machine learning models,
namely
\begin{inparaenum}[\itshape (i)]
\item Logistic Regression, 
\item Linear Regression, 
\item Naive Bayes Classifier,
\item Decision Tree~\cite{BFSO84},
\item Random Forest~\cite{B01}, 
\item Ada Boost~\cite{FS97} with Decision Tree as the base estimator,
\item Extra Trees~\cite{GEW06},
\item Voting Classifier that is composed of Decision Tree, Random
  Forest, Ada Boost, and Extra Trees.
\end{inparaenum}
Among these, Logistic Regression, Naive Bayes, and Voting Classifier
are only used for classification tasks, and Linear Regression is only
used for regression tasks. The rest are used for both.
Notably, we also use a Deep Learning Model~\cite{GBC16}. In
particular, we use the Deep Feed-Forward Neural Network and we
consider various sizes of the network by taking into account several
different depth and width of the network (we consider different
numbers of hidden layers ranging from 2 to 6 and three variants of the
number of neurons namely 75, 100 and 150). 
%
%
%
%
 In the implementation, we use the machine learning
libraries provided by \emph{scikit-learn}~\cite{scikit-learn}. For the
implementation of neural network, we use
\emph{Keras}\footnote{\url{https://keras.io}} with
\emph{Theano}~\cite{theano} backend.



To assess the prediction quality,
we use the standard metrics for evaluating classification and
regression models that are generally used in the machine learning 
literatures. 
%
These metrics are also widely used in many works in this research
area 
(e.g.~\cite{ASS11,MRRT17,MFDG14,VDLMD15,LCDDM15,TVLD17}).
%
%
For the classification task, we use Accuracy, Area Under the ROC Curve
(AUC), Precision, Recall, and F-Measure. For the regression task, we
use Mean Absolute Error (MAE) and Root Mean Square Error (RMSE).  In
the following, we briefly explain these metrics. A more elaborate
explanation on these metrics can be found in the typical literature on
machine learning and data mining, e.g.,~\cite{MRT12,HPK11,FHT01}.

%
\emph{Accuracy} is the fraction of predictions that are correct.
%
It is computed by dividing the number of correct predictions by the
number of all predictions.
The range of accuracy value is between 0 and 1. The
value 1 indicates the best model, while 0 indicates the worst
model. 
%
%
%
An \emph{ROC (Receiver Operating Characteristic) curve} allows us to
visualize the prediction quality of a classifier. If the classifier is
good, the curve should be as closer to the top left corner as
possible. A random guessing is depicted as a straight diagonal
line. Thus, the closer the curve to the straight diagonal line, the
worse the classifier is.
The value of the \emph{area under the ROC curve (AUC)} allows us to
assess a classifier as follows: the AUC value equal to 1 shows a
perfect classifier while the AUC value equal to 0.5 shows the worst
classifier that is not better than random guessing. Thus, the closer
the value to 1, the better it is, and the closer the value to 0.5, the
worse it is.
%
%
\emph{Precision} measures the \textit{exactness} of the
prediction. 
%
%
When a classifier predicts a certain output for a certain case, the
precision value intuitively indicates how much is the chance that such
prediction is correct.
%
%
%
%
%
Specifically, among all cases that are classified into a particular
class, precision measures the fraction of those cases that are
correctly classified.
%
%
%
%
%
On the other hand, \emph{recall} measures the \textit{completeness} of
the prediction. 
%
%
%
%
Specifically, among all cases that should be classified as a
particular class, recall measures the fraction of those cases that can
be classified correctly.
%
%
%
Intuitively, 
given a particular class, the recall value indicates the ability of
the model to correctly classify all cases that should be classified
into that particular class.
%
The best precision and recall value is 1. 
%
%
\emph{F-Measure} is harmonic mean of precision and recall. It provides
a measurement that combines both precision and recall values by also
giving equal weight to them. Formally, it is computed as follows:
$\text{F-Measure} = (2 \times P \times R) / (P + R)$, where $P$ is
precision and $R$ is recall. The best F-Measure value is 1. Thus, the
closer the value to 1, the better it is.
%
%
%

\emph{MAE} computes the average of the absolute error of all
predictions over the whole testing data, where each error is computed
as the difference between the expected and the predicted
values. Formally, given $n$ testing data,
$\text{\emph{MAE}}~=~\left(\sum_{i=1}^n \card{y_i  - \hat{y_i}}\right) / n$,
where $\hat{y_i}$ (resp.\ $y_i$) is the predicted value (resp.\ the
expected/actual value) for the testing instance $i$.
\emph{RMSE} can be computed as follows:
$\text{\emph{RMSE}}~=~\sqrt{\left(\sum_{i=1}^n (y_i -
    \hat{y_i})^2\right) / n}$, where $\hat{y_i}$ (resp.\ $y_i$) is the
predicted value (resp.\ the expected/actual value) for the testing
instance $i$. Compare to MAE, RMSE is more sensitive to errors since
it gives larger penalty to larger errors by using the 'square'
operation. For both MAE and RMSE, the lower the score, the better the
model is.

In our experiments, we use the trace encoding that incorporates the
information of the last $n$-events, where $n$ is the maximal length of
the traces in the event log under consideration. Furthermore, for each
experiment we consider two types of encoding, where each of them
considers different available event attributes (One encoding
incorporates more event attributes than the others). The detail of
event attributes that are considered is explained in each experiment
below.






\subsection{Experiment on BPIC 2013 Event Log}

The event log from BPIC 2013 \footnote{More information on BPIC 2013
  can be found in
  \url{http://www.win.tue.nl/bpi/doku.php?id=2013:challenge}}~\cite{BPI-13-data} contains the data from
the Volvo IT incident management system called
VINST. 
It stores information concerning the incidents handling process. For
each incident, a solution should be found as quickly as possible so as
to bring back the service with minimum interruption to the
business. It contains 7554 traces (process instances) and 65533
events. There are also several attributes in each event containing
various information such as the problem status, the support team
(group) that is involved in handling the problem, the person who works
on the problem, etc.

In BPIC 2013, ping-pong behaviour 
is one of the interesting problems to be analyzed. Ideally, an
incident should be solved quickly without involving too many support
teams.
%
%
%
To specify the tasks for predicting whether a process would probably
exhibit a ping-pong behaviour, we first identify and express the
possible characteristics of ping-pong behaviour as follows:
%

\smallskip

\begin{flushleft}
$
\begin{array}{r@{}l}
  \condPingPongExpa =   \fexists i . (&i > \curr \fand i < \last \fand \\
  &\eventquery{i}{org:group} \noteq \eventquery{i+1}{org:group}  \fand 
  \\
                                      & \eventquery{i}{concept:name}\noteq \mbox{``Queued"} )
\end{array}
$
\end{flushleft}

\smallskip

\begin{flushleft}
$
\begin{array}{r@{}l}
\condPingPongExpb =  \fexists i . (&\eventquery{i}{org:resource} \noteq
                                 \eventquery{i+1}{org:resource}
                                     \fand\\
                                   &\eventquery{i}{org:resource} \logeq
                                 \eventquery{i+2}{org:resource})
\end{array}
$
\end{flushleft}

\smallskip

\begin{flushleft}
$
\begin{array}{r@{}l}
\condPingPongExpc =  \fexists i . (&\eventquery{i}{org:resource} \noteq
                                 \eventquery{i+1}{org:resource}
                                     \fand\\
                                   &\eventquery{i}{org:resource} \logeq
                                 \eventquery{i+3}{org:resource})
\end{array}
$
\end{flushleft}

\smallskip

\begin{flushleft}
$
\begin{array}{r@{}l}
\condPingPongExpd =  \fexists i . (&\eventquery{i}{org:group} \noteq
                                 \eventquery{i+1}{org:group}
                                     \fand\\
                                   &\eventquery{i}{org:group} \logeq
                                 \eventquery{i+2}{org:group})
\end{array}
$
\end{flushleft}

\smallskip

\begin{flushleft}
$
\begin{array}{r@{}l}
\condPingPongExpe =  \fexists i . (&\eventquery{i}{org:group} \noteq
                                 \eventquery{i+1}{org:group}
                                     \fand\\
                                   &\eventquery{i}{org:group} \logeq
                                 \eventquery{i+3}{org:group})
\end{array}
$
\end{flushleft}

\smallskip

\begin{flushleft}
$
\begin{array}{r@{}l}
\condPingPongExpf =  \fexists i . \fexists j . \fexists k . (&i < j ~\fand~ j < k ~\fand~ \\
                                   &\eventquery{i}{org:group} \noteq
                                     \eventquery{j}{org:group}
                                     \fand\\
                                   &\eventquery{i}{org:group} \noteq
                                     \eventquery{k}{org:group}
                                     \fand\\
                                   &\eventquery{j}{org:group} \noteq
                                     \eventquery{k}{org:group}
                                     )
\end{array}
$
\end{flushleft}

\smallskip

Roughly speaking, $\condPingPongExpa$ says that \emph{there is a
  change in the support team while the problem is not being
  ``Queued''}. $\condPingPongExpb$ and $\condPingPongExpc$ state that
\emph{there is a change in the person who handles the problem, but
  then at some point it changes back into the original
  person}. $\condPingPongExpd$ and $\condPingPongExpe$ say that
\emph{there is a change in the support team (group) who handles the
  problem, but then at some point it changes back into the original
  support team}. $\condPingPongExpf$ states that \emph{the process of
  handling the incident involves at least three different groups}.

We then specify three different \analrules below in order to specify
three different tasks for predicting ping-pong behaviour based on
various characteristics of this unexpected behaviour.

\begin{flushleft}
$\quad
  \arPingPongExpa = \ruletup{ \condPingPongExpa \targetarrow \mbox{``Ping-Pong''}, ~ \mbox{``Not Ping-Pong''} }
$
\end{flushleft}

\begin{flushleft}
$\quad
\begin{array}{@{}r@{ \ }l@{ \ }l}
\arPingPongExpb = \ruletup{
    &\condPingPongExpb \targetarrow &\mbox{``Ping-Pong''}, ~\\
    &\condPingPongExpc \targetarrow &\mbox{``Ping-Pong''}, ~ \\
    &\condPingPongExpd \targetarrow &\mbox{``Ping-Pong''}, ~\\
    &\condPingPongExpe \targetarrow &\mbox{``Ping-Pong''}, ~\\
    &&\mbox{``Not Ping-Pong''} \ },
\end{array}
$
\end{flushleft}

\begin{flushleft}
$\quad
  \arPingPongExpc = \ruletup{ \condPingPongExpf \targetarrow \mbox{``Ping-Pong''}, ~ \mbox{``Not Ping-Pong''} }
$
\end{flushleft}

\noindent
In this case, $\arPingPongExpa$ specifies the task for predicting
ping-pong behaviour based on the characteristic provided by
$\condPingPongExpa$ (Similarly for $\arPingPongExpb$ and
$\arPingPongExpc$).
These \analrules can be fed into our tool in order to obtain the
prediction model, and for these cases we create classification
models.
%
%

In BPIC 2013 event log, an incident can have several statuses. One of
them is waiting. In this experiment, we predict the \emph{remaining
  duration of all waiting-related events} by specifying the following
\analrule:
\[
\begin{array}{l}
  \arRemWaiting = \ruletup{ \curr < \last \targetarrow \rep{RemWaitingTime}, 0}
\end{array}
\]
where $\rep{RemWaitingTime}$ is as follows:
\[
\begin{array}[t]{l}
  \sumd{\eventquery{x+1}{time:timestamp} -
  \eventquery{x}{time:timestamp}}{\curr}{\last}{
\begin{array}[t]{l}
  \eventquery{x}{lifecycle:transition} \logeq \mbox{``Await.\! Assign."} \for \\
  \eventquery{x}{lifecycle:transition} \logeq \mbox{``Wait"}  \for \\ 
  \eventquery{x}{lifecycle:transition} \logeq \mbox{``Wait - Impl."}  \for \\ 
  \eventquery{x}{lifecycle:transition} \logeq \mbox{``Wait - User"}  \for \\ 
  \eventquery{x}{lifecycle:transition} \logeq \mbox{``Wait - Cust."}  \for \\ 
  \eventquery{x}{lifecycle:transition} \logeq \mbox{``Wait - Vendor"}
)
\end{array}
    }
\end{array}
\]
i.e., $\rep{RemWaitingTime}$ is the sum of all event duration in which
the status is related to waiting (e.g., {Awaiting
  Assignment}, {Wait}, {Wait-User}, etc). 
%
%
%
Similarly, we predict the \emph{remaining duration of all (exactly)
  waiting events} by specifying the following:
\[
\begin{array}{l}
  \arRemWaitingb = \ruletup{ \curr < \last \targetarrow \rep{RemWaitDur}, 0}
\end{array}
\]
where $\rep{RemWaitDur}$ is as follows:
\[
\begin{array}[t]{l}
  \sumd{\eventquery{x+1}{time:timestamp} -
  \eventquery{x}{time:timestamp}}{\curr}{\last}{
\begin{array}[t]{l}
  \eventquery{x}{lifecycle:transition} \logeq \mbox{``Wait"}  )
\end{array}
    }
\end{array}
\]
i.e., $\rep{RemWaitDur}$ is the sum of all event duration in which
the status is ``wait''.
Both $\arRemWaiting$ and $\arRemWaitingb$ can be fed into our tool,
and in this case we generate regression models.
%

For all of these tasks, we consider two different trace
encodings. First, we use the trace encoding that incorporates several
available event attributes, namely concept:name, org:resource,
org:group, lifecycle:transition, organization involved, impact,
product, resource country, organization country, org:role. Second, we
use the trace encoding that only incorporates the event names, i.e.,
the values of the attribute concept:name.
Intuitively, the first encoding considers more information than the
second encoding. Thus, the prediction models that are obtained by
using the first encoding use more input information for doing the
prediction.
The evaluation on the generated prediction models from all prediction
tasks specified above is reported in
\Cref{tab:exp-bpi13-1,tab:exp-bpi13-2}.




\begin{table*}
  \caption{The results from the experiments on BPIC 2013 event log
    using prediction tasks
    $\arPingPongExpa$, $\arPingPongExpb$, and $\arPingPongExpc$} 
\centering 
\begin{tabular}{| c | l | c | c | c | c | c || c | c | c | c | c |} 
\hline
&\multicolumn{11}{c|}{Experiments with the \analrule $\arPingPongExpa$ (change of group while the concept:name is not 'queued')}\\
\cline{2-12}
\multirow{10}{*}{$\arPingPongExpa$}&\multirow{2}{*}{Model}&\multicolumn{5}{c||}{1st encoding (more features)}&\multicolumn{5}{c|}{2nd encoding (less features)}\\
\cline{3-12}
&& \auc & \acc & \wprecision & \wrecall & \fmeasure & \auc & \acc & \wprecision & \wrecall & \fmeasure \\ 
\cline{2-12}
  &\zeroR & 0.50 & 0.82 & 0.68 & 0.82 & 0.75 & 0.50 & {0.82} & 0.68 & {0.82} & 0.75 \\ 
\cline{2-12}
  &\logreg & 0.64 & 0.81 & 0.75 & 0.81 & 0.76& 0.55 & {0.82} & 0.68 & {0.82} & 0.75 \\ 
\cline{2-12}
  &\nbayes  & 0.51 & 0.21 & 0.80 & 0.21 & 0.12& 0.54 & 0.19 & \textbf{0.79} & 0.19 & 0.09 \\ 
\cline{2-12}
  &\dectree & 0.67 & 0.78 & 0.80 & 0.78 & 0.79& \textbf{0.68} & {0.82} & 0.76 & {0.82} & \textbf{0.77} \\ 
\cline{2-12}
  &\randfor & \textbf{0.83} & \textbf{0.84} & \textbf{0.83} & \textbf{0.84} & \textbf{0.83} & \textbf{0.68} & {0.82} & 0.76 & {0.82} & \textbf{0.77} \\ 
\cline{2-12}
 &\adaboost & 0.73 & 0.81 & 0.77 & 0.81 & 0.78& 0.66 & {0.82} & 0.75 & {0.82} & 0.75 \\ 
\cline{2-12}
  &\extratree & 0.81 & 0.83 & 0.81 & 0.83 & 0.82& \textbf{0.68} & {0.82} & 0.76 & {0.82} & \textbf{0.77} \\ 
\cline{2-12}
  &\voting & 0.81 & 0.81 & 0.81 & 0.81 & 0.81& \textbf{0.68} & {0.82} & 0.76 & {0.82} & \textbf{0.77} \\ 
\cline{2-12}
  &\nn & 0.73 & 0.83 & 0.81 & 0.83 & 0.81 & \textbf{0.68} & \textbf{0.83} & 0.78 & \textbf{0.83} & 0.75\\
\cline{1-12}
\cline{1-12}
 
&\multicolumn{11}{c|}{Experiments with the \analrule $\arPingPongExpb$ (change of people/group and change back to the original person/group) }\\
\cline{2-12}
\multirow{10}{*}{$\arPingPongExpb$}&\multirow{2}{*}{Model}&\multicolumn{5}{c||}{1st encoding (more features)}&\multicolumn{5}{c|}{2nd encoding (less features)}\\
\cline{3-12}
&& \auc & \acc & \wprecision & \wrecall & \fmeasure & \auc & \acc & \wprecision & \wrecall & \fmeasure \\ 
\cline{2-12}
 & \zeroR & 0.50 & 0.79 & 0.63 & 0.79 & 0.70 & 0.50 & 0.79 & 0.63 & 0.79 & 0.70\\ 
\cline{2-12}
&  \logreg & 0.77 & 0.82 & 0.80 & 0.82 & 0.80 & 0.62 & 0.81 & 0.78 & 0.81 & 0.76\\ 
\cline{2-12}
  &\nbayes   & 0.69 & 0.79 & 0.75 & 0.79 & 0.75 & 0.63 & 0.80 & 0.77 & 0.80 & 0.76 \\ 
\cline{2-12}
 & \dectree & 0.73 & 0.82 & 0.82 & 0.82 & 0.82 & 0.76 & 0.82 & 0.80 & 0.82 & 0.80 \\ 
\cline{2-12}
  &\randfor  & \textbf{0.85} & \textbf{0.86} & {0.85} & \textbf{0.86} & 0.85 & \textbf{0.78} & 0.82 & 0.80 & 0.82 & 0.80  \\ 
\cline{2-12}
  &\adaboost & 0.81 & 0.84 & 0.83 & 0.84 & 0.83 & 0.68 & 0.81 & 0.79 & 0.81 & 0.77 \\ 
\cline{2-12}
  &\extratree & \textbf{0.85} & \textbf{0.86} & {0.85} &
                                                                \textbf{0.86} & \textbf{0.86} & \textbf{0.78} & 0.82 & 0.80 & 0.82 & 0.80 \\
\cline{2-12}
  &\voting & \textbf{0.85} & \textbf{0.86} & {0.85} & \textbf{0.86} & 0.85 & 0.77 & 0.82 & {0.81} & 0.82 & \textbf{0.81} \\ 
\cline{2-12}
  &\nn & 0.77 & \textbf{0.86} & \textbf{0.86} & \textbf{0.86} &{0.85} & \textbf{0.78} & \textbf{0.83} & \textbf{0.82} & \textbf{0.83} & 0.80\\
\cline{1-12}
\cline{1-12}

&\multicolumn{11}{c|}{Experiments with the \analrule $\arPingPongExpc$ (involves at least three different groups)}\\
\cline{2-12}
\multirow{10}{*}{$\arPingPongExpc$}&\multirow{2}{*}{Model}&\multicolumn{5}{c||}{1st encoding (more features)}&\multicolumn{5}{c|}{2nd encoding (less features)}\\
\cline{3-12}
&& \auc & \acc & \wprecision & \wrecall & \fmeasure & \auc & \acc & \wprecision & \wrecall & \fmeasure \\ 
\cline{2-12}
  &\zeroR & 0.50 & 0.74 & 0.54 & 0.74 & 0.63 & 0.50 & 0.74 & 0.54 & 0.74 & 0.63\\ 
\cline{2-12}
  &\logreg  & 0.78 & 0.78 & 0.76 & 0.78 & 0.76 & 0.77 & 0.79 & 0.77 & 0.79 & 0.77\\ 
\cline{2-12}
  &\nbayes & 0.75 & 0.76 & 0.73 & 0.76 & 0.70 & 0.76 & 0.77 & 0.75 & 0.77 & 0.73\\ 
\cline{2-12}
  &\dectree & 0.79 & 0.82 & 0.83 & 0.82 & 0.83 & 0.81 & {0.82} & \textbf{0.82} & {0.82} & \textbf{0.82}\\ 
\cline{2-12}
  &\randfor  & \textbf{0.92} & \textbf{0.87} & \textbf{0.87} & \textbf{0.87} & \textbf{0.87} & \textbf{0.83} & {0.82} & \textbf{0.82} & {0.82} & \textbf{0.82}\\ 
\cline{2-12}
  &\adaboost  & 0.89 & 0.86 & 0.86 & 0.86 & 0.86 & \textbf{0.83} & 0.81 & 0.80 & 0.81 & 0.80\\ 
\cline{2-12}
  &\extratree & 0.91 & \textbf{0.87} & \textbf{0.87} & \textbf{0.87} & \textbf{0.87}  & 0.82 & {0.82} & \textbf{0.82} & {0.82} & \textbf{0.82}\\
\cline{2-12}
  &\voting  & 0.91 & 0.85 & 0.85 & 0.85 & 0.85  & 0.82 & {0.82} & 0.81 & {0.82} & \textbf{0.82}\\ 
\cline{2-12}
  &\nn & {0.85} & {0.85} & {0.84} & {0.85} & {0.84}& \textbf{0.83} & \textbf{0.83} & \textbf{0.82} & \textbf{0.83} & \textbf{0.82}\\
\cline{1-12}

\end{tabular}  
\label{tab:exp-bpi13-1}
\end{table*} 

\begin{table*}
  \caption{The results of the experiments on BPIC 2013 event log
    using prediction tasks
    $\arRemWaiting$ and $\arRemWaitingb$} 
\centering 
\begin{tabular}{| c | l | c | c || c | c |} 


\hline
&\multicolumn{5}{c|}{Experiments with the \analrule $\arRemWaiting$ (the remaining duration of all waiting-related events)}\\
\cline{2-6}
\multirow{8}{*}{$\arRemWaiting$}&\multirow{2}{*}{Model}&\multicolumn{2}{c||}{1st Encoding (more features)} & \multicolumn{2}{c|}{2nd Encoding (less features)}\\
\cline{3-6}
&&\mae (in days)&\rmse (in days) &\mae (in days)&\rmse (in days)\\
\cline{2-6}
&\zeroRreg & 5.977 & {6.173} &5.977 & {6.173} \\ 
\cline{2-6}
&\linreg & 5.946 & 6.901 & 6.16 & 6.462\\ 
\cline{2-6}
&\dectreereg & 5.431 & 17.147 & 5.8 & 7.227 \\ 
\cline{2-6}
&\randforreg & 4.808 & 8.624 & 5.81 & 7.114 \\ 
\cline{2-6}
&\adaboostreg  & 14.011 & 18.349 & 14.181 & 15.164 \\ 
\cline{2-6}
&\extratreereg & {4.756} & 8.612 & {5.799} & 7.132\\ 
\cline{2-6}
&\nnreg & \textbf{2.205} & \textbf{4.702} &\textbf{4.064} & \textbf{4.596}\\ 
\hline



\hline
&\multicolumn{5}{c|}{Experiments with the \analrule $\arRemWaitingb$
  (the remaining duration of all events in which the status is ``wait'')}\\
\cline{2-6}
\multirow{8}{*}{$\arRemWaitingb$}&\multirow{2}{*}{Model}&\multicolumn{2}{c||}{1st Encoding (more features)} & \multicolumn{2}{c|}{2nd Encoding (less features)}\\
\cline{3-6}
&&\mae (in days)&\rmse (in days) &\mae (in days)&\rmse (in days)\\
\cline{2-6}
&\zeroRreg & 1.061 & \textbf{1.164} & 1.061 & {1.164} \\ 
\cline{2-6}
&\linreg & 1.436 & 1.974 &  1.099 & 1.233 \\ 
\cline{2-6}
&\dectreereg & {0.685} & 5.165 & {1.003} & 1.66 \\ 
\cline{2-6}
&\randforreg & 0.713 & 3.396 & 1.016 & 1.683\\ 
\cline{2-6}
&\adaboostreg  & 1.507 & 3.89 & 1.044 & 1.537\\ 
\cline{2-6}
&\extratreereg & 0.843 & 3.719 & 1.005 & 1.649 \\ 
\cline{2-6}
&\nnreg &  \textbf{0.37} & 2.037 & \textbf{0.683} & \textbf{0.927}\\ 
\hline


\end{tabular}  
\label{tab:exp-bpi13-2}
\end{table*} 

\subsection{Experiment on BPIC 2012 Event Log}

The event log for BPIC 2012\footnote{More information on BPIC 2012 can
  be found in
  \url{http://www.win.tue.nl/bpi/doku.php?id=2012:challenge}}~\cite{BPI-12-data}
comes from a Dutch financial institute. 
It stores the information concerning the process of handling either
personal loan or overdraft application.
It contains 13.087 traces (process instances) and 262.200 events.
%
%
Generally, the process of handling an application is as follows: Once
an application is submitted, some checks are performed. After that,
the application is augmented with necessary additional information
that is obtained by contacting the client by phone. An offer will be
send to the client, if the applicant is eligible. After this offer is
received back, it is assessed. The customer will be contacted again if
there is missing information. After that, a final assessment is
performed.
%
%
In this experiment, we consider two prediction task as follows:

\begin{enumerate}
\item 
%
%
  One type of activity within this process is named
  \emph{W\_Completeren aanvraag}, which stands for ``{Filling in
    information for the application}''. The task for predicting the
  \emph{total duration of all remaining activities of this type} is
  formulated as follows:
\[
    \begin{array}{l}
      \arRemTimeFillingApplication = \ruletup{ \curr < \last \targetarrow \rep{RemTimeFillingInfo}, 0}
    \end{array}
\]
  where $\rep{RemTimeFillingInfo}$ is as follows:
\[
{
      \begin{array}[t]{l}
        \sumd{\eventquery{x+1}{time:timestamp} -
        \eventquery{x}{time:timestamp}}{\curr}{\last}{
        \begin{array}[t]{l}
          \eventquery{x}{concept:name} \logeq \\
          \qquad \qquad \qquad\mbox{``W\_Completeren aanvraag"})
        \end{array}
        }
      \end{array}
}
\]
i.e., it computes the sum of the duration of all remaining
\emph{W\_Completeren aanvraag} activities.

\item At the end of the process, an application can be declined. The
  task to predict whether an application will eventually be declined
  is specified as follows:
%
\[
\arEventuallyDeclined = \ruletup{ \condEventuallyDeclined \targetarrow \mbox{``Declined}, ~ \mbox{``Not\_Declined} }
\]
%
where $\condEventuallyDeclined$ is as follows:
%
\[
\begin{array}{r@{}l}
\condEventuallyDeclined =  \fexists i . (&i > \curr \fand\\
                                           &\eventquery{i}{concept:name}
                                           \logeq \mbox{``A\_DECLINED"})
\end{array}
\]
i.e., $\condEventuallyDeclined$ says that eventually there will be an
event in which the application is declined. 

\end{enumerate}


Both $\arRemTimeFillingApplication$ and $\arEventuallyDeclined$ can be
fed into our tool. For $\arRemTimeFillingApplication$, we generate a
regression model, while for $\arEventuallyDeclined$, we generate a
classification model. Different from the BPIC 2013 and BPIC 2015 event
logs, there are not so many event attributes in this log. For all of
these tasks, we consider two different trace encodings. First, we use
the trace encoding that incorporates several available event
attributes, namely concept:name and lifecycle:transition. Second, we
use the trace encoding that only incorporates the event names, i.e.,
the values of the attribute concept:name.
Thus, intuitively the first encoding considers more information than
the second encoding.
The evaluation on the generated prediction models from the prediction
tasks specified above is shown in
\Cref{tab:exp-bpi12-1,tab:exp-bpi12-2}.

\begin{table*}
  \caption{The results of the experiments on BPIC 2012 event log
    using the prediction task $\arRemTimeFillingApplication$} 
\centering 
\begin{tabular}{| c | l | c | c || c | c |} 


\hline
&\multicolumn{5}{c|}{Experiments with the \analrule $\arRemTimeFillingApplication$
  (Total Duration of all remaining activites named 'W\_Completeren aanvraag')}\\
\cline{2-6}
\multirow{8}{*}{ $\arRemTimeFillingApplication$ }&\multirow{2}{*}{Model}&\multicolumn{2}{c||}{1st Encoding (more features)} & \multicolumn{2}{c|}{2nd Encoding (less features)}\\
\cline{3-6}
&& \quad\mae  (in days) \quad \qquad &\quad\rmse  (in days)\quad \qquad &\quad\mae (in days) \quad \qquad&\quad\rmse (in days) \quad \qquad\\
\cline{2-6}
&\zeroRreg & 3.963 & 5.916 & 3.963 & 5.916 \\ 
\cline{2-6}
&\linreg & 3.613  & 5.518 & 3.677 & 5.669\\ 
\cline{2-6}
&\dectreereg & 2.865 & 5.221 & 2.876 & 5.228 \\ 
\cline{2-6}
&\randforreg & 2.863 & 5.198 & 2.877 & 5.213\\ 
\cline{2-6}
&\adaboostreg  &  3.484 & 5.655 & 3.484 & 5.655 \\ 
\cline{2-6}
&\extratreereg & {2.857} & \textbf{5.185} & {2.868} & \textbf{5.191}\\ 
\cline{2-6}
&\nnreg & \textbf{2.487} & 5.683 & \textbf{2.523} & 5.667\\ 
\hline


\end{tabular}  
\label{tab:exp-bpi12-1}
\end{table*} 

\begin{table*}
  \caption{The results from the experiments on BPIC 2012 event log
    using the prediction task  $\arEventuallyDeclined$} 
\centering 
\begin{tabular}{| c | l | c | c | c | c | c || c | c | c | c | c |} 
\hline
&\multicolumn{11}{c|}{Experiments with the \analrule $\arEventuallyDeclined$
  (predict whether an application will be eventually 'DECLINED')}\\
\cline{2-12}
  \multirow{10}{*}{$\arEventuallyDeclined$}&\multirow{2}{*}{Model}&\multicolumn{5}{c||}{1st encoding (more features)}&\multicolumn{5}{c|}{2nd encoding (less features)}\\
  \cline{3-12}
&& \auc & \acc & \wprecision & \wrecall & \fmeasure & \auc & \acc & \wprecision & \wrecall & \fmeasure \\ 
  \cline{2-12}
&\zeroR &  0.50 & 0.78 & 0.61 & 0.78 & 0.68 &  0.50 & 0.78 & 0.61 & 0.78 & 0.68 \\ 
  \cline{2-12}
&\logreg &  0.69 & 0.78 & 0.75 & 0.78 & 0.76 & 0.69 & 0.77 & 0.71 & 0.77 & 0.71 \\ 
  \cline{2-12}
&\nbayes  & 0.67 & 0.33 & 0.74 & 0.33 & 0.30  & 0.67 & 0.33 & 0.73 & 0.33 & 0.30\\ 
  \cline{2-12}
&\dectree &  0.70 & 0.78 & 0.76 & 0.78 & 0.77 & 0.70 & 0.78 & 0.76 & 0.78 & 0.77\\ 
  \cline{2-12}
&\randfor &  \textbf{0.71} & 0.79 & 0.77 & 0.79 & \textbf{0.78} & \textbf{0.71} & 0.79 & 0.77 & 0.79 & \textbf{0.78} \\ 
  \cline{2-12}
&\adaboost &  \textbf{0.71} & \textbf{0.81} & \textbf{0.78} &
                                                              \textbf{0.81} & \textbf{0.78} & \textbf{0.71} & \textbf{0.80} & \textbf{0.78} & \textbf{0.80} & \textbf{0.78} \\ 
  \cline{2-12}
&\extratree &\textbf{0.71} & 0.79 & 0.77 & 0.79 & \textbf{0.78} & \textbf{0.71} & 0.79 & 0.77 & 0.79 & \textbf{0.78} \\ 
  \cline{2-12}
&\voting & \textbf{0.71} & 0.79 & 0.77 & 0.79 & \textbf{0.78} & \textbf{0.71} & 0.79 & 0.77 & 0.79 & 0.77\\ 
  \cline{2-12}
&\nn & \textbf{0.71} & {0.80} & {0.77} & {0.80} & \textbf{0.78}& \textbf{0.71} & \textbf{0.80} & \textbf{0.78} & \textbf{0.80} & \textbf{0.78}\\
  \cline{1-12}
  \cline{1-12}

\end{tabular}  
\label{tab:exp-bpi12-2}
\end{table*} 

\subsection{Experiment on BPIC 2015 Event Log}
In BPIC 2015\footnote{More information on BPIC 2015 can be found in
  \url{http://www.win.tue.nl/bpi/doku.php?id=2015:challenge}}~\cite{BPI-15-data},
5 event logs from 5 Dutch Municipalities are provided. They contain
the data of the processes for handling the building permit
application. In general, the processes in these 5 municipalities are
similar. Thus, in this experiment we only consider one of these logs.
There are several information available such as the activity name and
the resource/person that carried out a certain task/activity.
%
%
The statistic about the log that we consider is as follows: it has
1409 traces (process instances) and 59681 events.

For this event log, we consider several tasks related to predicting
workload-related information (i.e., related to the amount of
work/activities need to be done). First, we deal with the task for
predicting whether a process of handling an application is complex or
not based on the number of the remaining different activities that
need to be done. Specifically, we consider a process is complex (or
need more attention) if there are still more than 25 different
activities need to be done. This task can be specified as follows:
\[
\arComplexAppB = \langle \rep{NumDifRemAct} \geq 25 \targetarrow
  \mbox{``Complex"}, \mbox{``Normal"} \rangle 
\]
%
where $\rep{NumDifRemAct}$ is specified as follows: 
%
\[
\countvala{activityNameEN}{\curr}{\last}
\]
%
i.e., $\rep{NumDifRemAct}$ counts the number of different values of
the attribute 'activityNameEN' from the current time point until the
end of the process. As the next workload-related prediction task, we
specify the task for predicting the number of remaining
events/activities 
as follows:
\begin{center}
$
\begin{array}{l}
  \arRemainingAct = \ruletup{ \curr < \last \targetarrow \rep{RemAct}, 0}
\end{array}
$
\end{center}
where $\rep{RemAct} = \counta{\true}{\curr}{\last}$, i.e.,
$\rep{RemAct}$ counts the number of events/activities from the current
time point until the end of the process.

Both $\arComplexAppB$ and $\arRemainingAct$ can be fed into our
tool. For the former, we generate a classification model, and for the
latter, we generate a regression model. For all of these tasks, we
consider two different trace encodings. First, we use the trace
encoding that incorporates several available event attributes, namely
monitoringResource, org:resource, activityNameNL, activityNameEN,
question, concept:name. Second, we use the trace encoding that only
incorporates the event names, i.e., the values of the attribute
concept:name.
As before, the first encoding considers more information than the
second encoding. 
The evaluation on the generated prediction models from the prediction
tasks specified above is shown in
\Cref{tab:exp-bpi15-1,tab:exp-bpi15-2}


\begin{table*}
  \caption{The results from the experiments on BPIC 2015 event log
    using the prediction task $\arComplexAppB$} 
\centering 
\begin{tabular}{| c | l | c | c | c | c | c || c | c | c | c | c |} 
\hline
&\multicolumn{11}{c|}{Experiments with the \analrule $\arComplexAppB$
  (Predicting whether a process is complex)}\\
\cline{2-12}
\multirow{10}{*}{$\arComplexAppB$}&\multirow{2}{*}{Model}&\multicolumn{5}{c||}{1st encoding (more features)}&\multicolumn{5}{c|}{2nd encoding (less features)}\\
\cline{3-12}
&& \auc & \acc & \wprecision & \wrecall & \fmeasure & \auc & \acc & \wprecision & \wrecall & \fmeasure \\ 
\cline{2-12}
  &\zeroR & 0.50 & 0.57 & 0.32 & 0.57 & 0.41 & 0.50 & 0.57 & 0.32 & 0.57 & 0.41 \\ 
\cline{2-12}
  &\logreg & 0.92 & 0.83 & 0.85 & 0.83 & 0.83 & 0.90 & 0.84 & 0.84 & 0.84 & 0.83\\ 
\cline{2-12}
  &\nbayes  & 0.81 & 0.72 & 0.82 & 0.72 & 0.71 & 0.93 & 0.68 & 0.81 & 0.68 & 0.66\\ 
\cline{2-12}
  &\dectree & 0.80 & 0.79 & 0.80 & 0.79 & 0.80 & 0.84 & 0.85 & 0.85 & 0.85 & 0.85 \\ 
\cline{2-12}
&\randfor &\textbf{0.95} & \textbf{0.89} & \textbf{0.89} & \textbf{0.89} & \textbf{0.89}   & \textbf{0.95} & \textbf{0.90} & \textbf{0.90} & \textbf{0.90} & \textbf{0.90} \\ 
  \cline{2-12}
&\adaboost & 0.92 & 0.87 & 0.87 & 0.87 & 0.87 & 0.93 & 0.88 & 0.88 & 0.88 & 0.88 \\ 
  \cline{2-12}
&\extratree &\textbf{0.95} & 0.88 & 0.88 & 0.88 & 0.88 & \textbf{0.95} & 0.88 & 0.89 & 0.88 & 0.88\\ 
  \cline{2-12}
&\voting & 0.94 & 0.85 & 0.86 & 0.85 & 0.86 & \textbf{0.95} & 0.88 & 0.88 & 0.88 & 0.88\\ 
  \cline{2-12}
&\nn & {0.89} & {0.84} & {0.84} & {0.84} & {0.84} & {0.92} & {0.84} & 0.84 & {0.84} & {0.84}\\
  \cline{1-12}
  \cline{1-12}

\end{tabular}  
\label{tab:exp-bpi15-1}
\end{table*} 

\begin{table*}
  \caption{The results of the experiments on BPIC 2015 event log
    using the prediction task $\arRemainingAct$} 
\centering 
\begin{tabular}{| c | l | c | c || c | c |} 


\hline
&\multicolumn{5}{c|}{Experiments with the \analrule $\arRemainingAct$
  (the number of the remaining events/activities)}\\
\cline{2-6}
\multirow{8}{*}{ $\arRemainingAct$ }&\multirow{2}{*}{Model}&\multicolumn{2}{c||}{1st Encoding (more features)} & \multicolumn{2}{c|}{2nd Encoding (less features)}\\
\cline{3-6}
&& \qquad\mae \qquad \qquad &\qquad\rmse \qquad \qquad &\qquad\mae \qquad \qquad&\qquad\rmse \qquad \qquad\\
\cline{2-6}
&\zeroRreg & 11.21 & 13.274 & 11.21 & 13.274 \\ 
\cline{2-6}
&\linreg & 6.003 & 7.748 & 14.143 & 18.447 \\ 
\cline{2-6}
&\dectreereg & 6.972 & 9.296 & 6.752 & 9.167  \\ 
\cline{2-6}
&\randforreg & 4.965 & 6.884 & 4.948 & 6.993 \\ 
\cline{2-6}
&\adaboostreg  & 4.971 & 6.737 & 4.879 & 6.714 \\ 
\cline{2-6}
&\extratreereg &\textbf{4.684} & \textbf{6.567} & \textbf{4.703} & \textbf{6.627}\\ 
\cline{2-6}
&\nnreg & {6.325} & {8.185} & {5.929} & {7.835}  \\ 
\hline


\end{tabular}  
\label{tab:exp-bpi15-2}
\end{table*} 

\subsection{Discussion on the Experiments}

In total, our experiments involve
\totaltasksforexperiments 
different prediction tasks 
over 3 different real-life event logs from 3 different domains (1
event log from BPIC 2015, 1 event log from BPIC 2012, and 1 event log
from BPIC 2013).

Overall, these experiments show the capabilities of our language in
capturing and specifying the desired prediction tasks that are based
on the event logs coming from real-life situation. These experiments
also exhibit the applicability of our approach in automatically
constructing reliable prediction models based on the given
specification.
This is supported by the following facts:
first, for all prediction tasks that we have considered, by
considering different input features and machine learning models, we
are able to obtain prediction models that beat the baseline.
Moreover, for all prediction tasks that predict categorical values, in
our experiments we are always able to get a prediction model that has
AUC value greater than 0.5. Recall that AUC~=~0.5 indicates the worst
classifier that is not better than a random guess. Thus, since we have
AUC~>~0.5, the prediction models that we generate certainly take into
account the given input and predict the most probable output based on
the given input, instead of randomly guessing the output no matter
what the input is. In fact, in many cases, we could even get very high
AUC values which are ranging between 0.8 and 0.9 (see
\Cref{tab:exp-bpi13-1,tab:exp-bpi15-1}). This score is very close to
the AUC value for the best predictor (recall that AUC~=~1 indicates
the best classifier).

As can be seen from the experiments, the choice of the input features
and the machine learning models 
influence the quality of the prediction model. 
The result of our experiments also shows that there is no single
machine learning model that always outperforms other models on every
task.
Since our approach does not rely on a particular machine learning
model,
it justifies that we can simply plug in different supervised machine
learning techniques in order to get different or better
performance. In fact, 
in our experiments, by considering different models we could get
different/better prediction quality.
Concerning the input features, for each task in our experiments, we
intentionally consider two different input encodings.  The first one
includes many attributes (hence it incorporates many information), and
the second one includes only a certain attribute (i.e., it
incorporates less information).
In general, our common sense would expect that the more
information, 
the better the prediction quality would be. 
This is because we thought that, by having more information, we
have a more holistic view of the situation.
Although many of our experiment results show this fact, there are
several cases where considering less features could give us a better
result, e.g., the RMSE score in the experiment with several models on
the task $\arRemWaitingb$, and the scores of several metrics in the
experiment $\arComplexAppB$ show this fact (see
\Cref{tab:exp-bpi13-2,tab:exp-bpi15-1}).
%
%
In fact, this is aligned with the typical observation in machine
learning. The presence of irrelevant features could decrease the
prediction quality.
Although in the learning process a good model should (or will try to)
ignore 
irrelevant features, the absence of these unrelated features might
make the learning process better and might improve the quality of the
prediction.
%
%
%
%
%
%
%
%
Additionally, in some situation, too many features might cause
overfitting, i.e., the model fits the training data very well, but it
fails to generalize well while doing prediction on the new data.


Based on the experience from these experiments, time constraint would
also be a crucial factor in choosing the model when we would like to
apply this approach in practice. Some models require a lot of tuning
in order to achieve a good performance (e.g., neural network), while
other models do not need many adjustment and able to achieve
relatively good performance (e.g., Extra Trees, Random Forest).

Looking at another perspective, our experiments complement various
studies in the area of predictive process monitoring in several ways.
First, instead of using machine learning models that are typically
used in many studies within this area such as Random Forest and
Decision Tree (cf.~\cite{MFDG14,VDLMD15,DDFGMR16,DDFT16}), we also
consider 
other machine learning models that, to the best of our knowledge, are
not typically used. For instance, we use Extra Trees, Ada Boost,
and Voting Classifier.
%
%
Thus, we provide a fresh insight on the performance of these machine
learning models in predictive process monitoring by using them in
various different prediction tasks (e.g., predicting (fine-grained)
time-related information, unexpected behaviour).
Although this work is not aimed at comparing various machine learning
models, as we see from the experiments, in several cases, Extra Trees
exhibits similar performance (in terms of accuracy) as Random
Forest. There are also some cases where it outperforms the Random
Forest (e.g., see the experiment with the task $\arRemainingAct$ in
\Cref{tab:exp-bpi15-2}). In the experiment with the task
$\arEventuallyDeclined$, {AdaBoost} outperforms all other models.
%
%
%
Regarding the type of the prediction tasks, we also look into the
tasks that are not yet highly explored in the literature within the
area of predictive process monitoring. For instance, while there are
numerous works on predicting the remaining processing time, to the
best of our knowledge, there is no literature exploring a more
fine-grained task such as the prediction of the remaining duration of
a particular type of event (e.g., predicting the duration of all
remaining waiting events). We also consider several workload-related
prediction tasks, which is rarely explored in the area of predictive
process monitoring.

Concerning the Deep Learning approach, 
there have
been several studies that explore the usage of Deep Neural Network for
predictive process monitoring
(cf.~\cite{TVLD17,ERF17a,ERF17b,DGMPY17,MEF17}). However, they focus
on predicting the name of the future activities/events, the next
timestamp, and the remaining processing time.
In this light, our experiments 
contribute new insights on exhibiting the usage of Deep Learning
approach in dealing with different prediction tasks other than just
those tasks.
%
%
Although the deep neural network does not always give the best result
in all tasks in our experiments, there are several interesting cases
where it shows a very good performance. Specifically, in the
experiments with the tasks $\arRemWaiting$ and $\arRemWaitingb$
(cf.~\Cref{tab:exp-bpi13-2}), where all other models cannot beat the
RMSE score of the baseline, the deep neural network comes to the
rescue and becomes the only model that could beat the RMSE score of
our baseline. 



\section{Related Work}
\label{sec:related-work}

This work is tightly related to 
the area of predictive analysis in business process management. In the
literature, there have been several works focusing on predicting
time-related properties of running 
processes. 
%
The works by~\cite{VPS10,ASS11,RW13,RW15,PSBD14,PSBD18} focus on
predicting the remaining processing time.
In~\cite{VPS10, ASS11}, the authors present an approach for predicting
the remaining processing time based on annotated transition system
that contains time information extracted from event logs.
%
The work by~\cite{RW13,RW15} proposes a technique for predicting the
remaining processing time using stochastic petri nets.
The works by~\cite{SWGM14,SWGM15,MFE12,PVFTW12} focus on predicting
delays in process execution. In~\cite{SWGM14,SWGM15}, the authors use
queueing theory to address the problem of delay prediction,
while~\cite{MFE12} explores the delay prediction in the domain of
transport and logistics process.
In~\cite{FGP12}, the authors present an ad-hoc predictive clustering
approach for predicting process performance. The authors
of~\cite{TVLD17} present a deep learning approach (using LSTM neural
network) for predicting the timestamp of the next event and use it to
predict the remaining cycle
time 
by repeatedly predicting the timestamp of the next event.

Looking at another perspective, the works
by~\cite{MFDG14,DDFT16,VDLMD15} focus on predicting the outcomes of a
running process.
The work by~\cite{MFDG14} introduces a framework for predicting the
business constraints compliance of a running process.
In~\cite{MFDG14}, the business constraints are formulated in
propositional Linear Temporal Logic (LTL), where the atomic
propositions are all possible events during the process executions.
The work by~\cite{DDFT16} improves the performance of~\cite{MFDG14} by
using a clustering preprocessing step.
Another work on outcomes prediction is presented by~\cite{PVWFT16},
which proposes an approach for predicting aggregate process outcomes
by taking into account the information about overall process
risk. Related to process risks,~\cite{CDLV13,CDLVT15} propose an
approach for risks prediction. The work by~\cite{MRRT17} presents an
approach based on evolutionary algorithm for predicting business
process indicators of a running process instance, where business
process indicator is a quantifiable metric that can be measured by
data that is generated by the
processes. 
%
%
The authors of~\cite{MF17} present a work on predicting business
constraint satisfaction. Particularly,~\cite{MF17} studies the impact
of considering the estimation of prediction
reliability 
on the costs of the processes.

Another major stream of works tackle the problem of predicting the
future activities/events of a running process
(cf.~\cite{TVLD17,ERF17a,ERF17b,DGMPY17,MEF17,BMDB16,PSBD18}).  The
works by~\cite{TVLD17,ERF17a,ERF17b,DGMPY17,MEF17} use deep learning
approach for predicting the future events, e.g., the next event of the
current running
process. Specifically,~\cite{TVLD17,ERF17a,ERF17b,DGMPY17} use LSTM
neural network, while~\cite{MEF17} uses deep feed-forward neural
network. In~\cite{PSBD18,DGMPY17,TVLD17} the authors also tackle the
problem of predicting the whole sequence of future events (the suffix
of the current running process).

A key difference 
between many of those works and ours is that, instead of focusing on
dealing with a particular prediction task (e.g., predicting the
remaining processing time or the next event), this work introduces a
specification language that enables us to specify various desired
prediction tasks for predicting various future information of a
running business process.
%
%
To deal with these various desired prediction tasks, we present a
mechanism to automatically process the given specification of
prediction task and to build the corresponding prediction model.
%
%
From another point of view, several works in this area often describe
the prediction tasks under study simply by using a (possibly
ambiguous) natural language. In this light, the presence of our
language complements this area by providing a means to formally and
unambiguously specifying/describing the desired prediction tasks.
Consequently, it could ease the definition of the task and the
comparison among different works that propose a particular prediction
technique for a particular prediction task.

%
Regarding the specification language, unlike the propositional
LTL~\cite{Pnue77}, 
which is the basis of Declare language~\cite{PV06,PSV07} and often
used for specifying business constraints over a sequence of events
(cf.~\cite{MFDG14}), our \langnameabr language (which is part of our
rule-based specification language) allows us not only to specify
properties over sequence of events but also to specify properties over
the data (attribute values) of the events, i.e., it is data-aware.
Concerning data-aware specification language, the work
by~\cite{BCDDM13} introduces a data-aware specification language by
combining data querying mechanisms and temporal logic. Such language
has been used in several works on verification of data-aware processes
systems (cf.~\cite{AS-ICSOC-13,thesis-as-16,AS-IJCAI-15,AS-JELIA-14}).
%
The works by~\cite{DMM14,MDGM13} provide a data-aware extension of the
Declare language based on the First-Order LTL (LTL-FO).
Although those languages are data-aware, they do not support
arithmetic expressions/operations over the data which is absolutely
necessary for our purpose, e.g., for expressing the time difference
between the timestamp of the first and the last event.
Another interesting data-aware language is S-FEEL, which is part of
the Decision Model and Notation (DMN) standard~\cite{OMG15} by OMG.
Though S-FEEL supports arithmetic expressions over the data, it does
not allow us to universally/existentially quantify different event
time points and to compare different event attribute values at
different event time points, which is important for our needs, e.g.,
in specifying the ping-pong behaviour.

Concerning aggregation, there are several formal languages that
incorporate such feature (cf.~\cite{DLT15,BGS13,BKMZ13}) and many of
them have been used in system monitoring. The work by~\cite{DLT15}
extends the temporal logic Past Time LTL with counting
quantifier. Such extension allows us to
express a constraint on the number of occurrences of events (similar
to our $\aggcount$ function). In~\cite{BGS13} a language called
SOLOIST is introduced and it supports several aggregate functions on
the number of event occurrences within a certain time
window. Differently from ours, both~\cite{DLT15} and~\cite{BGS13} do
not consider aggregation over data (attribute values). The works
by~\cite{BKMZ13,BKMZ15a} extend the temporal logic that was introduced
in~\cite{BKMP08,BKMZ15b} with several aggregate functions. Such
language allows us to select the values to be aggregated. However, due
to the interplay between the set and bag semantics in their language,
as they have illustrated,
some values might be lost while computing the aggregation because
they first collect the set of tuples of values that satisfy the
specified condition and then they collect the bag of values to be
aggregated from that set of tuples of values. To avoid this situation,
they need to make sure that each tuple of values has a sort of unique
identifier. This situation does not happen in our aggregation because,
in some sense, we directly use the bag semantics while collecting the
values to be aggregated.

Importantly, unlike those languages above, apart from allowing us to
specify a complex constraint/pattern, a fragment of our \langnameabr
language also allows us to specify the way to compute certain values,
which is needed for specifying the way to compute the target/predicted
values, e.g., the remaining processing time, or the remaining number
of a certain activity/event. Our language is also specifically tuned
for expressing data-aware properties based on the typical structure of
business process execution logs (cf.~\cite{IEEE-XES:2016}), and the
design is highly driven by the typical prediction tasks in business
process management.
From another point of view, our work complements the works on
predicting SLA/business constraints compliance by providing an
expressive language to specify complex data-aware
constraints that may involve arithmetic expression and data
aggregation.


\section{Discussion}
\label{sec:discussion}

This section discusses potential limitations of this work, which might
pave the way towards our future direction.

This work focuses on the problem of predicting the future information
of a single running process based on the current information of that
corresponding running process. In practice, there could be several
processes running concurrently. Hence, it is absolutely interesting to
extend the work further so as to consider the prediction problems on
concurrently running processes. This extension would involve the
extension of the language itself. For instance, the language should be
able to specify some patterns over multiple running
processes. Additionally, it should be able to express the desired
predicted information or the way to compute the desired predicted
information, and it might involve the aggregation of information over
multiple running processes. Consequently, the mechanism for building
the corresponding prediction model needs to be adjusted.

Our experiments (cf.~\Cref{sec:implementation-experiment}) show a
possible instantiation of our generic approach in creating prediction
services. In this case we predict the future information of a running
process by only considering the information from a single running
process. However, in practice, other processes that are concurrently
running might affect the execution of other processes. For instance,
if there are so many processes running together and there are not
enough employees for handling all processes simultaneously, some
processes might need to wait. Hence, when we predict the remaining
duration of waiting events, the current workload information might be
a factor that need to be considered and ideally these information
should be incorporated in the prediction. One possibility to overcome
this limitation is to use the trace encoding function that
incorporates the information related to the processes that are
concurrently running. For instance, we can make an encoding function
that extracts relevant information from all processes that are
concurrently running, and use them as the input features.
Such information could be the number of employees that are actively
handling some processes, the number of available resources/employees,
the number of processes of a certain type that are currently running,
etc.

This kind of machine learning based technique performs the prediction
based on the observable information. Thus, if the information to be
predicted depends on some unobservable factors, the quality of the
prediction might be decreasing. Therefore, in practice, all factors
that highly influence the information to be predicted should be
incorporated as much as possible. Furthermore, the prediction model is
only built based on the historical information about the previously
running processes and neglects the possibility of the existence of the
domain knowledge (e.g., some organizational rules) that might
influence the prediction. 
In some sense, it (implicitly) assumes that the domain knowledge is
already incorporated in those historical data that captures the
processes execution in the past. Obviously, it is then interesting to
develop the technique further so as to incorporate the existing domain
knowledge in the creation of the prediction model with the aim of
enhancing the prediction quality.
Looking at another perspective, since the prediction model is only
built based on the historical data of the past processes execution,
this approach is absolutely suitable for the situation in which the
(explicit) process model is unavailable or hard to obtain.




As also observed by other works in this area (e.g.,~\cite{ASS11}), in
practice, by predicting the future information of a running process,
we might affect the future of the process itself, and hence we might
reduce the preciseness of the prediction. For instance, when it is
predicted that a particular process would exhibit an unexpected
behaviour, we might be eager to prevent it by closely watching the
process in order to prevent that unexpected behaviour. In the end,
that unexpected behaviour might not be happened due to our preventive
actions, and hence the prediction is not happened. On the other hand,
if we predict that a particular process will run normally, we might put
less attention than expected into that process, and hence the
unexpected behaviour might occur. Therefore, knowing the (prediction
of the) future might not always be good for this case. This also
indicates that a certain care need to be done while using the
predicted information.

\section{Conclusion}
\label{sec:conclusion}

We have introduced an approach for obtaining predictive process
monitoring services based on the specification of the desired
prediction tasks.
%
%
%
Specifically, we proposed a novel rule-based language for specifying
the desired prediction tasks, and
we devise a mechanism 
for automatically building the corresponding prediction models based
on the given specification.
Establishing such language is a non-trivial task. The language should
be able to capture various prediction tasks, while at the same time
allowing us to have a procedure for building/deriving the
corresponding prediction model.
Our language is a logic-based language which is fully equipped with a
well-defined formal semantics. Therefore, it allows us to do formal
reasoning over the specification, and to have a machine processable
language that enables us to automate the creation of the prediction
model. The language allows us to express complex properties involving
data and arithmetic expressions. It also allows us to specify the way
to compute certain values. Notably, our language supports several
aggregate functions.
A prototype that implements our approach has been developed and
several experiments using real life event logs confirmed the
applicability of our approach.
Remarkably, our experiments involve the usage of a deep learning model
(In particular, we use the deep feed-forward neural network).


%
%

Apart from those that have been discussed in \Cref{sec:discussion},
the future work includes the extension of the tool and the language.
%
One possible extension would be to incorporate \emph{trace attribute
  accessor} that allows us to specify properties involving trace
attribute values.
As our \langnameabr language is a logic-based language, there is a
possibility to exploit existing logic-based tools such as
Satisfiability Modulo Theories (SMT) solver~\cite{BSST09} for
performing some reasoning tasks related to the
language. 
%
Experimenting with other supervised machine learning techniques would
be the next step as well, 
for instance by using another deep learning approach (i.e., another
type of neural network such as recurrent neural network) with the aim
of improving the prediction quality.


\begin{acknowledgements}
%
  We thank 
  T.\ K.\ Wijaya for various suggestions related to this work, and
  Yasmin K. 
  for the implementation of several prototype components.
\end{acknowledgements}

\bibliographystyle{spbasic}
\bibliography{string-tiny,main}

\begin{thebibliography}{67}
\providecommand{\natexlab}[1]{#1}
\providecommand{\url}[1]{{#1}}
\providecommand{\urlprefix}{URL }
\expandafter\ifx\csname urlstyle\endcsname\relax
  \providecommand{\doi}[1]{DOI~\discretionary{}{}{}#1}\else
  \providecommand{\doi}{DOI~\discretionary{}{}{}\begingroup
  \urlstyle{rm}\Url}\fi
\providecommand{\eprint}[2][]{\url{#2}}

\bibitem[{van~der Aalst and et~al.(2012)}]{ProcessMiningManifesto}
van~der Aalst W, et~al (2012) Process mining manifesto. In: BPM Workshops 2012

\bibitem[{van~der Aalst(2016)}]{Aalst:2016}
van~der Aalst WMP (2016) Process Mining - Data Science in Action. Springer

\bibitem[{van~der Aalst et~al.(2010)van~der Aalst, Pesic, and Song}]{VPS10}
van~der Aalst WMP, Pesic M, Song M (2010) Beyond process mining: From the past
  to present and future. In: CAiSE 2010

\bibitem[{van~der Aalst et~al.(2011)van~der Aalst, Schonenberg, and
  Song}]{ASS11}
van~der Aalst WMP, Schonenberg M, Song M (2011) Time prediction based on
  process mining. Inf\ Sys\

\bibitem[{Bagheri~Hariri et~al.(2013{\natexlab{a}})Bagheri~Hariri, Calvanese,
  De~Giacomo, Deutsch, and Montali}]{BCDDM13}
Bagheri~Hariri B, Calvanese D, De~Giacomo G, Deutsch A, Montali M
  (2013{\natexlab{a}}) Verification of relational data-centric dynamic systems
  with external services. In: PODS 2013

\bibitem[{Bagheri~Hariri et~al.(2013{\natexlab{b}})Bagheri~Hariri, Calvanese,
  Montali, Santoso, and Solomakhin}]{AS-ICSOC-13}
Bagheri~Hariri B, Calvanese D, Montali M, Santoso A, Solomakhin D
  (2013{\natexlab{b}}) Verification of semantically-enhanced artifact systems.
  In: Proc.\ of the 11th Int.\ Joint Conf.\ on Service Oriented Computing
  (ICSOC), Springer, LNCS, vol 8274, pp 600--607,
  \doi{https://doi.org/10.1007/978-3-642-45005-1_51}

\bibitem[{Barrett et~al.(2009)Barrett, Sebastiani, Seshia, and
  Tinelli}]{BSST09}
Barrett CW, Sebastiani R, Seshia SA, Tinelli C (2009) Satisfiability modulo
  theories. In: Handbook of Satisfiability

\bibitem[{Basin et~al.(2008)Basin, Klaedtke, M{\"u}ller, and
  Pfitzmann}]{BKMP08}
Basin D, Klaedtke F, M{\"u}ller S, Pfitzmann B (2008) {Runtime Monitoring of
  Metric First-order Temporal Properties}. In: IARCS

\bibitem[{Basin et~al.(2013)Basin, Klaedtke, Marinovic, and
  Z{\u{a}}linescu}]{BKMZ13}
Basin D, Klaedtke F, Marinovic S, Z{\u{a}}linescu E (2013) Monitoring of
  temporal first-order properties with aggregations. In: RV 2013

\bibitem[{Basin et~al.(2015{\natexlab{a}})Basin, Klaedtke, Marinovic, and
  Z{\u{a}}linescu}]{BKMZ15a}
Basin D, Klaedtke F, Marinovic S, Z{\u{a}}linescu E (2015{\natexlab{a}})
  Monitoring of temporal first-order properties with aggregations. FMSD

\bibitem[{Basin et~al.(2015{\natexlab{b}})Basin, Klaedtke, M\"{u}ller, and
  Z\u{a}linescu}]{BKMZ15b}
Basin D, Klaedtke F, M\"{u}ller S, Z\u{a}linescu E (2015{\natexlab{b}})
  Monitoring metric first-order temporal properties. JACM

\bibitem[{Bianculli et~al.(2013)Bianculli, Ghezzi, and San~Pietro}]{BGS13}
Bianculli D, Ghezzi C, San~Pietro P (2013) The tale of {SOLOIST}: A
  specification language for service compositions interactions. In: FACS 2012

\bibitem[{Breiman(2001)}]{B01}
Breiman L (2001) Random forests. Machine Learning 45(1):5--32,
  \doi{10.1023/A:1010933404324}

\bibitem[{Breiman et~al.(1984)Breiman, Friedman, Stone, and Olshen}]{BFSO84}
Breiman L, Friedman J, Stone C, Olshen R (1984) Classification and Regression
  Trees. The Wadsworth and Brooks-Cole statistics-probability series, Taylor \&
  Francis

\bibitem[{Breuker et~al.(2016)Breuker, Matzner, Delfmann, and Becker}]{BMDB16}
Breuker D, Matzner M, Delfmann P, Becker J (2016) Comprehensible predictive
  models for business processes. MIS Quarterly

\bibitem[{Calvanese et~al.(2014)Calvanese, {İ}smail~{İ}lkan Ceylan, Montali,
  and Santoso}]{AS-JELIA-14}
Calvanese D, {İ}smail~{İ}lkan Ceylan, Montali M, Santoso A (2014)
  Verification of context-sensitive knowledge and action bases. In: Proc.\ of
  JELIA, Springer, LNCS, vol 8761, pp 514--528,
  \doi{https://doi.org/10.1007/978-3-319-11558-0_36}

\bibitem[{Calvanese et~al.(2015)Calvanese, Montali, and Santoso}]{AS-IJCAI-15}
Calvanese D, Montali M, Santoso A (2015) Verification of generalized
  inconsistency-aware knowledge and action bases. In: Proc.\ of the 24th Int.\
  Joint Conf.\ on Artificial Intelligence (IJCAI), AAAI Press, pp 2847--2853

\bibitem[{Conforti et~al.(2013)Conforti, de~Leoni, La~Rosa, and van~der
  Aalst}]{CDLV13}
Conforti R, de~Leoni M, La~Rosa M, van~der Aalst WMP (2013) Supporting
  risk-informed decisions during business process execution. In: CAiSE 2013

\bibitem[{Conforti et~al.(2015)Conforti, de~Leoni, La~Rosa, van~der Aalst, and
  ter Hofstede}]{CDLVT15}
Conforti R, de~Leoni M, La~Rosa M, van~der Aalst WM, ter Hofstede AH (2015) A
  recommendation system for predicting risks across multiple business process
  instances. DSS

\bibitem[{De~Masellis et~al.(2014)De~Masellis, Maggi, and Montali}]{DMM14}
De~Masellis R, Maggi FM, Montali M (2014) Monitoring data-aware business
  constraints with finite state automata. In: ICSSP 2014

\bibitem[{Di~Francescomarino et~al.(2016{\natexlab{a}})Di~Francescomarino,
  Dumas, Federici, Ghidini, Maggi, and Rizzi}]{DDFGMR16}
Di~Francescomarino C, Dumas M, Federici M, Ghidini C, Maggi FM, Rizzi W
  (2016{\natexlab{a}}) Predictive business process monitoring framework with
  hyperparameter optimization. In: CAiSE 2016

\bibitem[{Di~Francescomarino et~al.(2016{\natexlab{b}})Di~Francescomarino,
  Dumas, Maggi, and Teinemaa}]{DDFT16}
Di~Francescomarino C, Dumas M, Maggi FM, Teinemaa I (2016{\natexlab{b}})
  Clustering-based predictive process monitoring. IEEE TSC

\bibitem[{Di~Francescomarino et~al.(2017)Di~Francescomarino, Ghidini, Maggi,
  Petrucci, and Yeshchenko}]{DGMPY17}
Di~Francescomarino C, Ghidini C, Maggi FM, Petrucci G, Yeshchenko A (2017) An
  eye into the future: Leveraging a-priori knowledge in predictive business
  process monitoring. In: BPM 2017

\bibitem[{Di~Francescomarino et~al.(2018)Di~Francescomarino, Ghidini, Maggi,
  and Milani}]{DGMM18}
Di~Francescomarino C, Ghidini C, Maggi FM, Milani F (2018) Predictive process
  monitoring methods: Which one suits me best? In: BPM 2018

\bibitem[{Du et~al.(2015)Du, Liu, and Tiu}]{DLT15}
Du X, Liu Y, Tiu A (2015) Trace-length independent runtime monitoring of
  quantitative policies in {LTL}. In: FM 2015

\bibitem[{Evermann et~al.(2017{\natexlab{a}})Evermann, Rehse, and
  Fettke}]{ERF17a}
Evermann J, Rehse JR, Fettke P (2017{\natexlab{a}}) A deep learning approach
  for predicting process behaviour at runtime. In: BPM Workshops 2016

\bibitem[{Evermann et~al.(2017{\natexlab{b}})Evermann, Rehse, and
  Fettke}]{ERF17b}
Evermann J, Rehse JR, Fettke P (2017{\natexlab{b}}) Predicting process
  behaviour using deep learning. DSS

\bibitem[{Folino et~al.(2012)Folino, Guarascio, and Pontieri}]{FGP12}
Folino F, Guarascio M, Pontieri L (2012) Discovering context-aware models for
  predicting business process performances. In: OTM 2012

\bibitem[{Freund and Schapire(1997)}]{FS97}
Freund Y, Schapire RE (1997) A decision-theoretic generalization of on-line
  learning and an application to boosting. Journal of Computer and System
  Sciences

\bibitem[{Friedman et~al.(2001)Friedman, Hastie, and Tibshirani}]{FHT01}
Friedman J, Hastie T, Tibshirani R (2001) The elements of statistical learning.
  Springer

\bibitem[{Geurts et~al.(2006)Geurts, Ernst, and Wehenkel}]{GEW06}
Geurts P, Ernst D, Wehenkel L (2006) Extremely randomized trees. Machine
  Learning 63(1):3--42, \doi{10.1007/s10994-006-6226-1}

\bibitem[{Goodfellow et~al.(2016)Goodfellow, Bengio, and Courville}]{GBC16}
Goodfellow I, Bengio Y, Courville A (2016) Deep Learning. MIT Press

\bibitem[{Han et~al.(2011)Han, Pei, and Kamber}]{HPK11}
Han J, Pei J, Kamber M (2011) Data mining: concepts and techniques. Elsevier

\bibitem[{{IEEE Comp. Intelligence Society}(2016)}]{IEEE-XES:2016}
{IEEE Comp Intelligence Society} (2016) {IEEE Standard for eXtensible Event
  Stream (XES)} for achieving interoperability in event logs and event streams.
  IEEE Std 1849-2016

\bibitem[{Leontjeva et~al.(2015)Leontjeva, Conforti, Di~Francescomarino, Dumas,
  and Maggi}]{LCDDM15}
Leontjeva A, Conforti R, Di~Francescomarino C, Dumas M, Maggi FM (2015) Complex
  symbolic sequence encodings for predictive monitoring of business processes.
  In: BPM 2015

\bibitem[{Maggi et~al.(2013)Maggi, Dumas, Garc{\'i}a-Ba{\~{n}}uelos, and
  Montali}]{MDGM13}
Maggi FM, Dumas M, Garc{\'i}a-Ba{\~{n}}uelos L, Montali M (2013) Discovering
  data-aware declarative process models from event logs. In: BPM 2013

\bibitem[{Maggi et~al.(2014)Maggi, Di~Francescomarino, Dumas, and
  Ghidini}]{MFDG14}
Maggi FM, Di~Francescomarino C, Dumas M, Ghidini C (2014) Predictive monitoring
  of business processes. In: CAiSE 2014

\bibitem[{M{\'a}rquez-Chamorro et~al.(2017{\natexlab{a}})M{\'a}rquez-Chamorro,
  Resinas, and Ruiz-Cort{\'e}s}]{MRR18}
M{\'a}rquez-Chamorro AE, Resinas M, Ruiz-Cort{\'e}s A (2017{\natexlab{a}})
  Predictive monitoring of business processes: a survey. IEEE TSC

\bibitem[{M{\'a}rquez-Chamorro et~al.(2017{\natexlab{b}})M{\'a}rquez-Chamorro,
  Resinas, Ruiz-Cort{\'e}s, and Toro}]{MRRT17}
M{\'a}rquez-Chamorro AE, Resinas M, Ruiz-Cort{\'e}s A, Toro M
  (2017{\natexlab{b}}) Run-time prediction of business process indicators using
  evolutionary decision rules. ESWA

\bibitem[{Mehdiyev et~al.(2017)Mehdiyev, Evermann, and Fettke}]{MEF17}
Mehdiyev N, Evermann J, Fettke P (2017) A multi-stage deep learning approach
  for business process event prediction. In: CBI 2017

\bibitem[{Metzger and F{\"o}cker(2017)}]{MF17}
Metzger A, F{\"o}cker F (2017) Predictive business process monitoring
  considering reliability estimates. In: CAiSE 2017

\bibitem[{Metzger et~al.(2012)Metzger, Franklin, and Engel}]{MFE12}
Metzger A, Franklin R, Engel Y (2012) Predictive monitoring of heterogeneous
  service-oriented business networks: The transport and logistics case. In:
  SRII 2012

\bibitem[{Metzger et~al.(2015)Metzger, Leitner, Ivanović, Schmieders,
  Franklin, Carro, Dustdar, and Pohl}]{MLISFCDP15}
Metzger A, Leitner P, Ivanović D, Schmieders E, Franklin R, Carro M, Dustdar
  S, Pohl K (2015) Comparing and combining predictive business process
  monitoring techniques. IEEE TSMC

\bibitem[{Mohri et~al.(2012)Mohri, Rostamizadeh, and Talwalkar}]{MRT12}
Mohri M, Rostamizadeh A, Talwalkar A (2012) Foundations of machine learning.
  MIT press

\bibitem[{{Object Management Group}(2015)}]{OMG15}
{Object Management Group} (2015) {Decision Model and Notation (DMN)} 1.0.
  \urlprefix\url{http://www.omg.org/spec/DMN/1.0/}

\bibitem[{Pedregosa et~al.(2011)Pedregosa, Varoquaux, Gramfort, Michel,
  Thirion, Grisel, Blondel, Prettenhofer, Weiss, Dubourg, Vanderplas, Passos,
  Cournapeau, Brucher, Perrot, and Duchesnay}]{scikit-learn}
Pedregosa F, Varoquaux G, Gramfort A, Michel V, Thirion B, Grisel O, Blondel M,
  Prettenhofer P, Weiss R, Dubourg V, Vanderplas J, Passos A, Cournapeau D,
  Brucher M, Perrot M, Duchesnay E (2011) Scikit-learn: Machine learning in
  {P}ython. Journal of Machine Learning Research 12:2825--2830

\bibitem[{Pesic and van~der Aalst(2006)}]{PV06}
Pesic M, van~der Aalst WMP (2006) A declarative approach for flexible business
  processes management. In: BPM Workshops 2006

\bibitem[{Pesic et~al.(2007)Pesic, Schonenberg, and van~der Aalst}]{PSV07}
Pesic M, Schonenberg H, van~der Aalst WMP (2007) {DECLARE}: Full support for
  loosely-structured processes. In: EDOC 2007

\bibitem[{Pika et~al.(2012)Pika, van~der Aalst, Fidge, ter Hofstede, and
  Wynn}]{PVFTW12}
Pika A, van~der Aalst WMP, Fidge CJ, ter Hofstede AHM, Wynn MT (2012)
  Predicting deadline transgressions using event logs. In: BPM Workshops 2012

\bibitem[{Pika et~al.(2016)Pika, van~der Aalst, Wynn, Fidge, and ter
  Hofstede}]{PVWFT16}
Pika A, van~der Aalst W, Wynn M, Fidge C, ter Hofstede A (2016) Evaluating and
  predicting overall process risk using event logs. Inf\ Sci\

\bibitem[{Pnueli(1977)}]{Pnue77}
Pnueli A (1977) The temporal logic of programs. In: Proc.\ of FOCS, pp 46--57

\bibitem[{Polato et~al.(2014)Polato, Sperduti, Burattin, and de~Leoni}]{PSBD14}
Polato M, Sperduti A, Burattin A, de~Leoni M (2014) Data-aware remaining time
  prediction of business process instances. In: IJCNN 2014

\bibitem[{Polato et~al.(2018)Polato, Sperduti, Burattin, and Leoni}]{PSBD18}
Polato M, Sperduti A, Burattin A, Leoni Md (2018) Time and activity sequence
  prediction of business process instances. Computing

\bibitem[{Rogge-Solti and Weske(2013)}]{RW13}
Rogge-Solti A, Weske M (2013) Prediction of remaining service execution time
  using stochastic petri nets with arbitrary firing delays. In: ICSOC 2013

\bibitem[{Rogge-Solti and Weske(2015)}]{RW15}
Rogge-Solti A, Weske M (2015) Prediction of business process durations using
  non-markovian stochastic petri nets. Inf\ Sys\

\bibitem[{Santoso(2016)}]{thesis-as-16}
Santoso A (2016) Verification of data-aware business processes in the presence
  of ontologies. PhD thesis, Free University of Bozen-Bolzano, Technische
  Universit{\"a}t Dresden,
  \protect\url{http://nbn-resolving.de/urn:nbn:de:bsz:14-qucosa-213372}

\bibitem[{Santoso(2018)}]{AS-BPMDS-18}
Santoso A (2018) Specification-driven multi-perspective predictive business
  process monitoring. In: Enterprise, Business-Process and Information Systems
  Modeling, BPMDS 2018, EMMSAD 2018, Springer, LNBIP, vol 318, pp 97--113,
  \doi{https://doi.org/10.1007/978-3-319-91704-7_7}

\bibitem[{Senderovich et~al.(2014)Senderovich, Weidlich, Gal, and
  Mandelbaum}]{SWGM14}
Senderovich A, Weidlich M, Gal A, Mandelbaum A (2014) Queue mining –
  predicting delays in service processes. In: CAiSE 2014

\bibitem[{Senderovich et~al.(2015)Senderovich, Weidlich, Gal, and
  Mandelbaum}]{SWGM15}
Senderovich A, Weidlich M, Gal A, Mandelbaum A (2015) Queue mining for delay
  prediction in multi-class service processes. Inf\ Sys\

\bibitem[{Senderovich et~al.(2017)Senderovich, Di~Francescomarino, Ghidini,
  Jorbina, and Maggi}]{SDGJM17}
Senderovich A, Di~Francescomarino C, Ghidini C, Jorbina K, Maggi FM (2017)
  Intra and inter-case features in predictive process monitoring: A tale of two
  dimensions. In: BPM 2017

\bibitem[{Smullyan(1968)}]{Smul68}
Smullyan RM (1968) First Order Logic. Springer, Berlin (Germany)

\bibitem[{Steeman(2013)}]{BPI-13-data}
Steeman W (2013) {BPI} challenge 2013.
  \urlprefix\url{https://doi.org/10.4121/uuid:a7ce5c55-03a7-4583-b855-98b86e1a2b07}

\bibitem[{Tax et~al.(2017)Tax, Verenich, La~Rosa, and Dumas}]{TVLD17}
Tax N, Verenich I, La~Rosa M, Dumas M (2017) Predictive business process
  monitoring with {LSTM} neural networks. In: CAiSE 2017

\bibitem[{{Theano Development Team}(2016)}]{theano}
{Theano Development Team} (2016) {Theano: A {Python} framework for fast
  computation of mathematical expressions}. arXiv e-prints
  \urlprefix\url{http://arxiv.org/abs/1605.02688}

\bibitem[{Van~Dongen(2012)}]{BPI-12-data}
Van~Dongen B (2012) {BPI} challenge 2012.
  \urlprefix\url{https://doi.org/10.4121/uuid:3926db30-f712-4394-aebc-75976070e91f}

\bibitem[{Van~Dongen(2015)}]{BPI-15-data}
Van~Dongen B (2015) {BPI} challenge 2015.
  \urlprefix\url{https://doi.org/10.4121/uuid:ed445cdd-27d5-4d77-a1f7-59fe7360cfbe}

\bibitem[{Verenich et~al.(2015)Verenich, Dumas, La~Rosa, Maggi, and
  Di~Francescomarino}]{VDLMD15}
Verenich I, Dumas M, La~Rosa M, Maggi FM, Di~Francescomarino C (2015) Complex
  symbolic sequence clustering and multiple classifiers for predictive process
  monitoring. In: BPM Workshops 2015

\end{thebibliography}

\end{document}